%% file: mda_aistats.tex
\pdfoutput=1

\documentclass[twoside]{article}
\usepackage[pdftex]{graphicx}
\usepackage{comment}
\usepackage[utf8]{inputenc} 
\usepackage[T1]{fontenc}    
\usepackage{hyperref}       
\usepackage{url}            
\usepackage{booktabs}       
\usepackage{amsfonts}       
\usepackage{nicefrac}       
\usepackage{microtype}      
\usepackage{placeins}
\usepackage{caption}
\usepackage{subcaption}
\usepackage{float}
\usepackage{mathtools}
\usepackage{framed}
\usepackage{pbox}
\usepackage{afterpage}
\usepackage{url}
\usepackage{array}
\usepackage{amsfonts}
\usepackage{amsthm}
\usepackage[table]{xcolor}

\usepackage{amsmath}
\usepackage{amssymb}
\usepackage{multirow}

\newcommand{\kartik}[1]{\textcolor{red}{KG: #1}}
\newcommand{\AJ}[1]{\textcolor{magenta}{AJ: #1}}

\usepackage{natbib}

\input{math_commands.tex}

\input{macros.tex}

\finalcopy 
\arxivcopy 

\ifarxiv
\ifsupp
\global\suppfalse
\fi
\fi

\ifsupp

\newcommand{\papertitle}{Mirror Descent View for Neural Network Quantization\\ -- Supplementary Material --}
\else
\newcommand{\papertitle}{Mirror Descent View for Neural Network Quantization}
\fi

\iffinal
\usepackage[accepted]{aistats2021}

\runningtitle{Mirror Descent View for Neural Network Quantization}

\runningauthor{Thalaiyasingam Ajanthan*, Kartik Gupta*, Philip H. S. Torr, Richard Hartley, Puneet K. Dokania}

\else
\usepackage{aistats2021}
\fi

\begin{document}
\ifsupp
\appendix
\onecolumn
\appendix
\begin{center}
\textsc{\large Supplementary document}

\textbf{\Large Mirror Descent View for Neural Network Quantization}
\end{center}

\input{text/appendix.tex}

\else
\twocolumn[
\aistatstitle{\papertitle{}}

\aistatsauthor{Thalaiyasingam Ajanthan$^{*\dagger}$\\
Australian National University \& Amazon
\And
Kartik Gupta$^*$\\
Australian National University \& Data61, CSIRO
\AND
Philip H. S. Torr\\
University of Oxford
\And
Richard Hartley\\
Australian National University
\And
Puneet K. Dokania$^\ddagger$\\
University of Oxford \& Five AI
}
\aistatsaddress{\\$^*$Equal contribution, $^\dagger$ Work done prior to joining Amazon, $^\ddagger$ Work done prior to joining Five AI}
]
\input{text/abstract.tex}

\input{text/intro.tex}

\input{text/prelim.tex}

\input{text/mda_proj.tex}

\input{text/relatedwork.tex}
\input{text/experiments.tex}

\input{text/discussion.tex}

\iffinal
\input{text/ack.tex}
\fi
\fi

\ifarxiv
\clearpage
\appendix
\onecolumn
\section*{Appendices}
\input{text/appendix.tex}
\fi
\bibliography{mda_aistats}
\bibliographystyle{apalike}

\end{document}

%% file: math_commands.tex
\usepackage{amsmath,amsfonts,bm}

\def\figref#1{figure~\ref{#1}}

\def\secref#1{section~\ref{#1}}

\def\eqref#1{equation~\ref{#1}}

\def\plaineqref#1{\ref{#1}}

\def\Twoalgref#1#2{Algorithms \ref{#1} and \ref{#2}}

\def\1{\bm{1}}

\newcommand{\test}{\mathcal{D_{\mathrm{test}}}}

\def\vs{{\bm{s}}}

\DeclareMathAlphabet{\mathsfit}{\encodingdefault}{\sfdefault}{m}{sl}
\SetMathAlphabet{\mathsfit}{bold}{\encodingdefault}{\sfdefault}{bx}{n}

\newcommand{\R}{\mathbb{R}}

\newcommand{\softmax}{\mathrm{softmax}}
\newcommand{\sigmoid}{\sigma}

\DeclareMathOperator*{\argmax}{argmax}
\DeclareMathOperator*{\argmin}{argmin}

\DeclareMathOperator{\sign}{sign}

%% file: macros.tex
\usepackage{float}
\usepackage{mathtools}
\usepackage{framed}
\usepackage{pbox}
\usepackage{afterpage}
\usepackage{url}
\usepackage{array}
\usepackage{amsfonts}
\usepackage[table]{xcolor}

\usepackage{multirow}
\usepackage{booktabs}
\usepackage{relsize}
\usepackage{xspace}
\usepackage{caption}
\usepackage{subcaption}
\newcommand{\ignore}[1]{}

\newcommand{\norm}[1]{\left\lVert#1\right\rVert}

\usepackage{framed}	
\usepackage{url}
\usepackage{enumerate}
\usepackage{color}

\usepackage{changepage}

\usepackage{algpseudocode}
\usepackage{algorithm}

\usepackage{amsthm}
\theoremstyle{definition}
\newtheorem{dfn}{Definition}
\newtheorem{pro}{Proposition}
\newtheorem{exm}{Example}
\theoremstyle{plain}
\newtheorem{thm}{Theorem}

\newcommand{\SKIP}[1]{}
\newcommand{\calV}{\mathcal{V}^m}

\newcommand{\calO}{\mathcal{O}}

\newcommand{\bfx}{\mathbf{x}}
\newcommand{\tbfx}{\tilde{\mathbf{x}}}

\newcommand{\bfy}{\mathbf{y}}
\newcommand{\bfz}{\mathbf{z}}

\newcommand{\calX}{\mathcal{X}}

\newcommand{\bfone}{\boldsymbol{1}}

\renewcommand{\R}{\rm I\!R}

\newcommand{\calC}{\mathcal{C}}

\newcommand{\bfp}{\mathbf{p}}

\newcommand{\calD}{\mathcal{D}}

\newcommand{\calB}{\mathcal{B}}

\newcommand{\bfw}{\mathbf{w}}

\newcommand{\bfu}{\mathbf{u}}

\newcommand{\conv}{\operatorname{conv}}

\newcommand{\calQ}{\mathcal{Q}}
\newcommand{\bfq}{\mathbf{q}}
\newcommand{\tbfw}{\tilde{\mathbf{w}}}
\newcommand{\tbfu}{\tilde{\mathbf{u}}}
\newcommand{\tbfv}{\tilde{\mathbf{v}}}
\newcommand{\calS}{\Delta^m}

\renewcommand{\sign}{\operatorname{sign}}

\renewcommand{\softmax}{\operatorname{softmax}}

\renewcommand{\sigmoid}{\operatorname{sigmoid}}

\newcommand{\prox}{\operatorname{prox}}

\newcommand{\tu}{\tilde{u}}
\newcommand{\tv}{\tilde{v}}
\newcommand{\tw}{\tilde{w}}
\newcommand{\tx}{\tilde{x}}

\newcommand{\bfg}{\mathbf{g}}

\newcommand{\allweights}{\{1,\ldots, m\}}

\newcommand{\iproj}{P^{-1}}

\newcommand{\bcalC}{\bar{\calC}}
\newcommand{\primal}{\calX\cap\calC}

\newcommand{\mdtanhs}{\acrshort{MD}-$\tanh$-\acrshort{S}}

\def\amin#1{{\underset{#1}{\operatorname{argmin}}}}%

\newcommand{\NOTE}[1]{[\textbf{\textcolor{blue}{Note:}}\emph{\textcolor{blue}{#1}}]}
\def\myref#1{{\color{red}{#1}}}%

\usepackage{xcolor}

\usepackage{bbm}

\graphicspath{{images/}}

\usepackage[page]{appendix}
\usepackage{chngcntr}
\usepackage{etoolbox}

\AtBeginEnvironment{subappendices}{%
\section*{Appendix}
\addcontentsline{toc}{section}{Appendix}
}

\usepackage{xspace}
\makeatletter
\DeclareRobustCommand\onedot{\futurelet\@let@token\@onedot}
\def\@onedot{\ifx\@let@token.\else.\null\fi\xspace}

\def\eg{\emph{e.g}\onedot} 
\def\ie{\emph{i.e}\onedot} 
 
 \def\vs{\emph{vs}\onedot}
\def\wrt{w.r.t\onedot} 

\makeatother

\usepackage{enumitem}

\newenvironment{tight_enumerate}{
\begin{enumerate}[leftmargin=10pt]
  \setlength{\topsep}{0pt}
  \setlength{\itemsep}{0pt}
  \setlength{\parskip}{0pt}
  \setlength{\parsep}{0pt}
}{\end{enumerate}}

\def\figref#1{Fig.~\ref{#1}}

\def\secref#1{Sec.~\ref{#1}}
\def\eqref#1{Eq.~(\ref{#1})}
\def\twoeqref#1#2{Eqs.~(\ref{#1})~and~(\ref{#2})}
\def\plaineqref#1{(\ref{#1})}
\def\twoplaineqref#1#2{(\ref{#1})~and~(\ref{#2})}

\def\threetabref#1#2#3{Tables~\ref{#1},~\ref{#2}~and~\ref{#3}}
\def\tabref#1{Table~\ref{#1}}

\def\thmref#1{Theorem~\ref{#1}}

\def\dfnref#1{Definition~\ref{#1}}

\def\tabref#1{Table~\ref{#1}}

\algnewcommand\INPUT{\item[\textbf{Input:}]}%
\algnewcommand\OUTPUT{\item[\textbf{Output:}]}%

\usepackage[acronym,smallcaps]{glossaries} 
\makeglossaries 
\glsdisablehyper
\newacronym{MRF}{mrf}{Markov Random Field}
\newacronym{SGD}{sgd}{Stochastic Gradient Descent}
\newacronym{GD}{gd}{Gradient Descent}
\newacronym{PGD}{pgd}{Projected Gradient Descent}
\newacronym{ICM}{icm}{Iterative Conditional Modes}
\newacronym{DNN}{dnn}{Deep Neural Networks}
\newacronym{NN}{nn}{Neural Network}
\newacronym{PMF}{pmf}{Proximal Mean-Field}
\newacronym{PICM}{picm}{Proximal Iterative Conditional Modes}
\newacronym{BC}{bc}{BinaryConnect}
\newacronym{BWN}{bwn}{Binary Weight Network}
\newacronym{IP}{ip}{Integer Programming}
\newacronym{fc}{fc}{fully-connected}
\newacronym{REF}{ref}{Reference Network}
\newacronym{KL}{kl}{KL}
\newacronym{S}{s}{S}
\newacronym{LR}{lr}{LR}
\newacronym{KKT}{kkt}{KKT}
\newacronym{MAP}{map}{Maximum a Posteriori}
\newacronym{MDA}{mda}{Mirror Descent Algorithm}
\newacronym{MD}{md}{Mirror Descent}
\newacronym{BOP}{bop}{Binary Optimizer}
\newacronym{EDA}{eda}{Entropic Descent Algorithm}
\newacronym{ED}{ed}{Entropic Descent}
\newacronym{EGD}{egd}{Exponentiated Gradient Descent}
\newacronym{PQ}{pq}{ProxQuant}
\newacronym{STE}{ste}{Straight Through Estimator}
\newacronym{HGD}{hgd}{Hybrid Gradient Descent}
\newacronym{PQL}{pql}{PQL}
\newacronym{ELQ}{elq}{Explicit Loss-error-aware Quantization}
\newacronym{ADMM}{admm}{Alternating Direction Method of Multipliers}
\newacronym{QN}{qn}{Quantization Networks}
\newacronym{BR}{br}{BinaryRelax}
\newacronym{FC}{fc}{FC}
\newacronym{BN}{bn}{BN}

\newcommand{\svgg}[1]{{\small VGG#1}}
\newcommand{\sresnet}[1]{{\small ResNet#1}}

\newcommand{\srelu}{{\small ReLU}}
\newcommand{\smobilenet}[1]{{\small MobileNet#1}}

\newcommand{\vvgg}[1]{VGG#1}
\newcommand{\vresnet}[1]{{ResNet#1}}

\newcommand{\cifar}[1]{{\small CIFAR#1}}
\newcommand{\vcifar}[1]{{CIFAR#1}}
\newcommand{\tinyimagenet}{{TinyImageNet}}
\newcommand{\imagenet}{{ImageNet}}

\newcommand{\pql}[1]{\acrshort{PQL}$_#1$}

\usepackage{pifont}
\newcommand{\xmark}{\ding{55}} 

\usepackage{wrapfig}
\usepackage{slashbox}

\newif\ifsupp
\suppfalse

\newif\ifarxiv
\arxivfalse
\def\arxivcopy{\global\arxivtrue}
\newif\iffinal
\finalfalse
\def\finalcopy{\global\finaltrue}

\newcommand{\citenew}[1]{(\cite{#1})}


%% file: text/appendix.tex
Here, we first provide the proofs 
and the technical derivations. Later we give additional experiments and the details of our experimental setting.

\section{\acrshort{MD} Proofs and Derivations}
\SKIP{
\subsection{Deriving Mirror Maps from Projections}
\SKIP{
\begin{thm}\label{thm:mdaproj1}
Let $\calX$ be a compact convex set and $P:\R^r \to \calC$ be an invertible function where $\calC\subset\R^r$ is a convex open set such that $\calX=\bcalC$ ($\bcalC$ denotes the closure of $\calC$).
Now, if $P^{-1}$ satisfy:
\begin{tight_enumerate}
  \item $P^{-1}$ is monotonically increasing.
  \item $\lim_{\bfx\to\partial \calC}\|P^{-1}(\bfx)\| = \infty$ ($\partial \calC$ denotes the boundary of $\calC$).
\end{tight_enumerate}
Then, $\Phi(\bfx) = \langle\bfone,\int \iproj(\bfx) d\bfx\rangle$ is a valid mirror map, where $\bfone$ is an $r$-dimensional vector of ones.
\end{thm}
}
\begin{thm}\label{thm:mdaproj1}
Let $\calX$ be a compact convex set and $P:\R \to \calC$ be an invertible function where $\calC\subset\R$ is a convex open set such that $\calX=\bcalC$ ($\bcalC$ denotes the closure of $\calC$).
Now, if 
\begin{tight_enumerate}
  \item $P$ is strictly monotonically increasing.
  \item $\lim_{x\to\partial \calC}\|\iproj(x)\| = \infty$ ($\partial \calC$ denotes the boundary of $\calC$).
\end{tight_enumerate}
Then, $\Phi(x) = \int_{x_0}^x \iproj(y) dy$ is a valid mirror map.
\end{thm}
\begin{proof}
\kartik{Do we need a proof for this?}
From the fundamental theorem of calculus, the gradient of $\Phi(x)$ satisfies, $\nabla\Phi(x) = \iproj(x)$. Since $P$ is strictly monotonically increasing and invertible, $\iproj$ is strictly monotonically increasing. Therefore, $\Phi(x)$ is strictly convex and differentiable. Now, from the definition of projection and since it is invertible (\ie, $\iproj$ is {\em one-to-one} and {\em onto}), $\nabla\Phi(\calC) = \iproj(\calC) = \R$. Therefore, together with condition (2), we can conclude that $\Phi(x) = \int_{x_0}^x \iproj(y) dy$ is a valid mirror map (refer Definition~\myref{1} in the main paper).
For the multi-dimensional case, we need an additional condition that the vector field $\iproj(\bfx)$ is conservative. 
Then by the gradient theorem~\citenew{gradientwiki1}, there exists a mirror map $\Phi(\bfx) = \int_{\bfx_0}^{\bfx} \iproj(\bfy)d\bfy$ for some arbitrary base point $\bfx_0$. 

\end{proof}
}
\subsection{More details on Projections}
Here, we provide more details on projections used in the paper for \gls{NN} quantization. Even though we consider differentiable projections, Theorem \myref{1} in the main paper does not require the projection to be differentiable. 
For the rest of the section, 
we assume $m=1$, \ie, consider projections that are independent for each ${j\in\allweights}$. 

\begin{exm}[$\bfw$-space, binary, $\tanh$]\label{exm:tanh}
Consider the $\tanh$ function, which projects a real value to the interval $[-1,1]$:
\vspace{-1ex}
\begin{equation}\label{eq:tanh}
w = P_{\beta_k}(\tw):= \tanh(\beta_k\tw) = \frac{\exp(2\beta_k \tw) - 1}{\exp(2\beta_k \tw) + 1}\ ,
\vspace{-0.5ex}
\end{equation}
where $\beta_k \geq 1$ is the annealing hyperparameter and when $\beta_k \to \infty$, $\tanh$ approaches the step function.
The inverse of the $\tanh$ is:
\vspace{-1ex}
\begin{equation}
\iproj_{\beta_k}(w) = \frac{1}{\beta_k}\tanh^{-1}(w) = \frac{1}{2\beta_k}\log\frac{1+w}{1-w}\ .
\vspace{-1ex}
\end{equation}
Note that, $\iproj_{\beta_k}$ is monotonically increasing for a fixed $\beta_k$. 
Correspondingly, the mirror map from Theorem~\myref{1} in the main paper can be written as:
\begin{align}\label{eq:tanhmm}
\Phi_{\beta_k}(w) &= \int \iproj_{\beta_k}(w) dw \nonumber \\
&= \frac{1}{2\beta_k}\big[(1+w)\log(1+w) + (1-w)\log(1-w)\big]\ . 
\end{align}
Here, the constant from the integration is ignored. 
It can be easily verified that $\Phi_{\beta_k}(w)$ is in fact a valid mirror map.
Correspondingly, the Bregman divergence can be written as:
 \begin{align}
D_{\Phi_{\beta_k}}(w, v) &= \Phi_{\beta_k}(w) - \Phi_{\beta_k}(v) - \Phi'_{\beta_k}(v) (w-v)\ ,\quad\mbox{where $\Phi_{\beta_k}'(v) = \frac{1}{2\beta_k}\log\frac{1+v}{1-v}$}\ ,\\\nonumber 
 &= \frac{1}{2\beta_k}\left[w\log\frac{(1+w)(1-v)}{(1-w)(1+v)} +  \log(1-w)(1+w) - \log(1-v)(1+v) \right]\ . 
\end{align}
Now, consider the proximal form of \gls{MD} update
\begin{align}\label{eq:tanhmd1}
w^{k+1} & = \amin{\bfx\in(-1,1)}\, \langle \eta\,g^k, w\rangle + D_{\Phi_{\beta_k}}(w, w^k)\ .
\end{align}
The idea is to find $w$ such that the \acrshort{KKT} conditions are satisfied. 
To this end, let us first write the Lagrangian of~\eqref{eq:tanhmd1} by introducing dual variables $y$ and $z$ corresponding to the constraints $w> -1$ and $w< 1$, respectively:
\begin{align}\label{eq:lagtanh1}
F(w, x, y) &= \eta g^kw + y(-w-1) + z(w-1)\\\nonumber &+ \frac{1}{2\beta_k}\left[w\log\frac{(1+w)(1-w^k)}{(1-w)(1+w^k)} +  \log(1-w)(1+w)- \log(1-w^k)(1+w^k) \right]\ .
\end{align}
Now, setting the derivatives with respect to $w$ to zero:
\begin{align}\label{eq:dlagtanh1}
\frac{\partial F}{\partial w} &= \eta g^k + \frac{1}{2\beta_k}\log\frac{(1+w)(1-w^k)}{(1-w)(1+w^k)} - y + z = 0\ .
\end{align}
From complementary slackness conditions,
\begin{align}
y(-w-1) &= 0\ ,\qquad\text{since}\quad w > -1\ \Rightarrow\ y = 0\ ,\\\nonumber
z(w-1) &= 0\ ,\qquad\text{since}\quad w < 1\ \Rightarrow\ z = 0\ .
\end{align} 
Therefore,~\eqref{eq:dlagtanh1} now simplifies to:
\begin{align}
\frac{\partial F}{\partial w} &= \eta g^k + \frac{1}{2\beta_k}\log\frac{(1+w)(1-w^k)}{(1-w)(1+w^k)} = 0\ ,\\\nonumber
\log\frac{(1+w)(1-w^k)}{(1-w)(1+w^k)} &= -2\beta_k\eta g^k\ ,\\\nonumber
 \frac{1+w}{1-w} &= \frac{1+w^k}{1-w^k}\exp(-2\beta_k\eta g^k)\ ,\\\nonumber
 w &= \frac{\frac{1+w^k}{1-w^k}\exp(-2\beta_k\eta g^k) - 1}{\frac{1+w^k}{1-w^k}\exp(-2\beta_k\eta g^k) + 1}\ .
\end{align}
\SKIP{
Consequently, the resulting \gls{MD} update takes the following form:
\begin{align}\label{eq:tanhmd}
w^{k+1} &= \amin{w\in(-1,1)}\, \langle \eta\,g^k, w\rangle + D_{\Phi_{\beta_k}}(w, w^k) \nonumber \\ 
 &= \frac{\frac{1+w^k}{1-w^k}\exp(-2\beta_k\eta g^k) - 1}{\frac{1+w^k}{1-w^k}\exp(-2\beta_k\eta g^k) + 1}\ .
 \vspace{-0.5ex} 
\end{align}
The update formula is derived using the \acrshort{KKT} conditions~\citenew{boyd2009convex} as will be shown later.\AJ{Why don't we merge A.2 with this subsection. No need to divide this part into two.}}
A similar derivation can also be performed for the $\sigmoid$ function, where $\bcalC=\calX=[0,1]$.
Note that the $\sign$ function has been used for binary quantization in~\cite{courbariaux2015binaryconnect} and $\tanh$ can be used as a soft version of $\sign$ function as pointed out by~\cite{zhang2015bit}.
Mirror map corresponding to $\tanh$ is used for online linear optimization in~\cite{bubeck2012towards} but here we use it for \gls{NN} quantization. The pseudocodes of original (\acrshort{MD}-$\tanh$) and numerically stable versions (\acrshort{MD}-$\tanh$-\acrshort{S}) for $\tanh$ are presented in~\Twoalgref{alg:mdtanh}{alg:mdtanhste} respectively.

  \begin{algorithm}[t]
\caption{\acrshort{MD}-$\tanh$}
\label{alg:mdtanh}
\begin{algorithmic}[1]

\Require $K, b, \{\eta^k\}, \rho >1, \calD, L$ 
\Ensure $\bfw^*\in\calQ^m$

\State $\bfw^0\in \R^{m},\quad \beta_0\gets 1$ 
\Comment{Initialization}

\State $\bfw^0 \gets \tanh({\beta_0\bfw^0})$
\Comment{Projection}

\For{$k \gets 0,\ldots, K$} 

\State $\calD^{b} = \{\left( \bfx_i, \bfy_i \right)\}^{b}_{i=1} \sim
\mathcal{D}$
\Comment{Sample a mini-batch}

\State $\bfg^k \gets
\left.\nabla_{\bfw}L(\bfw;\calD^b)\right|_{\bfw=\bfw^k}$
\Comment{Gradient \wrt $\bfw$ at $\bfw^k$ (Adam based gradients)}

\For{$j \gets {1,\ldots,m}$}
\State $w_j^{k+1}\gets \frac{\frac{1+w_j^k}{1-w_j^k}\exp(-2\beta_k\eta^k g_j^k) - 1}{\frac{1+w_j^k}{1-w_j^k}\exp(-2\beta_k\eta^k g_j^k) + 1}$
\Comment{\acrshort{MD} update}
\EndFor

\State $\beta_{k+1}\gets \rho\beta_k$
\Comment{Increase $\beta$}

\EndFor

\State $\bfw^* \gets \sign({\tbfw^K})$
\Comment{Quantization}

\end{algorithmic}
\end{algorithm}

  \begin{algorithm}[t]
\caption{\acrshort{MD}-$\tanh$-\acrshort{S}}
\label{alg:mdtanhste}
\begin{algorithmic}[1]

\Require $K, b, \{\eta^k\}, \rho >1, \calD, L$ 
\Ensure $\bfw^*\in\calQ^m$

\State $\tbfw^0\in \R^{m},\quad \beta_0\gets 1$ 
\Comment{Initialization}

\For{$k \gets 0,\ldots, K$} 

\State $\bfw^k \gets \tanh({\beta_k\tbfw^k})$
\Comment{Projection}

\State $\calD^{b} = \{\left( \bfx_i, \bfy_i \right)\}^{b}_{i=1} \sim
\mathcal{D}$
\Comment{Sample a mini-batch}

\State $\bfg^k \gets
\left.\nabla_{\bfw}L(\bfw;\calD^b)\right|_{\bfw=\bfw^k}$
\Comment{Gradient \wrt $\bfw$ at $\bfw^k$ (Adam based gradients)}

\State $\tbfw^{k+1} \gets \tbfw^{k} - \eta^k \bfg^k$
\Comment{Gradient descent on $\tbfw$}

\State $\beta_{k+1}\gets \rho\beta_k$
\Comment{Increase $\beta$}

\EndFor

\State $\bfw^* \gets \sign({\tbfw^K})$
\Comment{Quantization}

\end{algorithmic}
\end{algorithm}

\SKIP{
the Bregman divergence can be written as:
 \begin{equation}
D_{\Phi}(w, v) = \frac{1}{2}\left[w\log\frac{(1+w)(1-v)}{(1-w)(1+v)} +  \log(1-w)(1+w) - \log(1-v)(1-v) \right]\ . 
\end{equation}
}
\end{exm}

\begin{exm}[$\bfu$-space, multi-label, $\softmax$]\label{exm:sm}
Now we consider the $\softmax$ projection used in \gls{PMF}~\citenew{ajanthan2018pmf} to optimize in the lifted probability space. 
In this case, the projection is defined as $P_{\beta_k}(\tbfu):= \softmax(\beta_k\tbfu)$ where $P_{\beta_k}:\R^d \to \calC$ with $\bcalC = \calX = \Delta$. Here $\Delta$ is the $(d-1)$-dimensional probability simplex and $|\calQ|=d$.
Note that the $\softmax$ projection is not invertible as it is a many-to-one mapping.
In particular, it is invariant to translation, \ie,
\vspace{-1ex}
\begin{align}
\bfu=\softmax(\tbfu + c\bfone) = \softmax(\tbfu)\ , \nonumber \\ \text{where}\quad
u_{\lambda} = \frac{\exp(\tu_{\lambda})}{\sum_{\mu\in\calQ}\exp(\tu_{\mu})}\ \nonumber,
\vspace{-0.5ex}
\end{align} 
for any scalar $c\in\R$ ($\bfone$ denotes a vector of all ones).
%
Therefore, the $\softmax$ projection does not satisfy Theorem~\myref{1} in the main paper.
However, one could obtain a solution of the inverse of $\softmax$ as follows: given $\bfu\in\Delta$, find a unique point $\tbfv = \tbfu+c\bfone$, for a particular scalar $c$, such that $\bfu=\softmax(\tbfv)$.
Now, by choosing $c=-\log(\sum_{\mu=\calQ} \exp(\tu_{\mu}))$, $\softmax$ can be written as:
\vspace{-1ex}
\begin{equation}\label{eq:sminv}
\bfu = \softmax(\tbfv)\ ,\qquad \text{where}\quad
u_{\lambda} = \exp(\tv_{\lambda})\ ,\quad\forall\,\lambda\in\calQ\ .
\vspace{-1ex}
\end{equation} 
Now, the inverse of the projection $P_{\beta_k}$ can be written as: 
\vspace{-1ex}
\begin{equation}
\tbfv = \iproj_{\beta_k}(\bfu)  = \frac{1}{\beta_k}\softmax^{-1}(\bfu)\ ,\qquad \text{where}\quad
\tv_{\lambda} = \frac{1}{\beta_k}\log(u_{\lambda})\ ,\qquad\forall\,\lambda\ .
\vspace{-1ex}
\end{equation} 
Indeed, $\log$ is a monotonically increasing function and from Theorem~\myref{1} in the main paper, by summing the integrals, the mirror map can be written as:
\vspace{-1ex}
\begin{equation}\label{eq:smmm}
\Phi_{\beta_k}(\bfu) = \frac{1}{\beta_k}\left[\sum_{\lambda}u_{\lambda}\log(u_{\lambda}) - u_{\lambda}\right] = -\frac{1}{\beta_k}H(\bfu) - 1/\beta_k\ .
\vspace{-1ex} 
\end{equation}
Here, $\sum_\lambda u_{\lambda} = 1$ as $\bfu\in\Delta$, and $H(\bfu)$ is the entropy.
Interestingly, as the mirror map in this case is the negative entropy (up to a constant), the \gls{MD} update leads to the well-known \gls{EGD} (or \gls{EDA})~\citenew{beck2003mirror,bubeck2015convex}.
Consequently, the update takes the following form:
\begin{align}\label{eq:eda}
u^{k+1}_{\lambda} &= \frac{u_{\lambda}^k\,\exp(-\beta_k\eta g^{k}_{\lambda})}{\sum_{\mu  \in \calQ} \; 
u_{\mu}^k\,\exp(-\beta_k\eta g^{k}_{\mu})}\quad \forall\, \lambda \ .
\end{align}
The derivation follows the same approach as in the $\tanh$ case above.
It is interesting to note that the \gls{MD} variant of $\softmax$ is equivalent to the well-known \gls{EGD}.
Notice, the authors of \gls{PMF}~\citenew{ajanthan2018pmf} hinted that \gls{PMF} is related to \gls{EGD} but here we have clearly showed that the \gls{MD} variant of \gls{PMF} under the above reparametrization~\plaineqref{eq:sminv} is exactly \gls{EGD}.
\end{exm}
\begin{figure*}[t]
    \centering
    \includegraphics[width=0.25\linewidth]{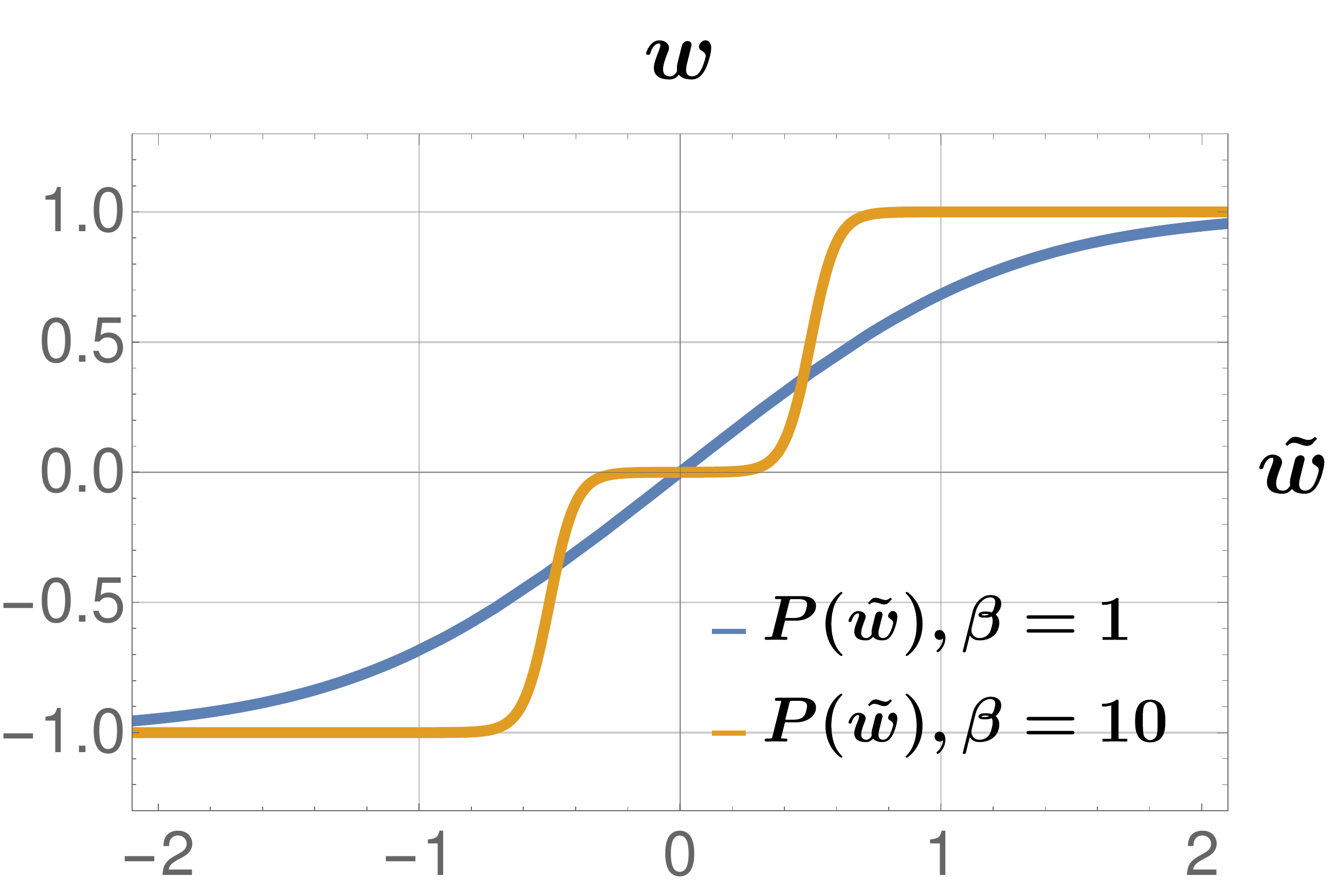}%
    \includegraphics[width=0.25\linewidth]{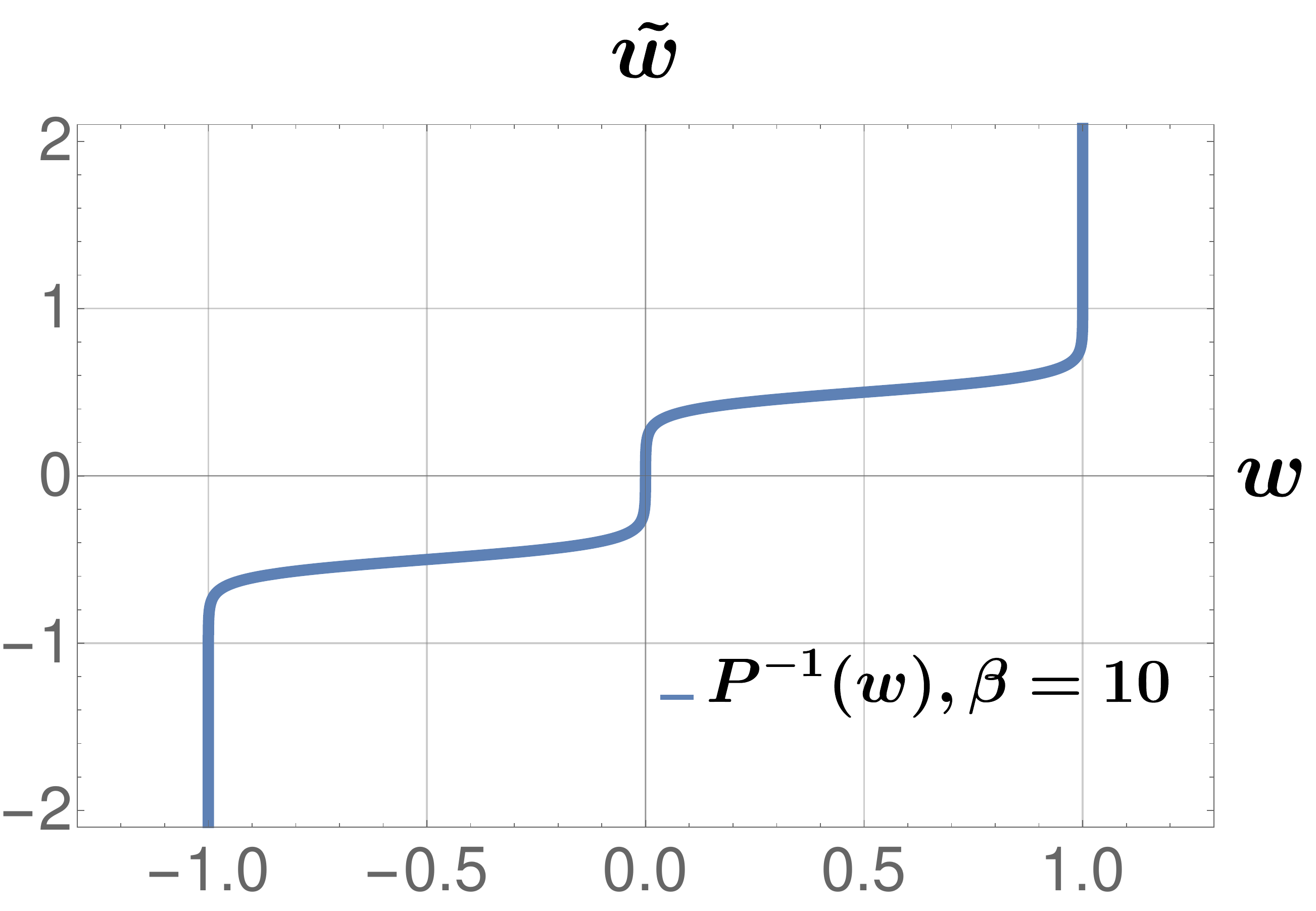}
    \caption{\em Plots of shifted $\tanh$ projections, and their inverses corresponding to the $\tanh$ projection. Note that, the inverse is monotonically increasing. Moreover, when $\beta_k\to \infty$, the projections approaches their respective hard versions. 
    \label{fig:stanh}
    }
\end{figure*}
\begin{exm}[$\bfw$-space, multi-label, shifted $\tanh$]\label{exm:stanh}
Note that, similar to $\softmax$, we wish to extend the $\tanh$ projection beyond binary. 
The idea is to use a function that is an addition of multiple shifted $\tanh$ functions.
\SKIP{
To this end, let us consider a given set of quantization levels as an ordered set $\calQ = \{l_1, \ldots, l_d\}$ with $l_i < l_j$ for all $i < j$.  
Hence, $\bcalC = \calX = [l_1, l_d]$.
Now, we define our shifted $\tanh$ projection $P:\R\to \calC$ as:
\begin{equation}
w = P(\tw) = \frac{1}{d-1}\sum_{i=1}^{d-1} \tanh\left(\beta\left(\tw - \frac{l_i + l_j}{2}\right)\right)\ ,
\end{equation}
where $\beta>0$ and $d= |\calQ|$.
It is easy to see that $w \in \calX$
}
To this end, as an example we consider ternary quantization, with $\calQ=\{-1,0,1\}$ and define our shifted $\tanh$ projection $P_{\beta_k}:\R\to \calC$ as:
\begin{equation}
w = P_{\beta_k}(\tw) = \frac{1}{2}\big[\tanh\left(\beta_k(\tw+0.5)\right) + \tanh\left(\beta_k(\tw-0.5)\right)\big]\ ,
\end{equation}
where $\beta_k \geq 1$ and $w = \bcalC = \calX=[-1,1]$.
When $\beta_k\to \infty$, $P_{\beta_k}$ approaches a stepwise function with inflection points at $-0.5$ and $0.5$ (here, $\pm0.5$ is chosen  heuristically), meaning $w$ move towards one of the quantization levels in the set $\calQ$.

This behaviour together with its inverse is illustrated in~\figref{fig:stanh}.
Now, one could potentially find the functional form of $\iproj_{\beta_k}$ and analytically derive the mirror map corresponding to this projection.
Note that, while Theorem~\myref{1} in the main paper provides an analytical method to derive mirror maps, in some cases such as the above, the exact form of mirror map and the \gls{MD} update might be nontrivial.
In such cases, as shown in the paper, 
the \gls{MD} update can be easily implemented by storing an additional set of auxiliary variables $\tw$.
\end{exm}

\SKIP{
\subsection{\gls{MD} Update Derivation for the tanh Projection}
We now derive the \gls{MD} update corresponding to the $\tanh$ projection below.
From Theorem~\myref{1} in the main paper, the mirror map for the $\tanh$ projection can be written as:
\begin{equation}
\Phi_{\beta_k}(w) = \int \iproj_{\beta_k}(w) dw = \frac{1}{2\beta_k}\big[(1+w)\log(1+w) + (1-w)\log(1-w)\big]\ .  
\end{equation}
Correspondingly, the Bregman divergence can be written as:
 \begin{align}
D_{\Phi_{\beta_k}}(w, v) &= \Phi_{\beta_k}(w) - \Phi_{\beta_k}(v) - \Phi'_{\beta_k}(v) (w-v)\ ,\quad\mbox{where $\Phi_{\beta_k}'(v) = \frac{1}{2\beta_k}\log\frac{1+v}{1-v}$}\ ,\\\nonumber 
 &= \frac{1}{2\beta_k}\left[w\log\frac{(1+w)(1-v)}{(1-w)(1+v)} +  \log(1-w)(1+w) - \log(1-v)(1+v) \right]\ . 
\end{align}
Now, consider the proximal form of \gls{MD} update
\begin{align}\label{eq:tanhmd1}
w^{k+1} & = \amin{\bfx\in(-1,1)}\, \langle \eta\,g^k, w\rangle + D_{\Phi_{\beta_k}}(w, w^k)\ .
\end{align}
The idea is to find $w$ such that the \acrshort{KKT} conditions are satisfied. 
To this end, let us first write the Lagrangian of~\eqref{eq:tanhmd1} by introducing dual variables $y$ and $z$ corresponding to the constraints $w> -1$ and $w< 1$, respectively:
\begin{align}\label{eq:lagtanh1}
F(w, x, y) &= \eta g^kw + y(-w-1) + z(w-1)\\\nonumber &+ \frac{1}{2\beta_k}\left[w\log\frac{(1+w)(1-w^k)}{(1-w)(1+w^k)} +  \log(1-w)(1+w)- \log(1-w^k)(1+w^k) \right]\ .
\end{align}
Now, setting the derivatives with respect to $w$ to zero:
\begin{align}\label{eq:dlagtanh1}
\frac{\partial F}{\partial w} &= \eta g^k + \frac{1}{2\beta_k}\log\frac{(1+w)(1-w^k)}{(1-w)(1+w^k)} - y + z = 0\ .
\end{align}
From complementary slackness conditions,
\begin{align}
y(-w-1) &= 0\ ,\qquad\text{since}\quad w > -1\ \Rightarrow\ y = 0\ ,\\\nonumber
z(w-1) &= 0\ ,\qquad\text{since}\quad w < 1\ \Rightarrow\ z = 0\ .
\end{align} 
Therefore,~\eqref{eq:dlagtanh1} now simplifies to:
\begin{align}
\frac{\partial F}{\partial w} &= \eta g^k + \frac{1}{2\beta_k}\log\frac{(1+w)(1-w^k)}{(1-w)(1+w^k)} = 0\ ,\\\nonumber
\log\frac{(1+w)(1-w^k)}{(1-w)(1+w^k)} &= -2\beta_k\eta g^k\ ,\\\nonumber
 \frac{1+w}{1-w} &= \frac{1+w^k}{1-w^k}\exp(-2\beta_k\eta g^k)\ ,\\\nonumber
 w &= \frac{\frac{1+w^k}{1-w^k}\exp(-2\beta_k\eta g^k) - 1}{\frac{1+w^k}{1-w^k}\exp(-2\beta_k\eta g^k) + 1}\ .
\end{align}
The pseudocodes of original (\acrshort{MD}-$\tanh$) and numerically stable versions (\acrshort{MD}-$\tanh$-\acrshort{S}) for $\tanh$ are presented in~\Twoalgref{alg:mdtanh}{alg:mdtanhste} respectively.
}

\subsection{Convergence Proof for \acrshort{MD} with Adaptive Mirror Maps}

\begin{thm}\label{thm:betamd1}
Let $\calX\subset\R^r$ be a convex compact set and $\calC\subset\R^r$ be a convex open set with $\primal\ne \emptyset$ and $\calX\subset\bcalC$. 
Let $\Phi:\calC\to\R$ be a mirror map $\rho$-strongly convex on $\primal$ with respect to $\norm{\cdot}$, 
$R^2 = \sup_{\bfx\in\primal} \Phi(\bfx) - \Phi(\bfx^0)$ where $\bfx^0 = \argmin_{\bfx\in\primal}\Phi(\bfx)$ is the initialization, 
and $f:\calX\to\R$ be a convex function and $L$-Lipschitz with respect to $\norm{\cdot}$.
Then \gls{MD} with mirror map $\Phi_{\beta_k}(\bfx) = \Phi(\bfx)/\beta_k$ with $1 \le \beta_k \le B$ and $\eta = \frac{R}{L}\sqrt{\frac{2\rho}{Bt}}$ satisfies
\begin{equation}
f\left(\frac{1}{t}\sum_{k=0}^{t-1}\bfx^k\right) - f(\bfx^*) \le RL\sqrt{\frac{2 B}{\rho t}}\ ,
\end{equation}
where $\beta_k$ is the annealing hyperparameter, $\eta >0$ is the learning rate, $t$ is the iteration index, and $\bfx^*$ is the optimal solution.
\end{thm}
\begin{proof}
The proof is a slight modification to the proof of standard \acrshort{MD} and we refer the reader to the proof of Theorem~\myref{4.2} of~\cite{bubeck2015convex} for step by step derivation.
We first discuss the intuition and then turn to the detailed proof.
For the standard \acrshort{MD} the bound is:
\begin{equation}
f\left(\frac{1}{t}\sum_{k=0}^{t-1}\bfx^k\right) - f(\bfx^*) \le RL\sqrt{\frac{2}{\rho t}}\ ,
\end{equation}
with $\eta=\frac{R}{L}\sqrt{\frac{2\rho}{t}}$.
Here, since $\beta_k \le B$, the adaptive mirror map ${\Phi_{\beta_k}(\bfx) = \Phi(\bfx)/\beta_k}$ is $\rho/B$-strongly convex for all $k$.
Therefore, by simply replacing $\rho$ with $\rho/B$ the desired bound is obtained.

We now provide the step-by-step derivation for completeness.
First note the \acrshort{MD} update with the adaptive mirror map:
\begin{align}\label{eq:amd1}
\nabla\Phi_{\beta_k}(\bfy^{k+1}) &= \nabla\Phi_{\beta_k}(\bfx^k) - \eta \bfg^k\ , \quad\mbox{where $\bfg^k\in\partial f(\bfx^k)$ and $\bfy^{k+1}\in\calC$}\ ,\\\nonumber
\bfg^k &= (\nabla\Phi_{\beta_k}(\bfx^k) - \nabla\Phi_{\beta_k}(\bfy^{k+1}))/\eta, \quad\mbox{$\eta >0$}\ .
\end{align}
Now, let $\bfx\in\primal$. The claimed bound will be obtained by taking a limit $\bfx\to \bfx^*$.
\begin{align}\label{eq:mdbound1}
f(\bfx^k) - f(\bfx) &\le \langle \bfg^k, \bfx^k-\bfx\rangle\ ,\quad\mbox{$f$ is convex}\ ,\\\nonumber
&= \langle \nabla\Phi_{\beta_k}(\bfx^k) - \nabla\Phi_{\beta_k}(\bfy^{k+1}), \bfx^k-\bfx\rangle/\eta\ ,\quad\mbox{\eqref{eq:amd1}}\ ,\\\nonumber
&= \left(D_{\Phi_{\beta_k}}(\bfx, \bfx^k) + D_{\Phi_{\beta_k}}(\bfx^k, \bfy^{k+1}) - D_{\Phi_{\beta_k}}(\bfx, \bfy^{k+1})\right)/\eta\ ,\quad\mbox{Bregman div.}\ ,\\\nonumber
&\le \left(D_{\Phi_{\beta_k}}(\bfx, \bfx^k) + D_{\Phi_{\beta_k}}(\bfx^k, \bfy^{k+1}) - D_{\Phi_{\beta_k}}(\bfx, \bfx^{k+1}) - D_{\Phi_{\beta_k}}(\bfx^{k+1}, \bfy^{k+1})\right)/\eta\ .
\end{align}
The last line is due to the inequality ${D_{\Phi_{\beta_k}}(\bfx, \bfx^{k+1}) + D_{\Phi_{\beta_k}}(\bfx^{k+1}, \bfy^{k+1}) \ge D_{\Phi_{\beta_k}}(\bfx, \bfy^{k+1})}$, where $\bfx^{k+1} = \argmin_{\bfx\in\primal}\, D_{\Phi_{\beta_k}}(\bfx, \bfy^{k+1})$.
Notice that, 
\begin{align}\label{eq:telesum1}
\sum_{k=0}^{t-1} D_{\Phi_{\beta_k}}(\bfx, \bfx^k) - D_{\Phi_{\beta_k}}(\bfx, \bfx^{k+1}) &= \sum_{k=0}^{t-1} \left(D_{\Phi}(\bfx, \bfx^k) - D_{\Phi}(\bfx, \bfx^{k+1})\right)/\beta_k\ ,\\\nonumber 
&= \beta_0^{-1}\big(D_\Phi(\bfx, \bfx^0) - D_\Phi(\bfx, \bfx^1)\big) + \beta_1^{-1}\big(D_\Phi(\bfx, \bfx^1) - D_\Phi(\bfx, \bfx^2)\big) +  \ldots + \\\nonumber
& \qquad \beta_{t-1}^{-1}\big(D_\Phi(\bfx, \bfx^{t-1}) - D_\Phi(\bfx, \bfx^t)\big),\\\nonumber
&= \beta_0^{-1}D_\Phi(\bfx, \bfx^0) + (\beta_1^{-1}-\beta_0^{-1})D_\Phi(\bfx, \bfx^1)+ \ldots + (\beta_{t-1}^{-1}-\beta_{t-2}^{-1})D_\Phi(\bfx, \bfx^{t-1}) -  \\\nonumber
& \qquad \beta_{t-1}^{-1}D_\Phi(\bfx, \bfx^t), \\\nonumber
&\le \beta_0^{-1}D_\Phi(\bfx, \bfx^0),\qquad\mbox{$D_{\Phi}(\bfx, \bfz) \ge 0,\ \forall\,\bfx, \bfz\in\calC$ and $\beta_{k+1}^{-1} - \beta_{k}^{-1} < 0, \forall\,k$}\, \\\nonumber
&\le D_\Phi(\bfx, \bfx^0),\qquad\mbox{$\beta_0 \ge 1$}\, \\\nonumber
\end{align}
Now we bound the remaining term:
\begin{align}\label{eq:rhobound1}
&D_{\Phi_{\beta_k}}(\bfx^k, \bfy^{k+1}) - D_{\Phi_{\beta_k}}(\bfx^{k+1}, \bfy^{k+1})\\\nonumber
&= \Phi_{\beta_k}(\bfx^k) - \Phi_{\beta_k}(\bfx^{k+1}) - \langle \nabla\Phi_{\beta_k}(\bfy^{k+1}), \bfx^k - \bfx^{k+1}\rangle\ ,\quad\mbox{Bregman divergence def.}\ ,\\\nonumber
&\le \langle \nabla\Phi_{\beta_k}(\bfx^k) - \nabla\Phi_{\beta_k}(\bfy^{k+1}), \bfx^k - \bfx^{k+1}\rangle - \frac{\rho}{2\beta_k} \norm{\bfx^k - \bfx^{k+1}}^2\ ,\quad\mbox{$\Phi$ is $\rho$-strongly convex}\ ,\\\nonumber
&= \langle \eta\bfg^k, \bfx^k - \bfx^{k+1}\rangle - \frac{\rho}{2\beta_k} \norm{\bfx^k - \bfx^{k+1}}^2\ ,\quad\mbox{\eqref{eq:amd1}}\ ,\\\nonumber
&\le \eta L (\bfx^k - \bfx^{k+1}) - \frac{\rho}{2\beta_k} \norm{\bfx^k - \bfx^{k+1}}^2\ ,\quad\mbox{$f$ is $L$-Lipschitz}\ ,\\\nonumber
&\le \frac{(\eta L)^2\beta_k}{2\rho}\ ,\qquad\mbox{$a z - b z^2 \le a^2/(4b), \ \forall\,z\in\R$}\ ,\\\nonumber
&\le \frac{(\eta L)^2 B}{2\rho}\ ,\qquad\mbox{$\beta_k \le B$}\ .
\end{align}
Putting~\twoeqref{eq:telesum1}{eq:rhobound1} in~\eqref{eq:mdbound1},
\begin{align}\label{eq:sumbound}
\frac{1}{t}\sum_{k=0}^{t-1} \left(f(\bfx^k) - f(\bfx)\right) &\le \frac{D_{\Phi}(\bfx, \bfx^{0})}{\eta t} + \frac{\eta B L^2}{2\rho}\ ,\\\nonumber
 f\left(\frac{1}{t}\sum_{k=0}^{t-1}\bfx^k\right) - f(\bfx) &\le \frac{R^2}{\eta t} + \frac{\eta B L^2}{2\rho}\ ,\quad\mbox{Jensen inequality, defs. of $\bfx^0$ and $R$}\ ,\\\nonumber
 &= RL\sqrt{\frac{2B}{\rho t}}\ ,\quad\mbox{Substituting $\eta = \frac{R}{L}\sqrt{\frac{2\rho}{Bt}}$}\ .
\end{align}
Note the additional multiplication by $\sqrt{B}$ compared to the standard \gls{MD} bound. However, the convergence rate is still $\calO(1/\sqrt{t})$.
\end{proof}

\subsection{Proof for Epsilon Convergence to a Discrete Solution via Annealing}
\begin{pro}
For a given $B > 0$ and $0<\epsilon<1$, there exists a $\gamma>0$ such that if $|\tx|\ge\gamma$ then ${1-|\tanh(B\tx)|<\epsilon}$. Here $|\cdot|$ denotes the absolute value and $\gamma> \tanh^{-1}(1-\epsilon)/B$.
\end{pro}
\begin{proof}
For a given $B$ and $\epsilon$, we derive a condition on $|\tx|$ for the inequality to be satisfied.
\begin{align}
1-|\tanh(B\tx)|&<\epsilon\ ,\\
|\tanh(B\tx)| &> 1 - \epsilon\ ,\\
\tanh(B|\tx|)&> 1 - \epsilon\ ,\quad\mbox{$|\tanh(\tx)| = \tanh(|\tx|)$}\\
|\tx|&> \tanh^{-1}(1 - \epsilon)/B\ .\quad\mbox{$\tanh$ is monotone}
\end{align}  
Therefore for any $\gamma> \tanh^{-1}(1-\epsilon)/B$, the above inequality is satisfied.
\end{proof}

\section{Additional Experiments}
We first give training curves of all compared methods, provide ablation study of \imagenet{} experiments as well as ternary quantization results as a proof of concept. Later, we provide experimental details.

\begin{figure*}[t]
    \centering
    \begin{subfigure}{0.25\linewidth}
    \includegraphics[width=0.99\linewidth]{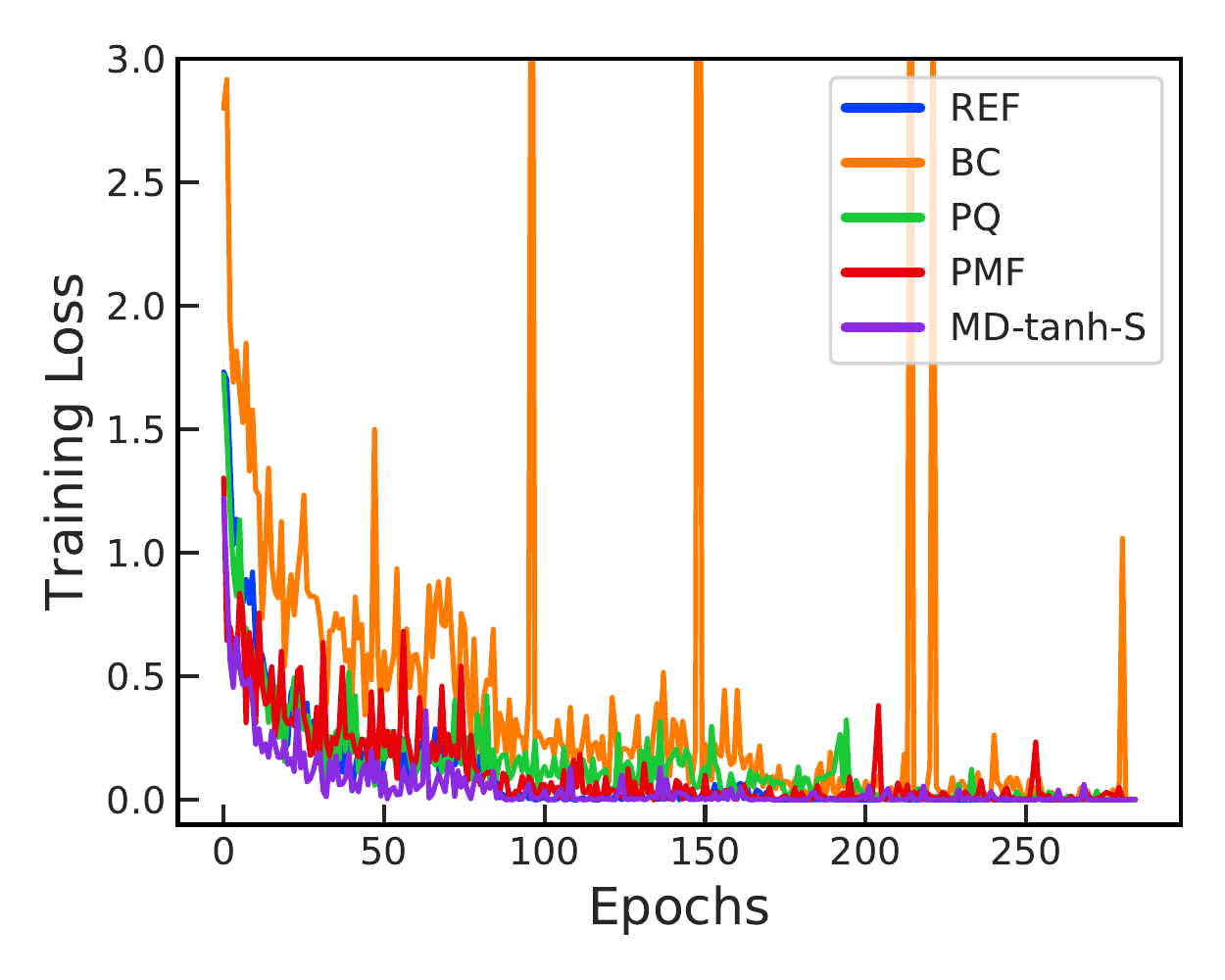}
    \end{subfigure}%
    \begin{subfigure}{0.25\linewidth}
    \includegraphics[width=0.99\linewidth]{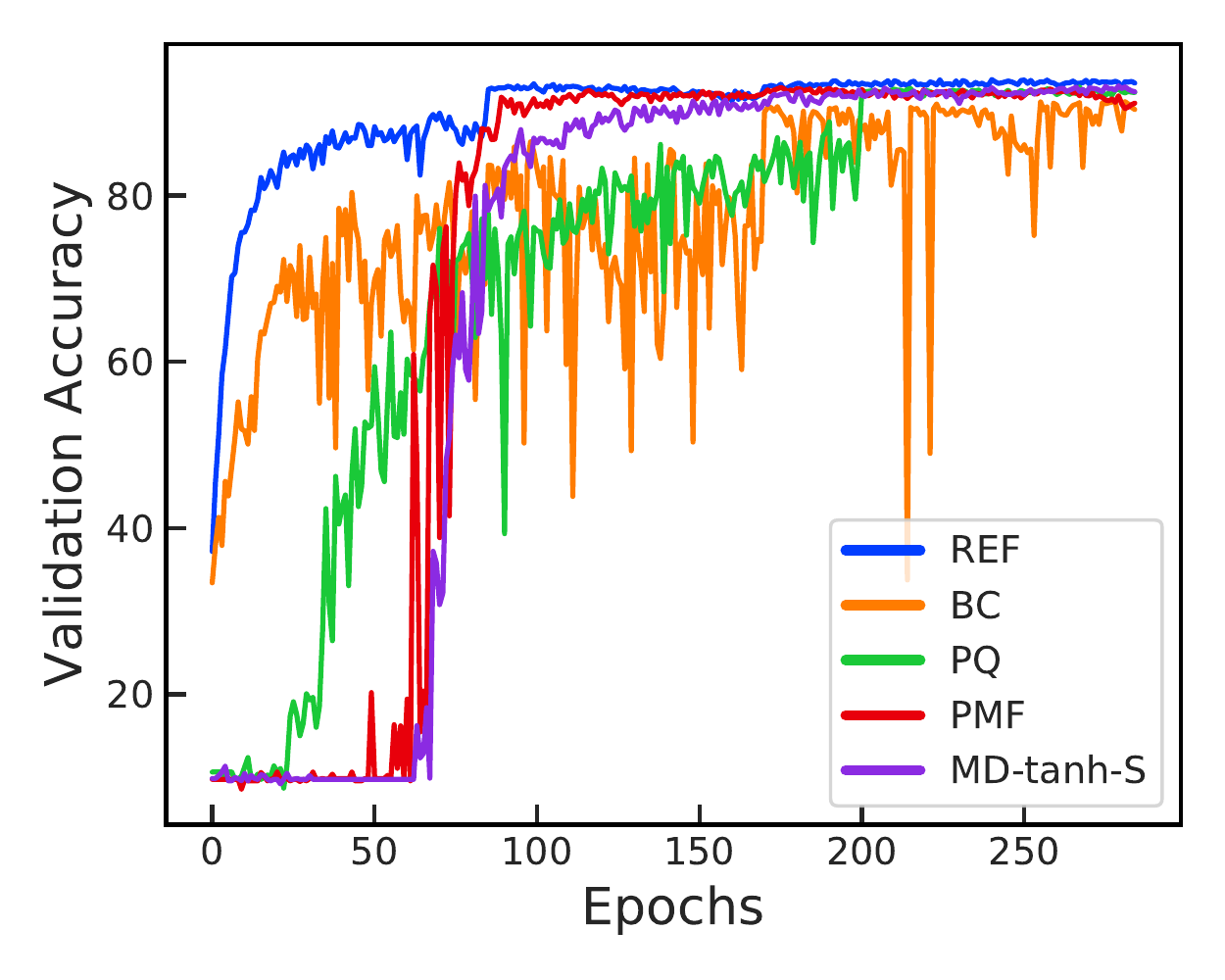}
    \end{subfigure}%
    \begin{subfigure}{0.25\linewidth}
    \includegraphics[width=0.99\linewidth]{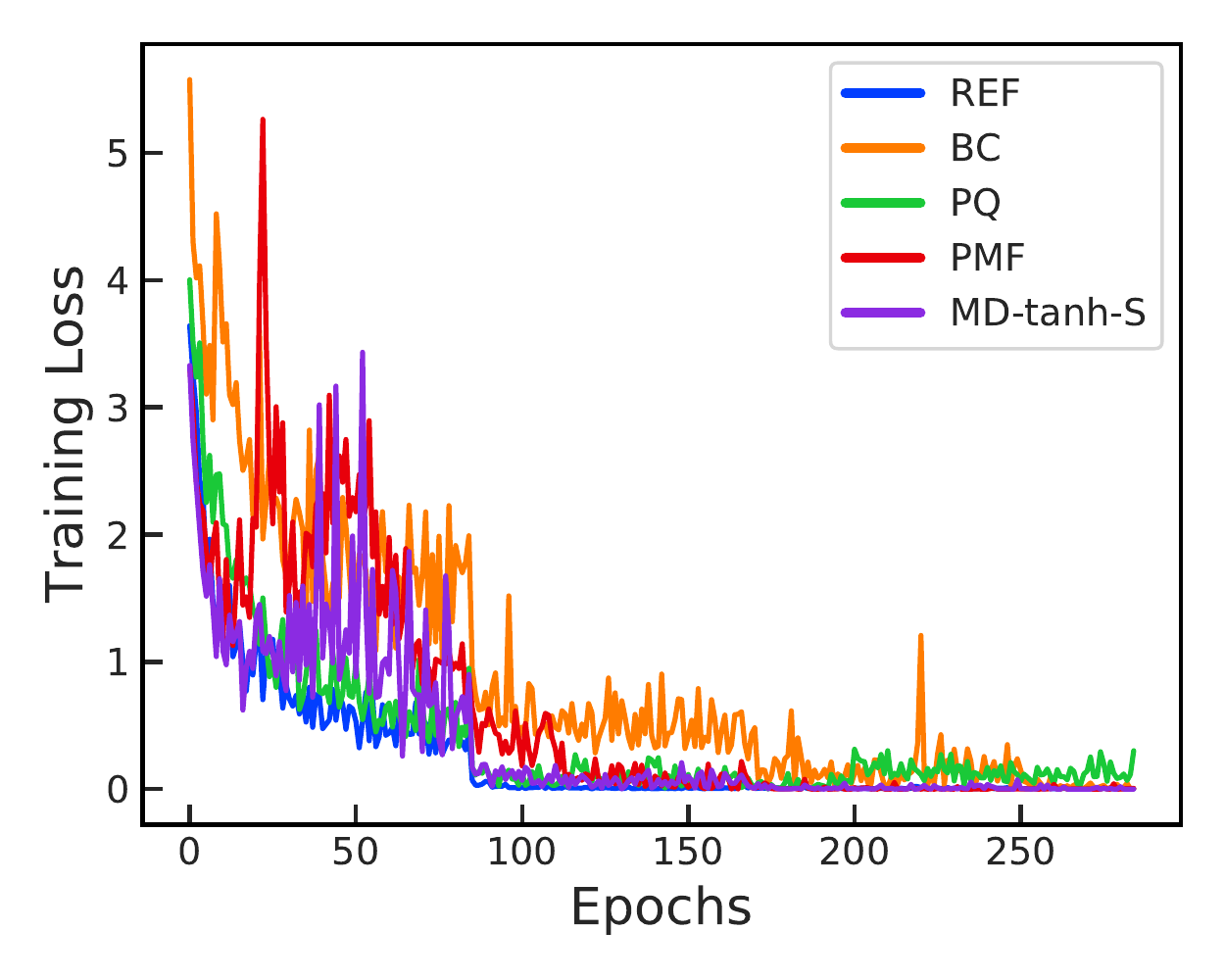}
    \end{subfigure}%
    \begin{subfigure}{0.25\linewidth}
    \includegraphics[width=0.99\linewidth]{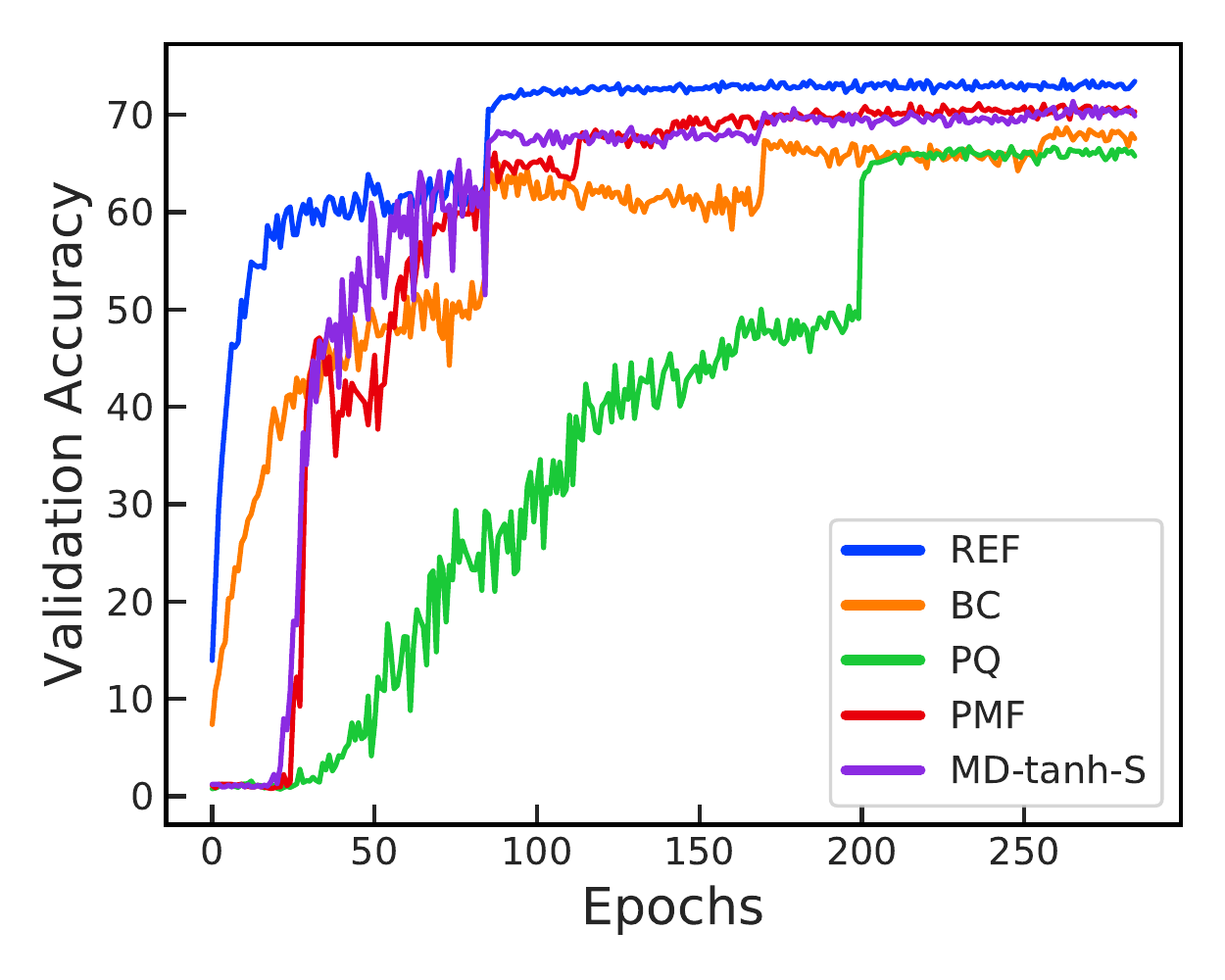}
    \end{subfigure}
    
    \caption{\em Training curves for binarization for \cifar{-10} (first two columns) and \cifar{-100} (last two columns) with \sresnet{-18}.
    Compared to \gls{BC}, our \gls{MD}-$\tanh$-\acrshort{S} and \gls{PMF} are less noisy and after the initial exploration phase (up to $60$ in \cifar{-10} and $25$ epochs \cifar{-100}), the validation accuracies rise sharply and closely resembles the floating point network afterwards. 
    This steep increase is not observed in regularization methods such as~\gls{PQ}.
    }
    \label{fig:trcurves1}
\end{figure*}

\subsection{Convergence Analysis}
The training curves for \cifar{-10} and \cifar{-100} datasets with \sresnet{-18} are shown in~\figref{fig:trcurves1}.
Notice, after the initial exploration phase (due to low $\beta$) the validation accuracies of our \gls{MD}-$\tanh$-\acrshort{S} increase sharply while this steep rise is not observed in regularization methods such as~\gls{PQ}. The training behaviour for both our stable \acrshort{MD}-variants ($\softmax$ and $\tanh$) is quite similar.

\subsection{\imagenet{} Ablation Study}
\input{tables/imagenet-ablation.tex}
We provide an ablation study for various experimental settings for weights binarization on \imagenet{} dataset using \sresnet{-18} architecture in~\tabref{tab:res_imagenet_abl}. We perform experiments for both training for scratch and pretrained networks with variation in binarization of first convolution layer, fully connected layer and biases. Note that the performance degradation of our binary networks is minimal on all layers binarized network except biases using simple layerwise scaling as mentioned in~\cite{mcdonnell2018training}. 
Contrary to the standard setup of binarized network training for \imagenet{}, where first and last layers are kept floating point, our \acrshort{MD}-$\tanh$-\acrshort{S} method achieves good performance even on the fully-quantized network irrespective of either using a pretrained network or network trained from scratch.

\input{tables/results_ternary.tex}
\subsection{Ternary Quantization Results}
As a proof of concept for our shifted $\tanh$ projection (refer~Example~\myref{3}), we also show results for ternary quantization with quantization levels ${\calQ=\{-1,0,1\}}$ in~\tabref{tab:res_ternary}. Note that the performance improvement of our ternary networks compared to their respective binary networks is marginal as only $0$ is included as the $3^{\text{rd}}$ quantization level. In contrast to us, the baseline method \gls{PQ}~\citenew{bai2018proxquant} optimizes for the quantization levels (differently for each layer) as well in an alternating optimization regime rather than fixing it to $\calQ=\{-1,0,1\}$. Also, \acrshort{PQ} does ternarize the first convolution layer, fully-connected layers and the shortcut layers. We crossvalidate hyperparameters for both the original \acrshort{PQ} setup and the equivalent setting of our \acrshort{MD}-variants where we optimize all the weights and denote them as \acrshort{PQ}* and \acrshort{PQ} respectively. 

Our \acrshort{MD}-$\tanh$ variant performs on par or sometimes even better in comparison to $\tanh$ projection results where gradient is calculated through the projection instead of performing \gls{MD}. This again empirically validates the hypothesis that \gls{MD} yields in good approximation for the task of network quantization. The better performance of \gls{PQ} in their original quantization setup, compared to our approach in \cifar{-10} can be accounted to their non-quantized layers and different quantization levels. We believe similar explorations are possible with our \gls{MD} framework as well.
%

\begin{table}[t]
    \centering
    \small
    \begin{tabular}{lccccc}
        \toprule
        Dataset & Image & \# class & Train~/~Val. & $b$ & $K$ \\
        \midrule
        \cifar{-10} & $32\times 32$ & $10$ & $45$k~/~$5$k  & $128$ & $100$k\\
        \cifar{-100} & $32\times 32$ & $100$ & $45$k~/~$5$k  & $128$ & $100$k\\
        \tinyimagenet{} & $64\times 64$ & $200$ & $100$k~/~$10$k  & $128$ & $100$k\\
        ImageNet & $224\times 224$ & $1000$ & $1.2$M~/~$50$k  & $2048/256$ & $90/55$\\
        \bottomrule
    \end{tabular}
    \vspace{1ex}
    \caption{\em
    	Experiment setup. Here, $b$ is the batch size and $K$ is the
    	total number of iterations for all datasets except ImageNet where $K$ indicates number of epochs for training from scratch and pretrained network respectively. For ImageNet, $b$ represents batch size for training from scratch and pretrained networks respectively.}
    \label{tab:setup}
    \end{table}
    
    \begin{table}[t]
    \centering
    \small
        \begin{tabular}{cc}
        \toprule
        Hyperparameters & Fine-tuning grid \\
        \midrule
        learning\_rate & $[0.1, 0.01, 0.001, 0.0001]$\\
        lr\_scale & $[0.1, 0.2, 0.3, 0.5]$\\
        beta\_scale & $[1.01, 1.02, 1.05, 1.1, 1.2]$\\
        beta\_scale\_interval & $[100, 200, 500, 1000, 2000]$\\
        \bottomrule
    \end{tabular}
    \vspace{1ex}
    \caption{\em The hyperparameter search space for all the experiments. Chosen parameters are given in~\threetabref{tab:hyper1}{tab:hyper2}{tab:hyper3}.}
    \label{tab:finetuning}
\end{table}

\subsection{Experimental Details}
As mentioned in the main paper the experimental protocol is similar to~\cite{ajanthan2018pmf}. To this end, the details of the datasets and their corresponding experiment setups are given in~\tabref{tab:setup}. For \cifar{-10/100} and \tinyimagenet{}, \svgg{-16}~\citenew{simonyan2014very}, \sresnet{-18}~\citenew{he2016deep} and \smobilenet{V2}~\citenew{sandler2018mobilenetv2} architectures adapted for \cifar{} dataset are used. 
In particular, for \cifar{} experiments, similar to~\cite{lee2018snip}, the size of the \gls{fc} layers of \svgg{-16} is set to $512$ and no dropout layers are employed. 
For \tinyimagenet{}, the stride of the first convolutional layer of \sresnet{-18} is set to $2$ to handle the image size~\citenew{huang2017snapshot}. 
In all the models, batch normalization~\citenew{ioffe2015batch} (with no learnable parameters) and \srelu{} nonlinearity are used. Only for the floating point networks (\ie, \acrshort{REF}), we keep the learnable parameters for batch norm. Standard data augmentation (\ie, random crop and horizontal flip) is used.

For both of our \gls{MD} variants, hyperparameters such as the learning rate, learning rate scale, annealing hyperparameter $\beta$ and its schedule are crossvalidated from the range reported in~\tabref{tab:finetuning} and the chosen parameters are given in the~\tabref{tab:hyper1}, \tabref{tab:hyper2} and \tabref{tab:hyper3}. 
To generate the plots, we used the publicly available codes of \gls{BC}\footnote{\url{https://github.com/itayhubara/BinaryNet.pytorch}}, \gls{PQ}\footnote{\url{https://github.com/allenbai01/ProxQuant}} and \gls{PMF}\footnote{\url{https://github.com/tajanthan/pmf}}.

All methods are trained from a random initialization and the model with the best validation accuracy is chosen for each method. 
Note that, in \gls{MD}, even though we use an increasing schedule for $\beta$ to enforce a discrete solution, the chosen network may not be fully-quantized (as the best model could be obtained in an early stage of training). 
Therefore, simple $\argmax{}$ rounding is applied to ensure that the network is fully-quantized.
\input{tables/hyper_new.tex}

\input{tables/imagenet_setup.tex}
\subsection{\imagenet{}}
We use the standard \sresnet{-18} architecture for \imagenet{} experiments where we train for $90$ epochs and $55$ epochs for training from scratch and pretrained network respectively. We perform all \imagenet{} experiments using NVIDIA DGX-1 machine with $8$ Tesla V-100 GPUs for training from scratch and single Tesla V-100 GPU for training from a pretrained network. We provide detailed hyperparameter setup used for our experiments in \tabref{tab:hyper_imagenet}. 
Similar to experiments on the other datasets, to enforce a discrete solution simple rounding based on $\sign$ operation is applied to ensure that the final network is fully-quantized.
The final accuracy is reported based on the $\sign$ operation based discrete model obtained at the end of the final epoch.

%% file: tables/imagenet-ablation.tex
\begin{table*}[t]
    \centering
    \small    
    \begin{tabular}{lcccccc}
        \toprule
        & Pretrained & Conv1 and \acrshort{FC} & Bias & \acrshort{BN} & Layerwise Scaling & Accuracy\\ 
        \midrule
        \parbox[t]{2mm}{\multirow{8}{*}{\rotatebox[origin=c]{90}{\acrshort{MD}-$\tanh$-\acrshort{S}}}}
        & \checkmark & Float & Float & Float & \xmark & $\textbf{66.78/87.01}$\\
        & \xmark & Float & Float & Float & \xmark & $65.92/86.29$\\
        & \checkmark & Binary & Float & Float & \xmark & $60.39/82.77$\\
        & \xmark & Binary & Float & Float & \xmark & $59.92/82.42$\\
        & \checkmark & Binary & Binary & Float & \xmark & $56.60/79.79$\\
        & \xmark & Binary & Binary & Float & \xmark & $56.67/79.66$\\
        & \checkmark & Binary & Float & Float & \checkmark & $61.91/83.87$\\
        & \checkmark & Binary & Binary & Float & \checkmark & $56.05/79.69$\\
        \bottomrule
    \end{tabular}
    \caption{\em Ablation study on \imagenet{} with \sresnet{-18} for weights binarization using \mdtanhs{}. While the best performance is obtained for the case where Conv1, \acrshort{FC} and biases are not quantized, \mdtanhs{} obtains good performance even when fully-quantized regardless of either using a pretrained network or training from scratch.
    }        
    \label{tab:res_imagenet_abl}
\end{table*}

%% file: tables/results_ternary.tex
\begin{table*}[h]
    \centering
    \small    
    \begin{tabular}{llccccccc}
        \toprule
        &\multirow{2}{*}{Algorithm}   & \multirow{2}{*}{Space} & \multicolumn{2}{c}{\cifar{-10}} & \multicolumn{2}{c}{\cifar{-100}} & \tinyimagenet{}\\
         &&  &  \svgg{-16}  & \sresnet{-18} & \svgg{-16}  & \sresnet{-18}  & \sresnet{-18} \\
        \midrule

        &\acrshort{REF} (float) & $\bfw$ & $93.33$ & $94.84$ & $71.50$ & $76.31$ & $58.35$\\
        &\acrshort{PQ}  & $\bfw$ & $83.32$ & $90.50$ & $32.16$ & $59.18$ & $41.46$ \\
        &\acrshort{PQ}*  & $\bfw$ & $\textbf{92.20}$ & $\textbf{93.85}$ & $57.64$ & $70.98$ & $45.72$ \\

        & \acrshort{GD}-$\tanh$ & $\bfw$ & $91.21$ & $93.20$ & $53.88$ & $69.48$ & $50.65$  \\
        \midrule

        \parbox[t]{2mm}{\multirow{2}{*}{\rotatebox[origin=c]{90}{Ours}}}

        &\acrshort{MD}-$\softmax$-\acrshort{S} & $\bfu$ & $91.69$ & $93.30$  & $65.11$ & $\textbf{72.01}$ & $52.21$ \\
        &\acrshort{MD}-$\tanh$-\acrshort{S} & $\bfw$ & $91.70$ & $93.42$ & $\textbf{66.15}$ & $71.29$  & $\textbf{52.69}$\\

        \bottomrule
    \end{tabular}
    \vspace{1ex}
    \caption{\em Classification accuracies on the test set for ternary quantization. \acrshort{PQ}* denotes performance with fully-connected layers, first convolution layer, and shortcut layers in floating point whereas \acrshort{PQ} represent results with all layers quantized. Also, \gls{PQ}* optimize for the quantization levels as well (different for each layer), in contrast we fix it to $\calQ=\{-1,0,1\}$. \acrshort{GD}-$\tanh$ denotes results without using \acrshort{STE} and actually calculating the gradient through the projection.}
        
    \label{tab:res_ternary}
\end{table*}

%% file: tables/hyper_new.tex
\begin{table*}[]
\centering
\small
\begin{tabular}{l|cccccccc}
\toprule

\multirow{2}{*}{} & \multicolumn{8}{c}{\vcifar{-10} with \vresnet{-18}} \\
 & \acrshort{MD}-$\softmax$ & \acrshort{MD}-$\tanh$ & \acrshort{MD}-$\softmax$-\acrshort{S} & \acrshort{MD}-$\tanh$-\acrshort{S} & \acrshort{PMF}* & \acrshort{GD}-$\tanh$ & \acrshort{BC} & \acrshort{PQ} \\

\midrule
learning\_rate       & 0.001                      & 0.001                   & 0.001                        & 0.001                     & 0.001  & 0.001                   & 0.001  & 0.01    \\
lr\_scale            & 0.2                        & 0.3                     & 0.3                          & 0.3                       & 0.3    & 0.3                     & 0.3    & 0.5     \\
beta\_scale          & 1.02                       & 1.01                    & 1.02                         & 1.02                      & 1.1    & 1.1                     & -      & 0.0001  \\
beta\_scale\_interval & 200                        & 100                     & 200                          & 200                       & 1000   & 1000                    & -      & -       \\

\midrule

\multirow{2}{*}{} & \multicolumn{8}{c}{\vcifar{-100} with \vresnet{-18}} \\
 & \acrshort{MD}-$\softmax$ & \acrshort{MD}-$\tanh$ & \acrshort{MD}-$\softmax$-\acrshort{S} & \acrshort{MD}-$\tanh$-\acrshort{S} & \acrshort{PMF}* & \acrshort{GD}-$\tanh$ & \acrshort{BC} & \acrshort{PQ} \\
 \midrule 
learning\_rate       & 0.001                      & 0.001                   & 0.001                        & 0.001                     & 0.001  & 0.001                   & 0.001  & 0.1     \\
lr\_scale            & 0.2                        & 0.3                     & 0.2                          & 0.2                       & 0.3    & 0.5                     & 0.2    & -       \\
beta\_scale          & 1.05                       & 1.05                    & 1.1                          & 1.2                       & 1.01   & 1.01                    & -      & 0.001   \\
beta\_scale\_interval & 500                        & 500                     & 200                          & 500                       & 100    & 100                     & -      & -       \\

\midrule

\multicolumn{1}{c|}{\multirow{2}{*}{}} & \multicolumn{8}{c}{\vcifar{-10} with \vvgg{-16}} \\
 & \acrshort{MD}-$\softmax$ & \acrshort{MD}-$\tanh$ & \acrshort{MD}-$\softmax$-\acrshort{S} & \acrshort{MD}-$\tanh$-\acrshort{S} & \acrshort{PMF}* & \acrshort{GD}-$\tanh$ & \acrshort{BC} & \acrshort{PQ} \\
 \midrule
learning\_rate       & 0.01                       & 0.001                   & 0.001                        & 0.001                     & 0.001  & 0.001                   & 0.0001 & 0.01    \\
lr\_scale            & 0.2                        & 0.3                     & 0.3                          & 0.2                       & 0.5    & 0.3                     & 0.3    & 0.5     \\
beta\_scale          & 1.05                       & 1.1                     & 1.2                          & 1.2                       & 1.05   & 1.1                     & -      & 0.0001  \\
beta\_scale\_interval & 500                        & 1000                    & 2000                         & 2000                      & 500    & 1000                    & -      & -       \\

\midrule

\multicolumn{1}{c|}{\multirow{2}{*}{}} & \multicolumn{8}{c}{\vcifar{-100} with \vvgg{-16}} \\
 & \acrshort{MD}-$\softmax$ & \acrshort{MD}-$\tanh$ & \acrshort{MD}-$\softmax$-\acrshort{S} & \acrshort{MD}-$\tanh$-\acrshort{S} & \acrshort{PMF}* & \acrshort{GD}-$\tanh$ & \acrshort{BC} & \acrshort{PQ} \\
\midrule
learning\_rate       & 0.001                      & 0.001                   & 0.0001                       & 0.001                     & 0.0001 & 0.001                   & 0.0001 & 0.01    \\
lr\_scale            & 0.3                        & 0.3                     & 0.2                          & 0.5                       & 0.5    & 0.5                     & 0.2    & 0.5     \\
beta\_scale          & 1.01                       & 1.05                    & 1.2                          & 1.05                      & 1.02   & 1.1                     & -      & 0.0001  \\
beta\_scale\_interval & 100                        & 500                     & 500                          & 500                       & 200    & 1000                    & -      & -       \\

\midrule

\multirow{2}{*}{} & \multicolumn{8}{c}{\tinyimagenet{} with \vresnet{-18}} \\
 & \acrshort{MD}-$\softmax$ & \acrshort{MD}-$\tanh$ & \acrshort{MD}-$\softmax$-\acrshort{S} & \acrshort{MD}-$\tanh$-\acrshort{S} & \acrshort{PMF}* & \acrshort{GD}-$\tanh$ & \acrshort{BC} & \acrshort{PQ} \\
 \midrule

learning\_rate       & 0.001                      & 0.001                   & 0.001                        & 0.001                     & 0.001  & 0.001                   & 0.001  & 0.01    \\
lr\_scale            & 0.2                        & 0.5                     & 0.1                          & 0.1                       & 0.5    & 0.5                     & 0.5    & -       \\
beta\_scale   & 1.02                       & 1.2                     & 1.02                         & 1.2                       & 1.01   & 1.01                    & -      & 0.0001  \\
beta\_scale\_interval & 200                        & 2000                    & 100                          & 500                       & 100    & 100                     & -      & -       \\

\bottomrule
\end{tabular}
\vspace{1ex}
    \caption{\em Hyperparameter settings used for the binary quantization experiments. 
	Here, the learning rate is multiplied by {\em lr\_scale} after every 30k iterations and annealing hyperparameter $(\beta)$ is multiplied by {\em beta\_scale} after every {\em beta\_scale\_interval} iterations. We use Adam optimizer with zero weight decay. For \acrshort{PQ}, {\em beta\_scale} denotes regularization rate.
	}
    \label{tab:hyper1}
\end{table*}


\begin{table*}[]
\small
\centering
\begin{tabular}{l|cccccc}
\toprule

\multirow{2}{*}{} & \multicolumn{6}{c}{\vcifar{-10} with \vresnet{-18}} \\
 & \acrshort{REF} (float) & \acrshort{MD}-$\softmax$-\acrshort{S} & \acrshort{MD}-$\tanh$-\acrshort{S} & \acrshort{GD}-$\tanh$ & \acrshort{PQ} & \acrshort{PQ}* \\

\midrule
learning\_rate       & 0.1           & 0.001                        & 0.01                      & 0.01                    & 0.01     & 0.01    \\
lr\_scale            & 0.3           & 0.3                          & 0.2                       & 0.5                     & 0.3      & -       \\
beta\_scale   & -             & 1.05                         & 1.2                       & 1.02                    & 0.0001   & 0.0001  \\
beta\_scale\_interval & -             & 500                          & 1000                      & 500                     & -        & -       \\
weight\_decay        & 0.0001        & 0                            & 0                         & 0                       & 0        & 0.0001  \\

\midrule

\multirow{2}{*}{} & \multicolumn{6}{c}{\vcifar{-100} with \vresnet{-18}} \\
 & \acrshort{REF} (float) & \acrshort{MD}-$\softmax$-\acrshort{S} & \acrshort{MD}-$\tanh$-\acrshort{S} & \acrshort{GD}-$\tanh$ & \acrshort{PQ} & \acrshort{PQ}* \\
 \midrule 
learning\_rate       & 0.1           & 0.001                        & 0.001                     & 0.01                    & 0.01     & 0.001   \\
lr\_scale            & 0.1           & 0.1                          & 0.5                       & 0.5                     & 0.2      & -       \\
beta\_scale   & -             & 1.1                          & 1.1                       & 1.02                    & 0.0001   & 0.0001  \\
beta\_scale\_interval & -             & 100                          & 500                       & 1000                    & -        & -       \\
weight\_decay        & 0.0001        & 0                            & 0                         & 0                       & 0        & 0.0001  \\

\midrule

\multicolumn{1}{c|}{\multirow{2}{*}{}} & \multicolumn{6}{c}{\vcifar{-10} with \vvgg{-16}} \\
 & \acrshort{REF} (float) & \acrshort{MD}-$\softmax$-\acrshort{S} & \acrshort{MD}-$\tanh$-\acrshort{S} & \acrshort{GD}-$\tanh$ & \acrshort{PQ} & \acrshort{PQ}* \\
 \midrule
learning\_rate       & 0.1           & 0.001                        & 0.01                      & 0.01                    & 0.01     & 0.1     \\
lr\_scale            & 0.2           & 0.3                          & 0.3                       & 0.3                     & -        & -       \\
beta\_scale   & -             & 1.05                         & 1.1                       & 1.01                    & 1e-07 & 0.0001  \\
beta\_scale\_interval & -             & 500                          & 1000                      & 500                     & -        & -       \\
weight\_decay        & 0.0001        & 0                            & 0                         & 0                       & 0        & 0.0001  \\

\midrule

\multicolumn{1}{c|}{\multirow{2}{*}{}} & \multicolumn{6}{c}{\vcifar{-100} with \vvgg{-16}} \\
 & \acrshort{REF} (float) & \acrshort{MD}-$\softmax$-\acrshort{S} & \acrshort{MD}-$\tanh$-\acrshort{S} & \acrshort{GD}-$\tanh$ & \acrshort{PQ} & \acrshort{PQ}* \\
 \midrule
learning\_rate       & 0.1           & 0.0001                       & 0.001                     & 0.01                    & 0.01     & 0.0001  \\
lr\_scale            & 0.2           & 0.3                          & 0.5                       & 0.2                     & -        & -       \\
beta\_scale   & -             & 1.05                         & 1.1                       & 1.05                    & 0.0001   & 0.0001  \\
beta\_scale\_interval & -             & 100                          & 500                       & 2000                    & -        & -       \\
weight\_decay        & 0.0001        & 0                            & 0                         & 0                       & 0        & 0.0001  \\

\midrule

\multirow{2}{*}{} & \multicolumn{6}{c}{\tinyimagenet{} with \vresnet{-18}} \\
 & \acrshort{REF} (float) & \acrshort{MD}-$\softmax$-\acrshort{S} & \acrshort{MD}-$\tanh$-\acrshort{S} & \acrshort{GD}-$\tanh$ & \acrshort{PQ} & \acrshort{PQ}* \\
 \midrule

learning\_rate       & 0.1           & 0.001                        & 0.01                      & 0.01                    & 0.01     & 0.01    \\
lr\_scale            & 0.1           & 0.1                          & 0.1                       & 0.5                     & -        & -       \\
beta\_scale   & -             & 1.2                          & 1.2                       & 1.05                    & 0.01     & 0.0001  \\
beta\_scale\_interval & -             & 500                          & 2000                      & 2000                    & -        & -       \\
weight\_decay        & 0.0001        & 0                            & 0                         & 0                       & 0        & 0.0001  \\

\bottomrule
\end{tabular}
\vspace{1ex}
    \caption{\em Hyperparameter settings used for the ternary quantization experiments. 
	Here, the learning rate is multiplied by {\em lr\_scale} after every 30k iterations and annealing hyperparameter $(\beta)$ is multiplied by {\em beta\_scale} after every {\em beta\_scale\_interval} iterations. We use Adam optimizer except for \acrshort{REF} for which \acrshort{SGD} with momentum $0.9$ is used. For \acrshort{PQ}, {\em beta\_scale} denotes regularization rate.
	}
    \label{tab:hyper2}
\end{table*}


\begin{table*}[]
\small
\centering
\begin{tabular}{l|cccc}
\toprule

\multirow{2}{*}{} & \multicolumn{4}{c}{\vcifar{-10} with \smobilenet{-V2}} \\
 & \acrshort{REF} (float) & \acrshort{BC} & \acrshort{MD}-$\softmax$-\acrshort{S} & \acrshort{MD}-$\tanh$-\acrshort{S} \\

\midrule
learning\_rate       & 0.01           & 0.001                        & 0.01                      & 0.01    \\
lr\_scale            & 0.5           & 0.5                          & 0.3                       & 0.2     \\
beta\_scale   & -             & -                         & 1.1                       & 1.2    \\
beta\_scale\_interval & -             & -                          & 1000                      & 2000    \\
weight\_decay        & 0.0001        & 0                            & 0                         & 0       \\

\midrule

\multirow{2}{*}{} & \multicolumn{4}{c}{\vcifar{-100} with \smobilenet{-V2}} \\
 & \acrshort{REF} (float)  & \acrshort{BC} & \acrshort{MD}-$\softmax$-\acrshort{S} & \acrshort{MD}-$\tanh$-\acrshort{S} \\
 \midrule 
learning\_rate       & 0.01           & 0.001                        & 0.01                     & 0.01    \\
lr\_scale            & 0.5           & 0.2                          & 0.1                       & 0.1     \\
beta\_scale   & -             & -                          & 1.1                       & 1.02    \\
beta\_scale\_interval & -             & -                          & 500                       & 100   \\
weight\_decay        & 0.0001        & 0                            & 0                         & 0       \\

\bottomrule
\end{tabular}
\caption{\em Hyperparameter settings used for the binary quantization experiments. 
	Here, the learning rate is multiplied by {\em lr\_scale} after every 30k iterations and annealing hyperparameter $(\beta)$ is multiplied by {\em beta\_scale} after every {\em beta\_scale\_interval} iterations. We use Adam optimizer except for \acrshort{REF} for which \acrshort{SGD} with momentum $0.9$ is used.
	}
    \label{tab:hyper3}
\end{table*}

%% file: tables/imagenet_setup.tex
\begin{table*}[!t]
\small
\centering
\begin{tabular}{l|ccc}
\toprule

\multirow{2}{*}{} & \multicolumn{3}{c}{\imagenet{} with \sresnet{-18}} \\
 & \acrshort{REF} (float) & \mdtanhs{}* & \mdtanhs{} \\

\midrule
base\_learning\_rate       & 2.048           & 2.048                        & 0.0768         \\
warmup\_epochs      & 8           & 8                          & 0           \\
beta\_scale  & -             & 1.02                         & 1.02              \\
beta\_scale\_interval (iterations) & -             & 62                          & 275           \\
batch\_size          & 2048             & 2048                          & 256           \\
weight\_decay        & 3.0517e-05        & 3.0517e-05                            & 3.0517e-05             \\
\bottomrule
\end{tabular}
    \caption{\em Hyperparameter settings used for the binary quantization experiments on \imagenet{} dataset using \sresnet{-18} architecture. Here \mdtanhs{}* is trained from scratch while \mdtanhs{} is finetuned on the pretrained network. We use~\acrshort{SGD} optimizer with momentum $0.875$ and cosine learning rate scheduler for all experiments. For all the experiments, weight decay in batchnorm layers are off. Similar to~\cite{goyal2017accurate}, for experiments with larger batch size we use gradual warmup where learning rate is linearly scaled from small learning rate to the base learning rate. Also, note that training schedule is fixed based on above hyperparameters for ablation study on \imagenet{} dataset.}
    \label{tab:hyper_imagenet}
\end{table*}

%% file: text/abstract.tex
\begin{abstract}
Quantizing large Neural Networks~(\acrshort{NN}) while maintaining the performance is highly desirable for resource-limited devices due to reduced memory and time complexity. 
It is usually formulated as a constrained optimization problem and optimized via a modified version of gradient descent.
In this work, by interpreting the continuous parameters (unconstrained) as the dual of the quantized ones,
we introduce a \acrfull{MD} framework~\citenew{bubeck2015convex} for \acrshort{NN} quantization.
Specifically, we provide conditions on the projections (\ie, mapping from continuous to quantized ones) which would enable us to derive valid mirror maps and in turn the respective \acrshort{MD} updates. 
Furthermore, we present a numerically stable implementation of \acrshort{MD} that requires storing an additional set of auxiliary variables (unconstrained), and show that it is strikingly analogous to the \acrfull{STE} based method which is typically viewed as a ``trick'' to avoid vanishing gradients issue. 
Our experiments on \cifar{-10/100}, \tinyimagenet{}, and \imagenet{} classification datasets with \svgg{-16}, \sresnet{-18}, and \smobilenet{V2} architectures show that our \acrshort{MD} variants yield
state-of-the-art performance. 
\end{abstract}

%% file: text/intro.tex
\section{Introduction}
Despite the success of deep neural networks in various domains, their excessive computational and memory requirements limit their practical usability for real-time applications or in resource-limited devices.
Quantization is a prominent technique for network compression, where the objective is to learn a network while restricting the parameters (and activations) to take values from a small discrete set. 
This leads to a dramatic reduction in memory (a factor of $32$ for binary quantization) and inference time -- as it enables specialized implementation using bit operations. 

\gls{NN} quantization is usually formulated as a constrained optimization problem ${\min_{\bfx\in\calX} f(\bfx)},$ where $f(\cdot)$ denotes the loss function 
by abstracting out the dependency on the dataset 
and $\calX\subset\R^r$ denotes the set of all possible quantized solutions.
Majority of the works in the literature~\citenew{ajanthan2018pmf,hubara2017quantized,yin2018binaryrelax} convert this into an unconstrained problem by introducing auxiliary variables ($\tbfx$) and optimize via (stochastic) gradient descent. 
Specifically, the objective and the update step take the following form:
\begin{equation}\label{eq:hgdintro}
    \min_{\tbfx\in\R^r} f(P(\tbfx))\ ,\quad
    \tbfx^{k+1} = \tbfx^k - \eta
    \left.
    \nabla_{\tbfx}f(P(\tbfx))
    \right|_{\tbfx=\tbfx^k}\ ,\\[-0.5ex]
\end{equation}
where $P:\R^r\to \calX$ is a mapping from the unconstrained space to the quantized space (sometimes called projection) and $\eta>0$ is the learning rate.
In cases where the mapping $P$ is not differentiable, a suitable approximation is employed~\citenew{hubara2017quantized}. 

In this work, by noting that the well-known \gls{MD} algorithm, widely used for online convex optimization~\citenew{bubeck2015convex}, provides a theoretical framework to perform gradient descent in the unconstrained space (dual space, $\R^r$) with gradients computed in the quantized space (primal space, $\calX$), we introduce an \gls{MD} framework for \gls{NN} quantization.
In essence, \gls{MD} extends gradient descent to non-Euclidean spaces where Euclidean projection is replaced with a more general projection defined based on the associated distance metric.
Briefly, the key ingredient of \gls{MD} is a concept called {\em mirror map} which defines both the mapping between primal and dual spaces and the exact form of the projection.
Specifically, in this work, by observing $P$ in~\eqref{eq:hgdintro} as a mapping from dual space to the primal space, we analytically derive corresponding mirror maps under certain conditions on $P$. 
This enables us to derive different variants of the \gls{MD} algorithm useful for \gls{NN} quantization.

Note that, \gls{MD} requires the constrained set to be convex, however, the quantization set is discrete.
Therefore, as discussed later in \secref{sec:mdannq}, to ensure quantized solutions, we employ a monotonically increasing annealing hyperparameter similar to~\cite{ajanthan2018pmf,bai2018proxquant}. 
This translates into {\em time-varying mirror maps}, and for completeness, we theoretically analyze the convergence behaviour of \gls{MD} in this case for the convex setting. 
Furthermore, as \gls{MD} is often found to be numerically unstable~\citenew{hsieh2018mirrored}, we discuss a numerically stable implementation of \gls{MD} by storing an additional set of auxiliary variables. 
This update is strikingly analogous to the popular \gls{STE} based gradient method~\citenew{bai2018proxquant,hubara2017quantized} which is typically viewed as a ``trick'' to avoid vanishing gradients issue but here we show that it is an implementation method for \gls{MD} under certain conditions on the mapping $P$.
We believe this connection sheds some light on the practical effectiveness of \gls{STE}. 

In summary, we make the following contributions:
\begin{itemize}
    \item We introduce an \gls{MD} framework with time-varying mirror maps for \gls{NN} quantization by deriving mirror maps from projections ($P$ in~\eqref{eq:hgdintro}) and present two \acrshort{MD} algorithms for quantization. 
    
    \item Theoretically, we first show that \gls{MD} with time-varying mirror maps converges at the same rate as the standard \gls{MD} in the convex setting. Second, we discuss conditions for the convergence to a discrete solution when a monotonically increasing annealing hyperparameter is employed.
    
    \item For practical usability, we introduce a numerically stable implementation of \gls{MD} and show its connection to the popular \gls{STE} approximation. 
    
    \item With extensive experiments on \cifar{-10/100}, \tinyimagenet{}, and \imagenet{} classification datasets using \svgg{-16}, \sresnet{-18}, and \smobilenet{V2} architectures we demonstrate that our \gls{MD} variants yield state-of-the-art performance.
\end{itemize}

%% file: text/prelim.tex
\section{Preliminaries}
Here we provide a brief background on the \gls{MD} algorithm and \gls{NN} quantization. 
\input{text/mda.tex}

\input{text/nnq.tex}

%% file: text/mda.tex
\subsection{\acrlong{MD}}\label{sec:mda}
The \acrfull{MD} algorithm was first introduced in~\cite{nemirovsky1983problem} and has extensively been studied in the convex optimization literature ever since.
In this section, we provide a brief overview and refer the interested reader to Chapter~\myref{4} of~\cite{bubeck2015convex}. 
In the context of \gls{MD}, we consider a problem of the form:
\begin{equation}\label{eq:mdobj}
\min_{\bfx\in\calX}\ f(\bfx)\ ,
\end{equation}
where $f:\calX\to\R$ is a convex function and ${\calX\subset\R^r}$ is a compact convex set. 
The main concept of \gls{MD} is to extend gradient descent to a more general non-Euclidean space (Banach space\footnote{A Banach space is a complete normed vector space where the norm is not necessarily derived from an inner product.}), thus overcoming the dependency of gradient descent on the Euclidean geometry.
The motivation for this generalization is that one might be able to exploit the geometry of the space to optimize much more efficiently. One such example is the simplex constrained optimization where \gls{MD} converges at a much faster rate than the standard \gls{PGD}.

To this end, since the gradients lie in the dual space, optimization is performed by first mapping the primal point $\bfx^k\in\calB$ (quantized space, $\calX$) to the dual space $\calB^*$ (unconstrained space, $\R^r$), then performing gradient descent in the dual space, and finally mapping back the resulting point to the primal space $\calB$.
If the new point $\bfx^{k+1}$ lie outside of the constraint set $\calX\subset \calB$, it is projected to the set $\calX$. 
Both the primal/dual mapping and the projection are determined by the {\em mirror map}. Specifically, the gradient of the mirror map defines the mapping from primal to dual and the projection is done via the Bregman divergence of the mirror map.
We first provide the definitions for mirror map and Bregman divergence and then turn to the \gls{MD} updates.
  
\begin{dfn}[Mirror map]\label{dfn:mm}
Let $\calC\subset \R^r$ be a convex open set such that $\calX\subset\bcalC$ ($\bcalC$ denotes the closure of set $\calC$) and $\primal\ne \emptyset$.
Then, $\Phi:\calC\to\R$ is a mirror map if it satisfies:
\vspace{-1ex}
\begin{tight_enumerate}
  \item $\Phi$ is strictly convex and differentiable.
  \item $\nabla\Phi(\calC) = \R^r$, \ie, $\nabla\Phi$ takes all possible values in $\R^r$. 
  \item $\lim_{\bfx\to\partial \calC}\|\nabla\Phi(\bfx)\| = \infty$ ($\partial \calC$ denotes the boundary of $\calC$), \ie, $\nabla\Phi$ diverges on the boundary of $\calC$.
\end{tight_enumerate}
\end{dfn}

\begin{dfn}[Bregman divergence]\label{dfn:bg}
Let $\Phi:\calC\to\R$ be a continuously differentiable, strictly convex function defined on a convex set $\calC$. 
The Bregman divergence associated with $\Phi$ for points $\bfp, \bfq\in\calC$ is the difference between the value of $\Phi$ at point $\bfp$ and the value of the 
first-order Taylor expansion of $\Phi$ around point $\bfq$ evaluated at point $\bfp$, \ie,
\vspace{-1ex} 
\begin{equation}\label{eq:bg}
D_\Phi(\bfp, \bfq) = \Phi(\bfp) - \Phi(\bfq) - \left\langle \nabla \Phi(\bfq), \bfp-\bfq\right\rangle\ .\\[-1ex]
\end{equation}
Notice, $D_\Phi(\bfp, \bfq)\ge 0$ with $D_\Phi(\bfp, \bfp)= 0$, and $D_\Phi(\bfp, \bfq)$ is convex on $\bfp$.
\end{dfn}
%
Now we are ready to provide the mirror descent strategy based on the mirror map $\Phi$. 
Let ${\bfx^0\in\argmin_{\bfx\in\primal}}\, \Phi(\bfx)$ be the initial point. 
Then, for iteration $k\ge0$ and step size $\eta>0$, the update of the \gls{MD} algorithm can be written as:
\vspace{-1ex}
\begin{align}\label{eq:md1}
\nabla\Phi(\bfy^{k+1}) &= \nabla\Phi(\bfx^k) - \eta\,\bfg^k\ , \\\nonumber
\bfx^{k+1} &= \amin{\bfx\in\primal}\, D_{\Phi}(\bfx, \bfy^{k+1})\ ,\\[-4.5ex]\nonumber
\end{align}
where $\bfg^k\in\partial f(\bfx^k)$ and $\bfy^{k+1}\in\calC$. Note that, in~\eqref{eq:md1}, the gradient $\bfg^k$ is computed at $\bfx^k\in\primal$ (solution space) but the gradient descent is performed in $\R^r$ (unconstrained dual space).
Moreover, by simple algebraic manipulation, it is easy to show that the above \gls{MD} update~\plaineqref{eq:md1} can be compactly written in a proximal form where the Bregman divergence of the mirror map becomes the proximal term~\citenew{beck2003mirror}:
\vspace{-1ex}
\begin{align}\label{eq:mdprox}
\bfx^{k+1} &= \amin{\bfx\in\primal}\, \langle \eta\,\bfg^k, \bfx\rangle + D_{\Phi}(\bfx, \bfx^k)\ .
\end{align}
%
Note, if $\Phi(\bfx) = \frac{1}{2}\norm{\bfx}_2^2$, then $D_{\Phi}(\bfx,\bfx^k) = \frac{1}{2}\norm{\bfx - \bfx^k}_2^2$, which when plugged back to the above problem and optimized for $\bfx$, leads to exactly the same update rule as that of \gls{PGD}. However, \gls{MD} allows us to choose various forms of $\Phi$ depending on the problem at hand.

%% file: text/nnq.tex
\subsection{Neural Network Quantization}\label{sec:nnq}
\acrfull{NN} quantization amounts to training networks with parameters (and activations) restricted to a small discrete set representing the quantization levels. 
%
Here we discuss how one can formulate parameter quantization as a constrained optimization problem and activation quantization can be similarly formulated.
 
\paragraph{Parameter Space Formulation.}
Given a dataset $\calD=\{\bfx_i, \bfy_i\}_{i=1}^n$, parameter quantization can be written as:
\vspace{-2ex}
\begin{equation}\label{eq:dnnobj}
\min_{\bfw\in \calQ^m} L(\bfw;\calD) := \frac{1}{n} \sum_{i=1}^n
\ell(\bfw;(\bfx_i,\bfy_i))\ .
\vspace{-1ex} 
\end{equation}
Here, $\ell(\cdot)$ denotes the input-output mapping composed with a standard loss function (\eg, cross-entropy loss), $\bfw$ is the $m$ dimensional parameter vector, and $\calQ$ with $|\calQ|=d$ is a predefined discrete set representing quantization levels (\eg, $\calQ=\{-1,1\}$ or $\calQ=\{-1,0,1\}$). 

The approaches that directly optimize in the parameter space include \gls{BC}~\citenew{courbariaux2015binaryconnect} and its variants~\citenew{hubara2017quantized,rastegari2016xnor}, where the constraint set is discrete.
In contrast, recent approaches~\citenew{bai2018proxquant,yin2018binaryrelax} relax this constraint set to be its convex hull: 
\vspace{-1ex}
\begin{equation}
\conv(\calQ^m) = [q_{\min}, q_{\max}]^m\ ,
\end{equation}
where $q_{\min}$ and $q_{\max}$ represent the minimum and maximum quantization levels, respectively.
In this case, a quantized solution is obtained by gradually increasing an annealing hyperparameter.

\paragraph{Lifted Probability Space Formulation.}
Another formulation is to treat \gls{NN} quantization as a discrete labelling problem based on the \gls{MRF} perspective~\citenew{ajanthan2018pmf}.
Here, the equivalent relaxed optimization problem corresponding to~\eqref{eq:dnnobj} can be written as:
\vspace{-1ex}
\begin{align}
\label{eq:simobj}
\min_{\bfu \in\calS} L(\bfu\bfq;\calD)
 := \frac{1}{n}
 \sum_{i=1}^n \ell(\bfu\bfq;(\bfx_i, \bfy_i))\ ,
\end{align}
where $\bfq$ is the vector of quantization levels with $\bfw = \bfu \bfq$ and the set $\calS$ takes the following form:
\vspace{-1ex}
\begin{equation}
\calS = \left\{\begin{array}{l|l}
\multirow{2}{*}{$\bfu$} & \sum_{\lambda} u_{j:\lambda} = 1, \quad\forall\,j\\
&u_{j:\lambda} \ge 0,\ \ \ \quad\quad\forall\,j, \lambda \end{array} \right\}\ .
\end{equation}
We can interpret the value $u_{j:\lambda}$ as the probability of assigning the discrete label $\lambda$ to the weight $w_j$.
Therefore~\eqref{eq:simobj} can be interpreted as optimizing the probability of each parameter taking a discrete label.

\SKIP{
Here we review two constrained optimization formulations for \gls{NN} quantization: 1) directly constrain each parameter to be in the discrete set; and 2) optimize the probability of each parameter taking a label from the set of quantization levels.
 
\subsubsection{Parameter Space Formulation}\label{sec:w}
Given a dataset $\calD=\{\bfx_i, \bfy_i\}_{i=1}^n$, \gls{NN} quantization can be written as:
\vspace{-2ex}
\begin{equation}\label{eq:dnnobj}
\min_{\bfw\in \calQ^m} L(\bfw;\calD) := \frac{1}{n} \sum_{i=1}^n
\ell(\bfw;(\bfx_i,\bfy_i))\ .
\vspace{-1ex} 
\end{equation}
Here, $\ell(\cdot)$ denotes the input-output mapping composed with a standard loss function (\eg, cross-entropy loss), $\bfw$ is the $m$ dimensional parameter vector, and $\calQ$ with $|\calQ|=d$ is a predefined discrete set representing quantization levels (\eg, $\calQ=\{-1,1\}$ or $\calQ=\{-1,0,1\}$). 

The approaches that directly optimize in the parameter space include \gls{BC}~\citenew{courbariaux2015binaryconnect} and its variants~\citenew{hubara2017quantized,rastegari2016xnor}, where the constraint set is discrete.
In contrast, recent approaches~\citenew{bai2018proxquant,yin2018binaryrelax} relax this constraint set to be its convex hull: 
\vspace{-1ex}
\begin{equation}
\conv(\calQ^m) = [q_{\min}, q_{\max}]^m\ ,
\end{equation}
where $q_{\min}$ and $q_{\max}$ represent the minimum and maximum quantization levels, respectively.
In this case, a quantized solution is obtained by gradually increasing an annealing hyperparameter.

\subsubsection{Lifted Probability Space Formulation}\label{sec:u}
Another formulation is based on the \gls{MRF} perspective to \gls{NN} quantization recently studied in~\citenew{ajanthan2018pmf}.
It treats~\eqref{eq:dnnobj} as a {\em discrete labelling problem} and introduces indicator variables $u_{j:\lambda}\in\{0,1\}$ for each parameter $w_j$ where $j\in\allweights$ such that $u_{j:\lambda}=1$ if and only if $w_j = \lambda \in \calQ$.
For convenience, by denoting the vector of quantization levels as $\bfq$, a parameter vector ${\bf w} \in \calQ^m$ can be written in a matrix vector product as $\bfw = \bfu\bfq$ where\footnote{To simplify the notation, we denote $\bfu$ as a matrix but it can be thought of as an $md$ dimensional vector obtained by flattening the matrix $\bfu$.}:
\vspace{-2ex}
\begin{align}\label{eq:indicv}
\bfu \in \calV = \left\{\begin{array}{l|l}
\multirow{2}{*}{$\bfu$} & \sum_{\lambda} u_{j:\lambda} = 1,
\quad\forall\,j\\ 
&u_{j:\lambda} \in \{0,1\}, \quad\forall\,j, \lambda
\end{array} \right\}\ .\\[-4.5ex]\nonumber
\end{align}

Now, similar to the relaxation in the parameter space, one can relax the binary constraint in $\calV$ to form its convex hull:
\vspace{-1ex}
\begin{equation}
\calS = \conv(\calV) = \left\{\begin{array}{l|l}
\multirow{2}{*}{$\bfu$} & \sum_{\lambda} u_{j:\lambda} = 1, \quad\forall\,j\\
&u_{j:\lambda} \ge 0,\ \ \ \quad\quad\forall\,j, \lambda \end{array} \right\}\ .
\end{equation}
We can interpret the value $u_{j:\lambda}$ as the probability of assigning the discrete label $\lambda$ to the weight $w_j$.
This relaxed optimization can then be written as:
\vspace{-1ex}
\begin{align}
\label{eq:simobj}
\min_{\bfu \in\calS} L(\bfu\bfq;\calD)
 := \frac{1}{n}
 \sum_{i=1}^n \ell(\bfu\bfq;(\bfx_i, \bfy_i))\ .
\end{align}
Even in this case, a discrete solution $\bfu\in\calV$ can be enforced via an annealing hyperparameter or using rounding schemes.
}

%% file: text/mda_proj.tex
\section{Mirror Descent Framework for Network Quantization}\label{sec:mdannq}
Before introducing the \gls{MD} formulation, we first write \gls{NN} quantization as a single objective unifying~\twoplaineqref{eq:dnnobj}{eq:simobj} as:
\vspace{-1ex}
\begin{equation}\label{eq:comobj}
\min_{\bfx\in\calX}\ f(\bfx)\ ,
\end{equation}
where $f(\cdot)$ denotes the loss function by abstracting out the dependency on the dataset $\calD$, and $\calX$ denotes the constraint set.
As discussed in~\secref{sec:nnq}, many recent \gls{NN} quantization methods 
optimize over the convex hull of the constraint set.
Following this, we consider  
the solution space $\calX$ in~\eqref{eq:comobj} to be convex and compact.
%

\input{tables/example_projns}

To employ \gls{MD}, we need to choose a mirror map (refer~\dfnref{dfn:mm}) suitable for the problem at hand. 
In fact, as discussed in~\secref{sec:mda}, mirror map is the core component of an \gls{MD} algorithm which determines the effectiveness of the resulting \gls{MD} updates.
However, there is no straightforward approach to obtain a mirror map for a given constrained optimization problem, except in certain special cases.
 
To this end, we observe that the usual approach to optimize the above constrained problem is via a version of projected gradient descent, where the projection is the mapping from the unconstrained auxiliary variables (full-precision) to the quantized space $\calX$.
Now, noting the analogy between the purpose of the projection operator and the mirror maps in the \gls{MD} formulation, 
we intend to derive the mirror map analogous to a given projection.
Precisely, we prove that if the {\em projection is strictly monotone} (and hence invertible), a valid mirror map can be derived from the projection itself.
Even though this does not necessarily extend the theory of \gls{MD},
 this derivation is valuable as it connects existing \gls{PGD} type algorithms to their corresponding \gls{MD} variants. 
For completeness, we state it as a theorem for the case $\calX\subset \R$ and the multidimensional case can be proved with an additional assumption that the vector field $\iproj(\bfx)$ is conservative. 
\SKIP{
\begin{thm}\label{thm:mdaproj}
Let $\calX$ be a compact convex set and $P:\R \to \calC$ be an invertible function where $\calC\subset\R$ is a convex open set such that $\calX=\bcalC$ ($\bcalC$ denotes the closure of $\calC$).
Now, if
\begin{tight_enumerate}
  \item $P$ is strictly monotonically increasing.
  \item $\lim_{x\to\partial \calC}\|\iproj(x)\| = \infty$ ($\partial \calC$ denotes the boundary of $\calC$).
\end{tight_enumerate}
Then, $\Phi(x) = \int_{x_0}^x \iproj(y) dy$ is a valid mirror map.
\end{thm}
}
\begin{thm}\label{thm:mdaproj}
Let $\calC$ be a finite open interval and ${P:\R \to \calC}$ be a strictly monotonically increasing continuous function. 
Then, ${\Phi(x) = \int_{x_0}^x \iproj(y) dy}$ is a valid mirror map. 
\end{thm}
\vspace{-2ex}
\begin{proof}
This can be proved by noting that $P$ is invertible, $\nabla\Phi(x) = \iproj(x)$, and $\Phi(x)$ is strictly convex.
\end{proof}
\vspace{-1ex}

\begin{figure}
\begin{center}
\includegraphics[width=\linewidth, trim=4.2cm 8.5cm 10.5cm 4.3cm, clip=true,
page=4]{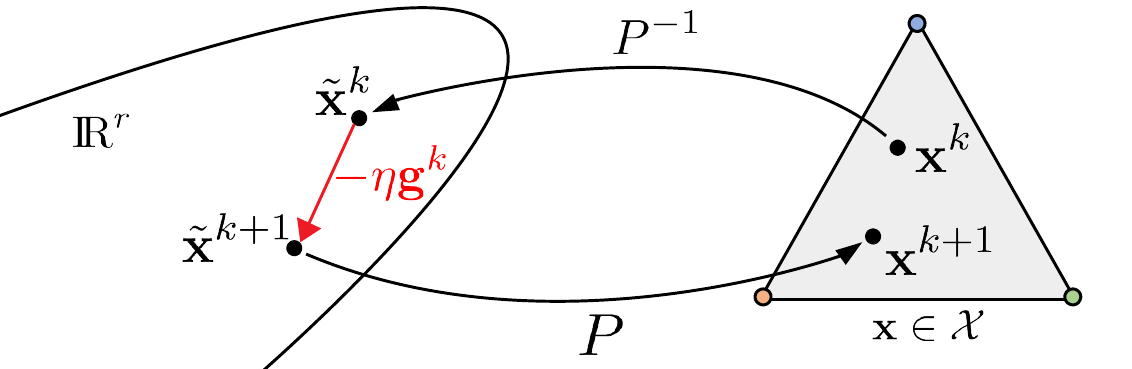}	
\end{center}
 \vspace{-2ex}
\caption{\em \gls{MD} formulation where mirror map is derived from the projection $P$. 
Note, $\bfg^k$ is computed in the primal space ($\calX$) but it is directly used to update the auxiliary variables in the dual space.}
\label{fig:mdaproj}
\vspace{-2ex}
\end{figure}
The \gls{MD} update based on the mirror map derived from a given projection is illustrated in~\figref{fig:mdaproj}.
Note that, to employ \gls{MD} to the problem~\plaineqref{eq:comobj}, in theory, any mirror map satisfying~\dfnref{dfn:mm} whose domain (\ie, its closure) is a superset of the constraint set $\calX$ can be chosen. 
The above theorem provides a method to derive a subset of all applicable mirror maps, where the closure of the domain of mirror maps is exactly equal to the constraint set $\calX$.

We now provide mirror maps and update steps for two different projections ($\tanh$ for $\bfw$-space~(\eqref{eq:dnnobj}) and $\softmax$ for $\bfu$-space~(\eqref{eq:simobj})) useful for \gls{NN} quantization in \tabref{tab:eg_projns}. Given mirror maps (from \thmref{thm:mdaproj}), the \gls{MD} updates are straightforwardly derived based on~\eqref{eq:mdprox} using \acrshort{KKT} conditions~\citenew{boyd2009convex}. For the detailed derivations and pseudocode for \acrshort{MD}-$\tanh$, please refer to Appendix. 
Furthermore, the $\tanh$ projection, its inverse, and the corresponding mirror map are illustrated in~\figref{fig:tanh}, showing monotonicity of the inverse and strict convexity of the derived mirror map.

According to the update steps in \tabref{tab:eg_projns}, our \acrshort{MD} variants corresponding to $\tanh$ and $\softmax$ projections can be performed directly in the primal space.
However, for some projections (\eg, multi-bit quantization), it might be non-trivial to derive the exact form of mirror maps (and the \acrshort{MD} update), nevertheless, the \acrshort{MD} update can be easily implemented by storing an additional set of auxiliary variables. This as discussed in \secref{sec:stablemd} also improves the numerical stability of \gls{MD}.

Note that, to ensure a discrete solution at the end of the training, the projection $P$ is parametrized by a scalar $\beta_k \geq 1$ and it is annealed throughout the optimization. 
This annealing hyperparameter translates into a time varying mirror map (refer to~\tabref{tab:eg_projns}) in our case.
Intuitively, such an adaptive mirror map gradually constrains the solution space $\calX$ to its boundary and in the limit enforces a quantized solution. 

 
  \begin{figure}[t]
    \centering
    \includegraphics[width=0.5\linewidth]{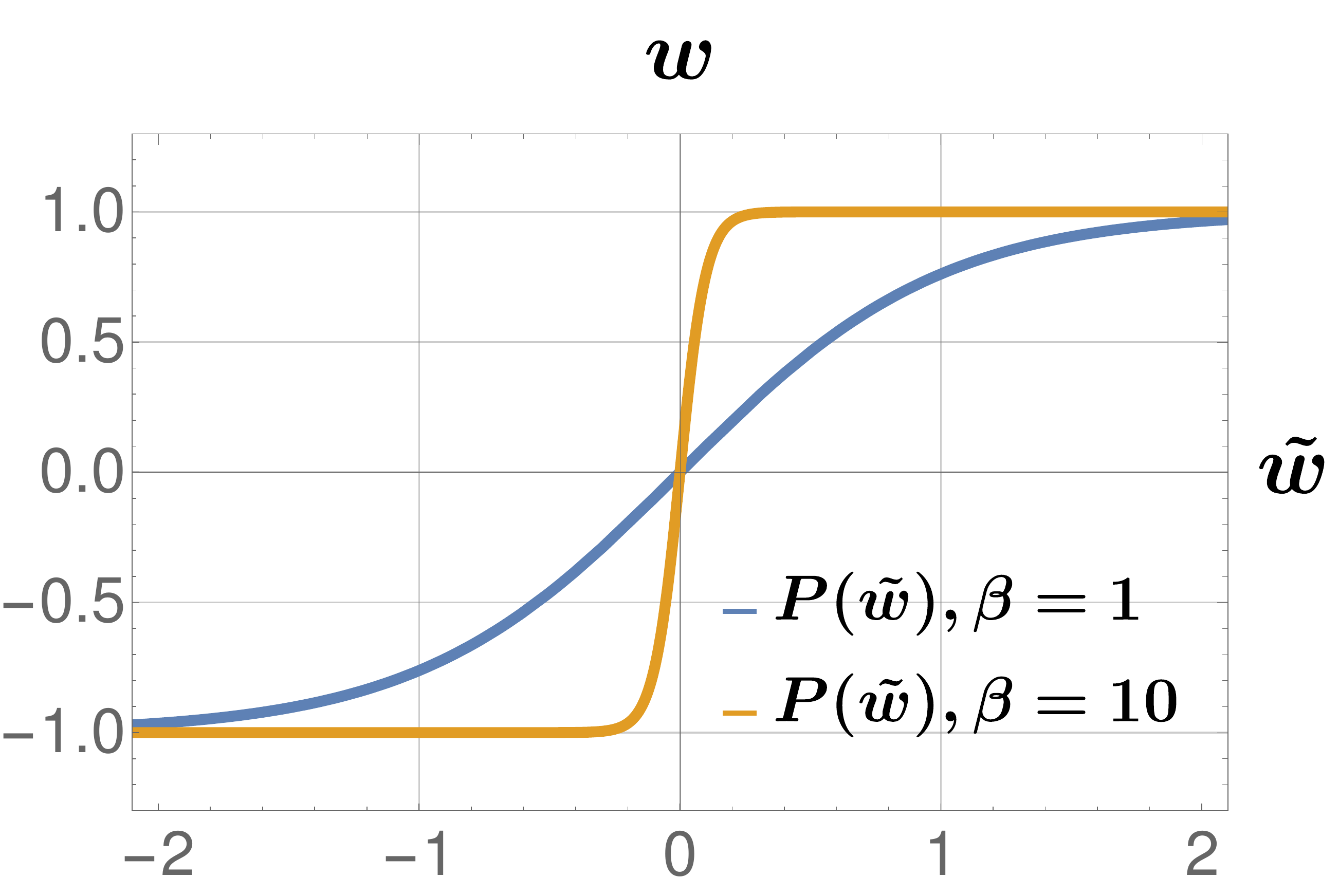}%
    \includegraphics[width=0.5\linewidth]{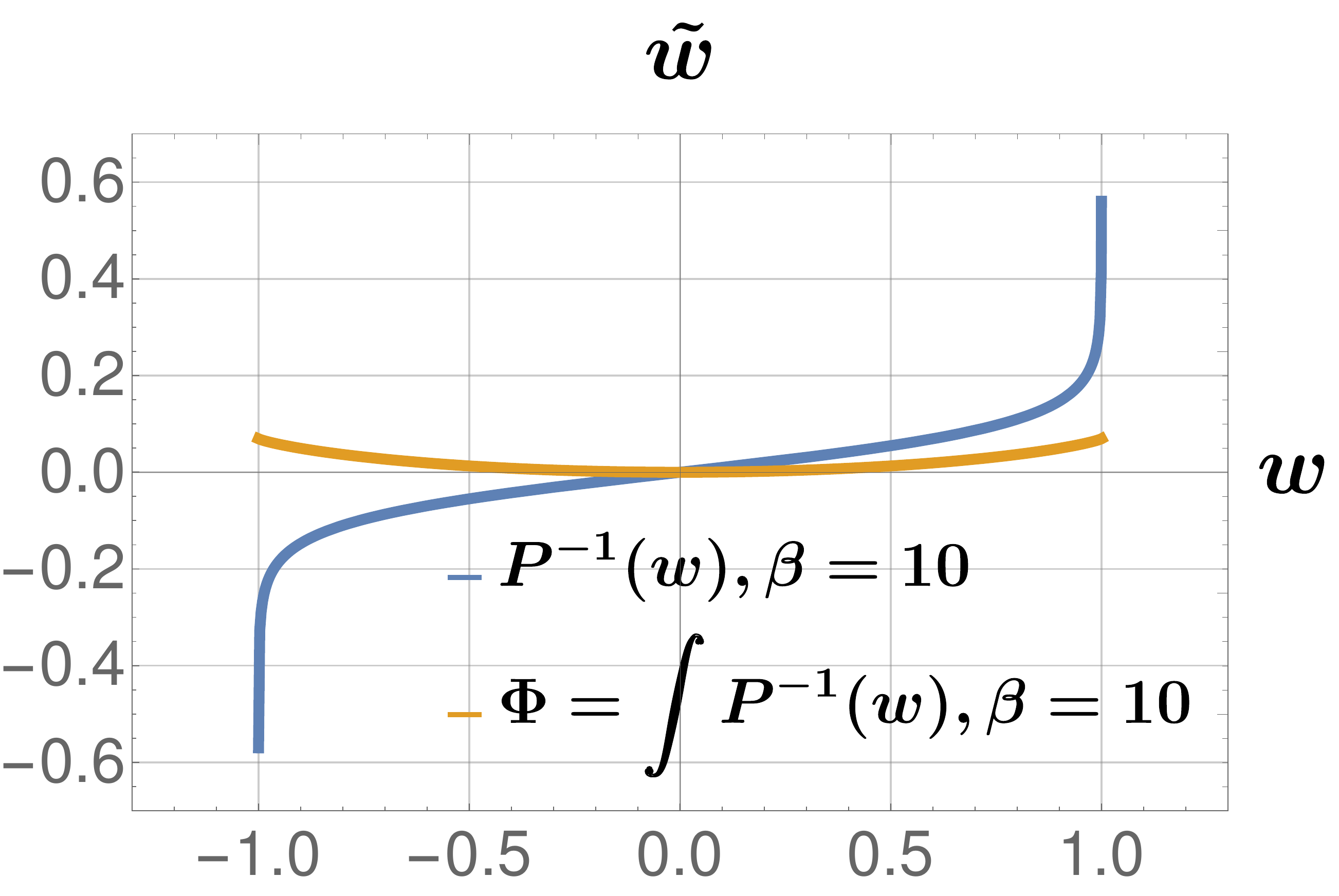}%
    \caption{\em Plots of $\tanh$, its inverse and corresponding mirror map. Note that, the inverse is monotonically increasing and the mirror map is strictly convex. Moreover, when $\beta\to \infty$, $\tanh$ approaches the step function. 
    }
    \label{fig:tanh}
\end{figure}


\SKIP{
We now give some example projections useful for \gls{NN} quantization ($\tanh$ for $\bfw$-space and $\softmax$ for $\bfu$-space) and derive their corresponding mirror maps.
Given mirror maps, the \gls{MD} updates are straightforward based on~\eqref{eq:mdprox}.
Even though we consider differentiable projections, \thmref{thm:mdaproj} does not require the projection to be differentiable. 
For the rest of the section, 
we assume $m=1$, \ie, consider projections that are independent for each ${j\in\allweights}$. 
\begin{figure*}[t]
    \centering
    \begin{subfigure}{0.5\linewidth}
    \includegraphics[width=0.5\linewidth]{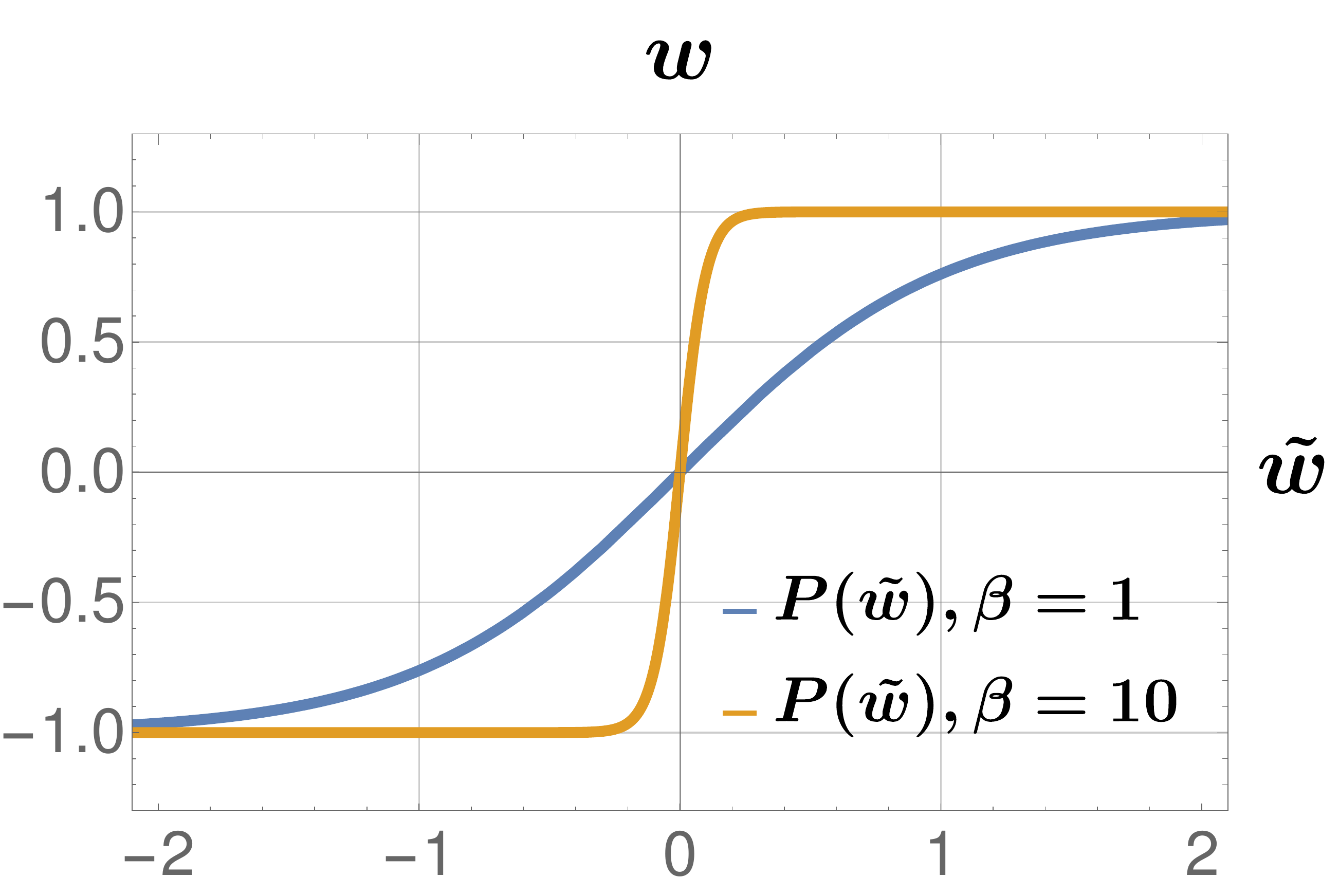}%
    \includegraphics[width=0.5\linewidth]{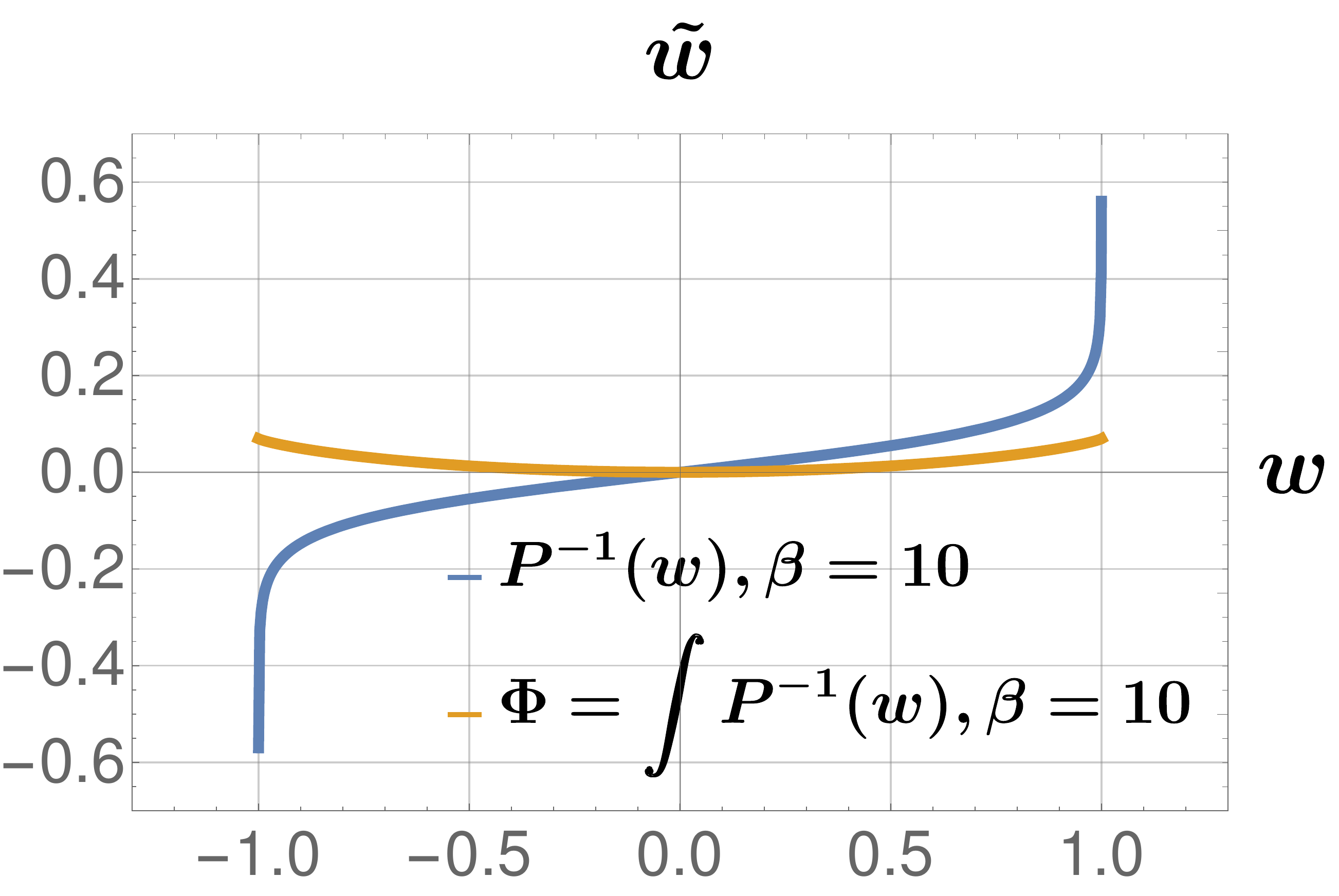}%
    \caption{$\tanh$, its inverse, and mirror map}
    \label{fig:tanh}
    \end{subfigure}%
    \begin{subfigure}{0.5\linewidth}
    \includegraphics[width=0.5\linewidth]{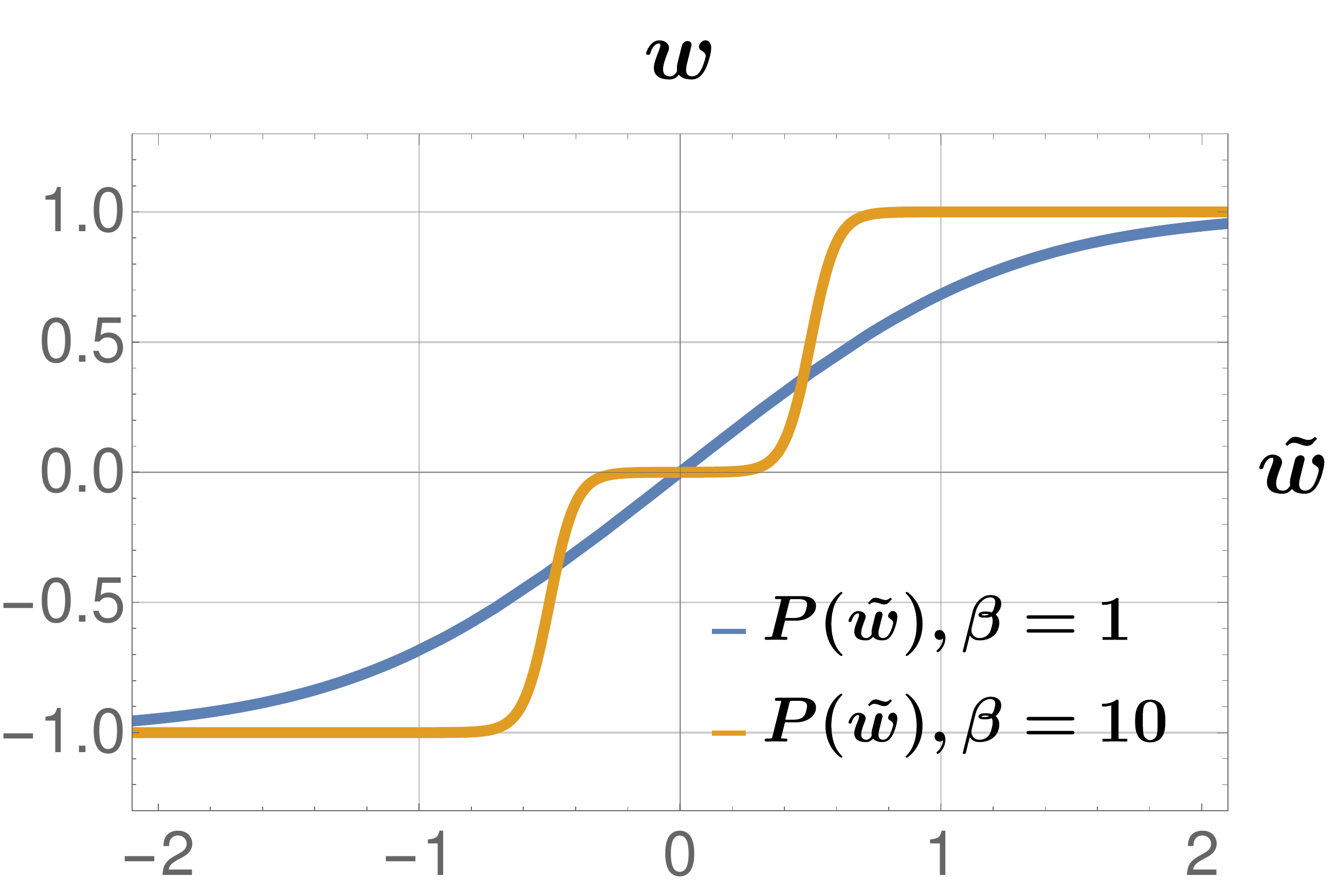}%
    \includegraphics[width=0.48\linewidth]{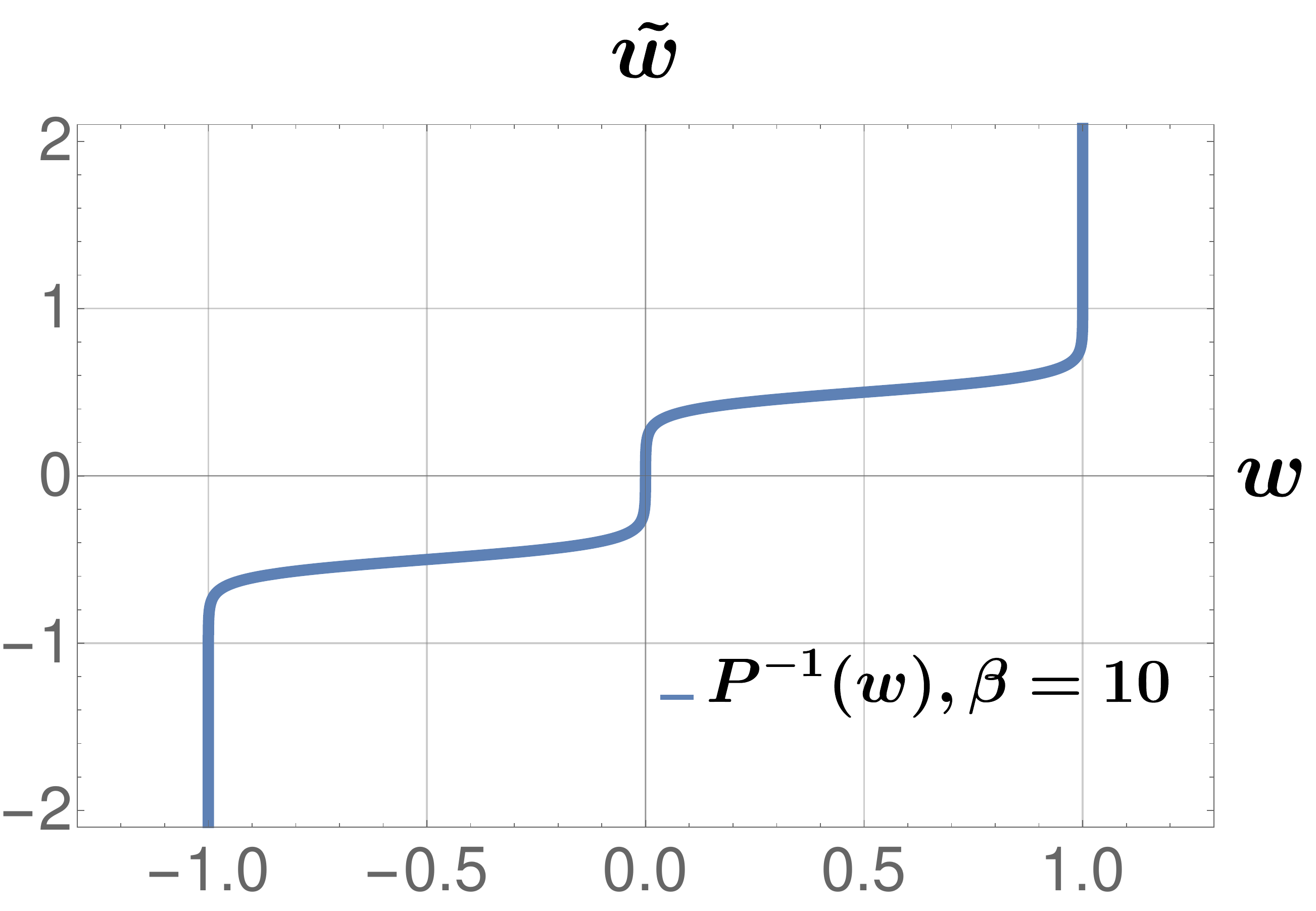}
    \caption{shifted $\tanh$ and its inverse}
    \label{fig:stanh}
    \end{subfigure}
    \vspace{-1ex}
    \caption{\em Plots of $\tanh$ and shifted $\tanh$ projections, and their inverses corresponding to the $\tanh$ projection.
		Note that, the inverses are monotonically increasing and the mirror map is strictly convex.
		Moreover, when $\beta\to \infty$, the projections approaches their respective hard versions. 
    }
    \vspace{-2ex}
    \label{fig:curves}
\end{figure*}

\begin{exm}[$\bfw$-space, binary, $\tanh$]\label{exm:tanh}
Consider the $\tanh$ function, which projects a real value to the interval $[-1,1]$:
\vspace{-1ex}
\begin{equation}\label{eq:tanh}
w = P(\tw):= \tanh(\beta\tw) = \frac{\exp(2\beta \tw) - 1}{\exp(2\beta \tw) + 1}\ ,
\vspace{-0.5ex}
\end{equation}
where $\beta >0$ is the annealing hyperparameter and when $\beta \to \infty$, $\tanh$ approaches the step function.
 \NOTE{need to change $\beta$ to $\beta^k$}
%
The inverse of the $\tanh$ is:
\vspace{-1ex}
\begin{equation}
\iproj(w) = \frac{1}{\beta}\tanh^{-1}(w) = \frac{1}{2\beta}\log\frac{1+w}{1-w}\ .
\vspace{-1ex}
\end{equation}
Note that, $\iproj$ is monotonically increasing for a fixed $\beta$. 
Correspondingly, the mirror map from~\thmref{thm:mdaproj} can be written as:
\begin{align}\label{eq:tanhmm}
\Phi(w) &= \int \iproj(w) dw \nonumber \\
&= \frac{1}{2\beta}\big[(1+w)\log(1+w) + (1-w)\log(1-w)\big]\ . 
\end{align}
Here, the constant from the integration is ignored. 
It can be easily verified that $\Phi(w)$ is in fact a valid mirror map.
The projection, its inverse and the corresponding mirror map are illustrated in~\figref{fig:tanh}.
Consequently, the resulting \gls{MD} update~\plaineqref{eq:mdprox} takes the following form:
\begin{align}\label{eq:tanhmd}
w^{k+1} &= \amin{w\in(-1,1)}\, \langle \eta\,g^k, w\rangle + D_{\Phi}(w, w^k) \nonumber \\ 
 &= \frac{\frac{1+w^k}{1-w^k}\exp(-2\beta\eta g^k) - 1}{\frac{1+w^k}{1-w^k}\exp(-2\beta\eta g^k) + 1}\ .
 \vspace{-0.5ex} 
\end{align}
The update formula is derived using the \acrshort{KKT} conditions~\citenew{boyd2009convex}.
For the detailed derivation please refer to Appendix~\myref{A.2}. 
A similar derivation can also be performed for the $\sigmoid$ function, where $\bcalC=\calX=[0,1]$.
Note that the $\sign$ function has been used for binary quantization in~\citenew{courbariaux2015binaryconnect} and $\tanh$ can be used as a soft version of $\sign$ function as pointed out by~\citenew{zhang2015bit}.
Mirror map corresponding to $\tanh$ is used for online linear optimization in~\citenew{bubeck2012towards} but here we use it for \gls{NN} quantization. The pseudocode of the above mentioned approach (\acrshort{MD}-$\tanh$) is provided in Algorithm~\myref{1} in the Appendix.
\SKIP{
the Bregman divergence can be written as:
 \begin{equation}
D_{\Phi}(w, v) = \frac{1}{2}\left[w\log\frac{(1+w)(1-v)}{(1-w)(1+v)} +  \log(1-w)(1+w) - \log(1-v)(1-v) \right]\ . 
\end{equation}
}
\end{exm}

\begin{exm}[$\bfu$-space, multi-label, $\softmax$]\label{exm:sm}
Now we consider the $\softmax$ projection used in \gls{PMF}~\citenew{ajanthan2018pmf} to optimize in the lifted probability space. 
In this case, the projection is defined as $P(\tbfu):= \softmax(\beta\tbfu)$ where $P:\R^d \to \calC$ with $\bcalC = \calX = \Delta$. Here $\Delta$ is the $(d-1)$-dimensional probability simplex and $|\calQ|=d$.
Note that the $\softmax$ projection is not invertible as it is a many-to-one mapping.
In particular, it is invariant to translation, \ie,
\vspace{-1ex}
\begin{align}
\bfu=\softmax(\tbfu + c\bfone) = \softmax(\tbfu)\ , \nonumber \\ \text{where}\quad
u_{\lambda} = \frac{\exp(\tu_{\lambda})}{\sum_{\mu\in\calQ}\exp(\tu_{\mu})}\ \nonumber,
\vspace{-0.5ex}
\end{align} 
for any scalar $c\in\R$ ($\bfone$ denotes a vector of all ones).
%
\NOTE{can be removed or change $\beta$ to $\beta^k$}
Therefore, the $\softmax$ projection does not satisfy~\thmref{thm:mdaproj}.
However, one could obtain a solution of the inverse of $\softmax$ as follows: given $\bfu\in\Delta$, find a unique point $\tbfv = \tbfu+c\bfone$, for a particular scalar $c$, such that $\bfu=\softmax(\tbfv)$.
Now, by choosing $c=-\log(\sum_{\mu=\calQ} \exp(\tu_{\mu}))$, $\softmax$ can be written as:
\vspace{-1ex}
\begin{equation}\label{eq:sminv}
\bfu = \softmax(\tbfv)\ ,\qquad \text{where}\quad
u_{\lambda} = \exp(\tv_{\lambda})\ ,\quad\forall\,\lambda\in\calQ\ .
\vspace{-1ex}
\end{equation} 
Now, the inverse of the projection can be written as: 
\vspace{-1ex}
\begin{equation}
\tbfv = \iproj(\bfu)  = \frac{1}{\beta}\softmax^{-1}(\bfu)\ ,\qquad \text{where}\quad
\tv_{\lambda} = \frac{1}{\beta}\log(u_{\lambda})\ ,\qquad\forall\,\lambda\ .
\vspace{-1ex}
\end{equation} 
Indeed, $\log$ is a monotonically increasing function and from~\thmref{thm:mdaproj}, by summing the integrals, the mirror map can be written as:
\vspace{-1ex}
\begin{equation}\label{eq:smmm}
\Phi(\bfu) = \frac{1}{\beta}\left[\sum_{\lambda}u_{\lambda}\log(u_{\lambda}) - u_{\lambda}\right] = -\frac{1}{\beta}H(\bfu) - 1/\beta\ .
\vspace{-1ex} 
\end{equation}
Here, $\sum_\lambda u_{\lambda} = 1$ as $\bfu\in\Delta$, and $H(\bfu)$ is the entropy.
Interestingly, as the mirror map in this case is the negative entropy (up to a constant), the \gls{MD} update leads to the well-known \gls{EGD} (or \gls{EDA})~\citenew{beck2003mirror,bubeck2015convex}.
Consequently, the update takes the following form:
\begin{align}\label{eq:eda}
u^{k+1}_{\lambda} &= \frac{u_{\lambda}^k\,\exp(-\beta\eta g^{k}_{\lambda})}{\sum_{\mu  \in \calQ} \; 
u_{\mu}^k\,\exp(-\beta\eta g^{k}_{\mu})}\quad \forall\, \lambda \ .
\end{align}
The derivation follows the same approach as in the $\tanh$ case above.
It is interesting to note that the \gls{MD} variant of $\softmax$ is equivalent to the well-known \gls{EGD}.
Notice, the authors of \gls{PMF}~\cite{ajanthan2018pmf} hinted that \gls{PMF} is related to \gls{EGD} but here we have clearly showed that the \gls{MD} variant of \gls{PMF} under the above reparametrization~\plaineqref{eq:sminv} is exactly \gls{EGD}.
\end{exm}

\begin{exm}[$\bfw$-space, multi-label, shifted $\tanh$]\label{exm:stanh}
\NOTE{change $\beta$ to $\beta^k$}
Note that, similar to $\softmax$, we wish to extend the $\tanh$ projection beyond binary. 
The idea is to use a function that is an addition of multiple shifted $\tanh$ functions.
\SKIP{
To this end, let us consider a given set of quantization levels as an ordered set $\calQ = \{l_1, \ldots, l_d\}$ with $l_i < l_j$ for all $i < j$.  
Hence, $\bcalC = \calX = [l_1, l_d]$.
Now, we define our shifted $\tanh$ projection $P:\R\to \calC$ as:
\begin{equation}
w = P(\tw) = \frac{1}{d-1}\sum_{i=1}^{d-1} \tanh\left(\beta\left(\tw - \frac{l_i + l_j}{2}\right)\right)\ ,
\end{equation}
where $\beta>0$ and $d= |\calQ|$.
It is easy to see that $w \in \calX$
}
To this end, as an example we consider ternary quantization, with $\calQ=\{-1,0,1\}$ and define our shifted $\tanh$ projection $P:\R\to \calC$ as:
\begin{equation}
w = P(\tw) = \frac{1}{2}\big[\tanh\left(\beta(\tw+0.5)\right) + \tanh\left(\beta(\tw-0.5)\right)\big]\ ,
\end{equation}
where $\beta \geq 1$ and $w = \bcalC = \calX=[-1,1]$.
When $\beta\to \infty$, $P$ approaches a stepwise function with inflection points at $-0.5$ and $0.5$ (here, $\pm0.5$ is chosen  heuristically), meaning $w$ move towards one of the quantization levels in the set $\calQ$.
This behaviour together with its inverse is illustrated in~\figref{fig:stanh}.
Now, one could potentially find the functional form of $\iproj$ and analytically derive the mirror map corresponding to this projection.
Note that, while~\thmref{thm:mdaproj} provides an analytical method to derive mirror maps, in some cases such as the above, the exact form of mirror map and the \gls{MD} update might be nontrivial.
In such cases, as will be shown subsequently, 
the \gls{MD} update can be easily implemented by storing an additional set of auxiliary variables $\tw$.

\end{exm}


}

\subsection{Annealing and Convergence Analysis}

The classical \acrshort{MD} literature studied the convergence behaviour of \acrshort{MD} for the convex setting, and the adaptive mirror maps are considered in online learning~\citenew{mcmahan2017survey}. 
We now prove that, in the convex setting, if the annealing hyperparameter $\beta_k$ is bounded, then \gls{MD} with an adaptive mirror map converges to the optimal value at the same rate of $\calO(1/\sqrt{t})$ as the standard \gls{MD}.
\begin{thm}\label{thm:betamd}
Let $\calX\subset\R^r$ be a convex compact set and $\calC\subset\R^r$ be a convex open set with $\primal\ne \emptyset$ and $\calX\subset\bcalC$ ( $\bcalC$ denotes the closure of $\calC$). 
Let $\Phi:\calC\to\R$ be a mirror map $\rho$-strongly convex\footnote{A convex function $\Phi:\calC\to\R$ is $\rho$-strongly convex with respect to $\norm{\cdot}$ if $\Phi(\bfx)- \Phi(\bfy) \leq \langle \bfg, \bfx-\bfy\rangle\ - \frac{\rho}{2}\norm{x-y}^2, \forall x,y \in \calC \text{ and } g \in \partial \Phi(\bfx)$.} on $\primal$ with respect to $\norm{\cdot}$, 
${R^2 = \sup_{\bfx\in\primal} \Phi(\bfx) - \Phi(\bfx^0)}$ 
where $\bfx^0 = \argmin_{\bfx\in\primal}\Phi(\bfx)$ is the initialization
and $f:\calX\to\R$ be a convex function and $L$-Lipschitz with respect to $\norm{\cdot}$.
Then \gls{MD} with mirror map $\Phi_{\beta_k}(\bfx) = \Phi(\bfx)/\beta_k$ with $1 \le \beta_k \le B$ and $\eta = \frac{R}{L}\sqrt{\frac{2\rho}{Bt}}$ satisfies
\begin{equation}
f\left(\frac{1}{t}\sum_{k=0}^{t-1}\bfx^k\right) - f(\bfx^*) \le RL\sqrt{\frac{2B}{\rho t}}\ ,
\end{equation}
where $\beta_k$ is the annealing hyperparameter, $\eta >0$ is the learning rate, $t$ is the iteration index, and $\bfx^*$ is the optimal solution.
\begin{proof}
The proof is a slight modification to the proof of standard \acrshort{MD} noting that, effectively $\Phi_{\beta_k}$ is $\rho/B$-strongly convex. Please refer to Appendix. 
\end{proof}
\end{thm}
%
Theoretical analysis of \gls{MD} for nonconvex, stochastic setting is an active research area~\citenew{zhou2017mirror,zhou2017stochastic} and \gls{MD} has been recently shown to converge in the nonconvex stochastic setting under certain conditions~\citenew{zhang2018convergence}. 
We believe, similar to~\thmref{thm:betamd}, the convergence analysis in~\cite{zhang2018convergence} can be extended to \gls{MD} with adaptive mirror maps.
Nevertheless, \gls{MD} converges in all our experiments while outperforming the baselines in practice.

\paragraph{Ensuring a discrete solution.}
Our original objective~\eqref{eq:comobj} is to obtain a discrete solution via annealing the hyperparameter $\beta_k\to \infty$.
However, according to~\thmref{thm:betamd}, $\beta_k$ is capped at an arbitrarily chosen maximum value $B$.
To this end, we now derive a constraint on the auxiliary variables $\tbfx$ such that the primal variables converge to a discrete solution with a chosen precision $\epsilon > 0$ for a given $B$.

We consider the $\tanh$ projection with $m=1$ without loss of generality and a similar derivation is possible for the $\softmax$ projection as well.
Since $\beta_k \le B$, $\tx$ has to be constrained away from zero to ensure that $\tanh(B\tx)$ is close to the set $\{-1,1\}$ with a desired precision $\epsilon$.
We now state it as a proposition below.
\begin{pro}
For a given $B > 0$ and ${0<\epsilon<1}$, there exists a $\gamma>0$ such that if $|\tx|\ge\gamma$ then ${1-|\tanh(B\tx)|<\epsilon}$. Here $|\cdot|$ denotes the absolute value and $\gamma> \tanh^{-1}(1-\epsilon)/B$.
\end{pro}
\begin{proof}
This is derived via a simple algebraic manipulation of $\tanh$. Please refer to Appendix.
\end{proof}

\SKIP{
\noindent According to the theorem, in our implementation, $\beta_k$ is capped at an arbitrarily chosen maximum value $B$. 
Even though $\beta_k$ is never infinity, experimentally we observed that after certain value of $\beta_k$ increasing it further does not change the solution, perhaps because of the finite machine precision.

\subsubsection{Epsilon Convergence to a Discrete Solution}
Our original objective~\eqref{eq:comobj} is to obtain a discrete solution. We intend to enforce a discrete solution by gradually increasing annealing hyperparameter $\beta_k$. To guarantee that our approach converges to a discrete solution with a chosen precision $\epsilon>0$, we need to ensure that the image of the projection $P$ approaches the set of discrete solutions $\calQ^m$ when $\beta_k\to \infty$. 

Let us consider the $\tanh$ projection with $m=1$ without loss of generality and note that, according to~\thmref{thm:betamd}  $\beta_k \le B$ in practice.
Since $\beta_k$ is bounded in practice, $\tw$ can become arbitrarily close to zero and therefore the image of $\tanh$ will still be the interval $[-1,1]$. 
Hence, to ensure that the image of $\tanh$ approaches the set $\{-1,1\}$ when $\beta_k$ is increased, we need to constrain $\tw$ away from zero. Furthermore, this constraint can be derived given the value of $B$ and the desired precision $\epsilon$.
We now state it as a Proposition below.
\begin{pro}
For a given $B > 0$ and $0<\epsilon<1$, there exists a $\gamma>0$ such that if $|\tx|\ge\gamma$ then ${1-|\tanh(B\tx)|<\epsilon}$. Here $|\cdot|$ denotes the absolute value. 
In addition $\gamma> \tanh^{-1}(1-\epsilon)/B$.
\end{pro}
\begin{proof}
This is derived via a simple algebraic manipulation of $\tanh$. Please refer to Appendix.
\end{proof}
A similar approach can be used to find a constraint on $\tbfu$ in case of $\softmax$ projection to obtain a discrete solution with a chosen precision $\epsilon>0$. 

\subsubsection{Convergence of~\acrshort{MD} in the Nonconvex Setting}
Even though \gls{MD} is originally developed for convex optimization (similar to gradient descent), in this paper we directly apply \gls{MD} to \gls{NN} quantization where the loss is highly nonconvex and gradient estimates are stochastic. 
Nevertheless, \gls{MD} converges in all our experiments while obtaining superior performance compared to the baselines. 
Theoretical analysis of \gls{MD} for nonconvex, stochastic setting is an active research area~\citenew{zhou2017mirror,zhou2017stochastic} and \gls{MD} has been recently shown to converge in the nonconvex stochastic setting under certain conditions~\citenew{zhang2018convergence}. 
We believe, similar to~\thmref{thm:betamd}, the convergence analysis in~\cite{zhang2018convergence} can be extended to \gls{MD} with adaptive mirror maps, which goes beyond the scope of this work.
}

\subsection{Numerically Stable form of \acrshort{MD}}\label{sec:stablemd}
We showed two examples of valid projections, their corresponding mirror maps, and the final \gls{MD} updates in \tabref{tab:eg_projns}. Even though, in theory, these updates can be used directly, they are sometimes numerically unstable due to the operations involving multiple logarithms, exponentials, and divisions~\citenew{hsieh2018mirrored}. 
To this end, we provide a numerically stable way of performing \gls{MD} by storing a set of auxiliary parameters during training.

A careful look at the~\figref{fig:mdaproj} suggests that the \gls{MD} update with the mirror map derived from~\thmref{thm:mdaproj} can be performed by storing auxiliary variables $\tbfx = \iproj(\bfx)$. 
In fact, once the auxiliary variable $\tbfx^{k}$ is updated using gradient $\bfg^k$, it is directly mapped back to the constraint set $\calX$ via the projection.
This is mainly because of the fact that the domain of the mirror maps derived based on the~\thmref{thm:mdaproj} is exactly the same as the constraint set. 
Formally, with this additional set of variables, one can write the \gls{MD} update~\plaineqref{eq:md1} corresponding to the projection $P$ as:
\begin{align}\label{eq:mdhybridgd}
\tbfx^{k+1} &= \tbfx^k - \eta\,\bfg^k\ ,\quad\ \mbox{update in the dual space}\\\nonumber
\bfx^{k+1} &= P(\tbfx^{k+1})\in\calX\ ,\quad\mbox{projection to primal space}
\end{align}
where $\eta >0$ and $\bfg^k\in\partial f(\bfx^k)$. Experimentally we observed these updates to show stable behaviour and performed remarkably well for both the $\tanh$ and $\softmax$. 
We provide the pseudocode of this stable version of \acrshort{MD} in Algorithm~\myref{2} for the $\tanh$ (\acrshort{MD}-$\tanh$-\acrshort{S}) projection. Extending it to other valid projections is trivial.

Note, above updates can be seen as optimizing the function $f(P(\tbfx))$ using gradient descent where the gradient through the projection (\ie, Jacobian) $J_P = \partial P(\tbfx)/ \partial \tbfx$ is replaced with the identity matrix. 
This is exactly the same as the \acrfull{STE} for \gls{NN} quantization (following the nomenclature of~\cite{bai2018proxquant,yin2018binaryrelax}). 
Despite being a crude approximation, \gls{STE} has shown to be highly effective for \gls{NN} quantization with various network architectures and datasets~\citenew{yin2018binaryrelax,zhou2016dorefa}. 
However, a solid understanding of the effectiveness of \gls{STE} is lacking in the literature except for its convergence analysis in certain cases~\citenew{li2017training,yin2019understanding}.
In this work, by showing \gls{STE} based gradient descent as an implementation method of~\gls{MD} under certain conditions on the projection, we provide a justification on the effectiveness of \gls{STE}.
%

\paragraph{\acrlong{MD} \vs \acrlong{PQ}.}
The connection between the dual averaging version of~\gls{MD} and~\gls{STE} was recently hinted in~\gls{PQ}~\citenew{bai2018proxquant}. However, no analysis of whether an analogous mirror map exists to the given projection is provided and their final algorithm is not based on \gls{MD}.

Briefly, \acrshort{PQ} optimizes a objective of the following form: 
\begin{equation}\label{eq:pqobj}
\min_{\bfx\in\R^r} f(\bfx) + \beta R(\bfx)\ ,
\end{equation}
where $f$ is the loss function, the regularizer $R$ is a ``W'' shaped nonconvex function and $\beta$ is an annealing hyperparamter similar to ours.
Notice, even when the loss function $f$ is convex, the above \gls{PQ} objective would be nonconvex and has multiple local minima for a range of values of $\beta$.
Therefore \gls{PQ} is prone to converge to any of these local minima, whereas, our \gls{MD} algorithm (even \acrshort{PGD}) is guaranteed to converge to the global optimum regardless of the value of $\beta$. 

\SKIP{
\begin{figure}[t]
    \centering    
    \begin{subfigure}{0.5\linewidth}
    \includegraphics[width=\linewidth]{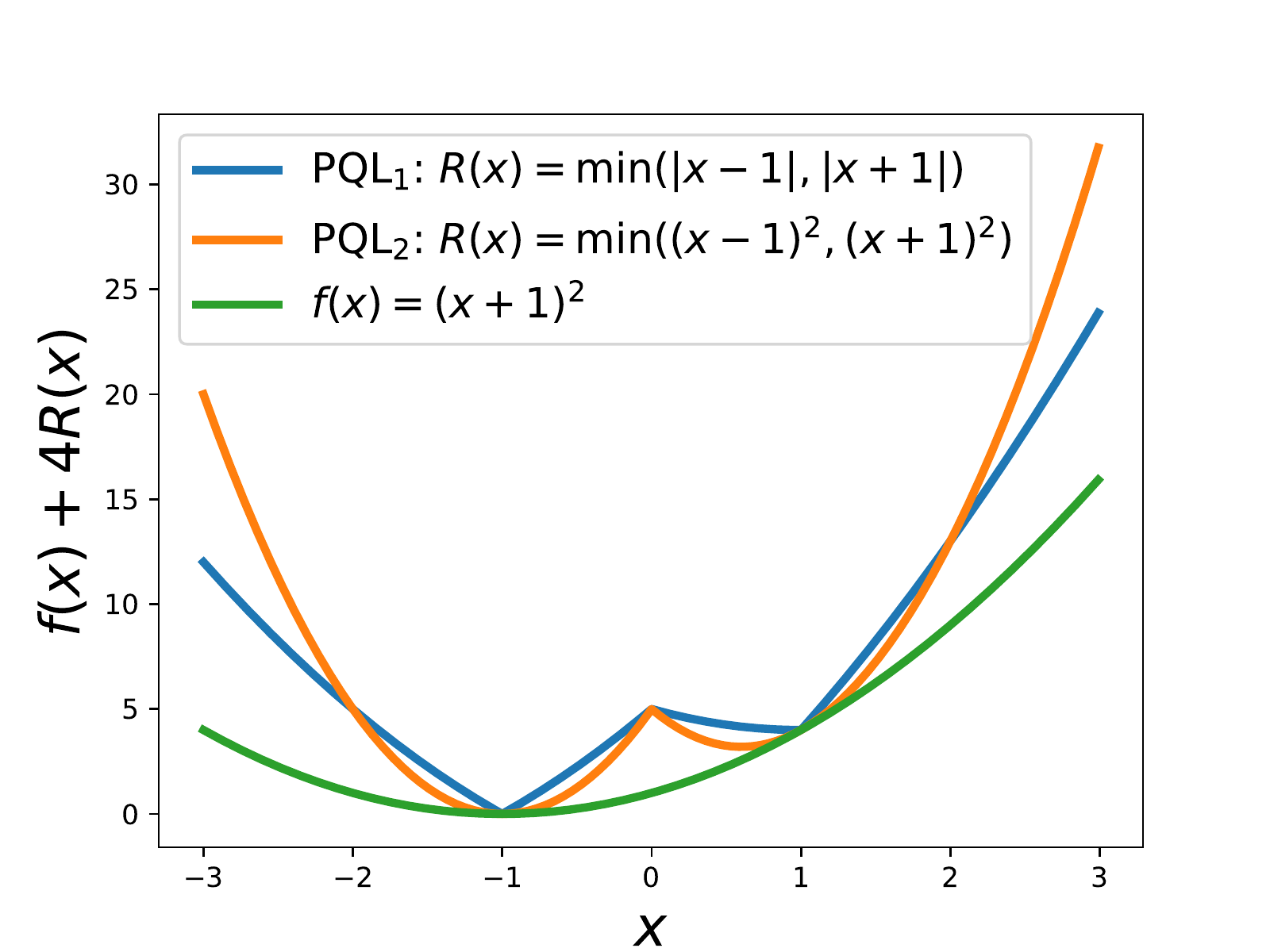}
    \caption{\small \acrshort{PQ} objectives}
    \end{subfigure}%
    \begin{subfigure}{0.5\linewidth}
    \includegraphics[width=\linewidth]{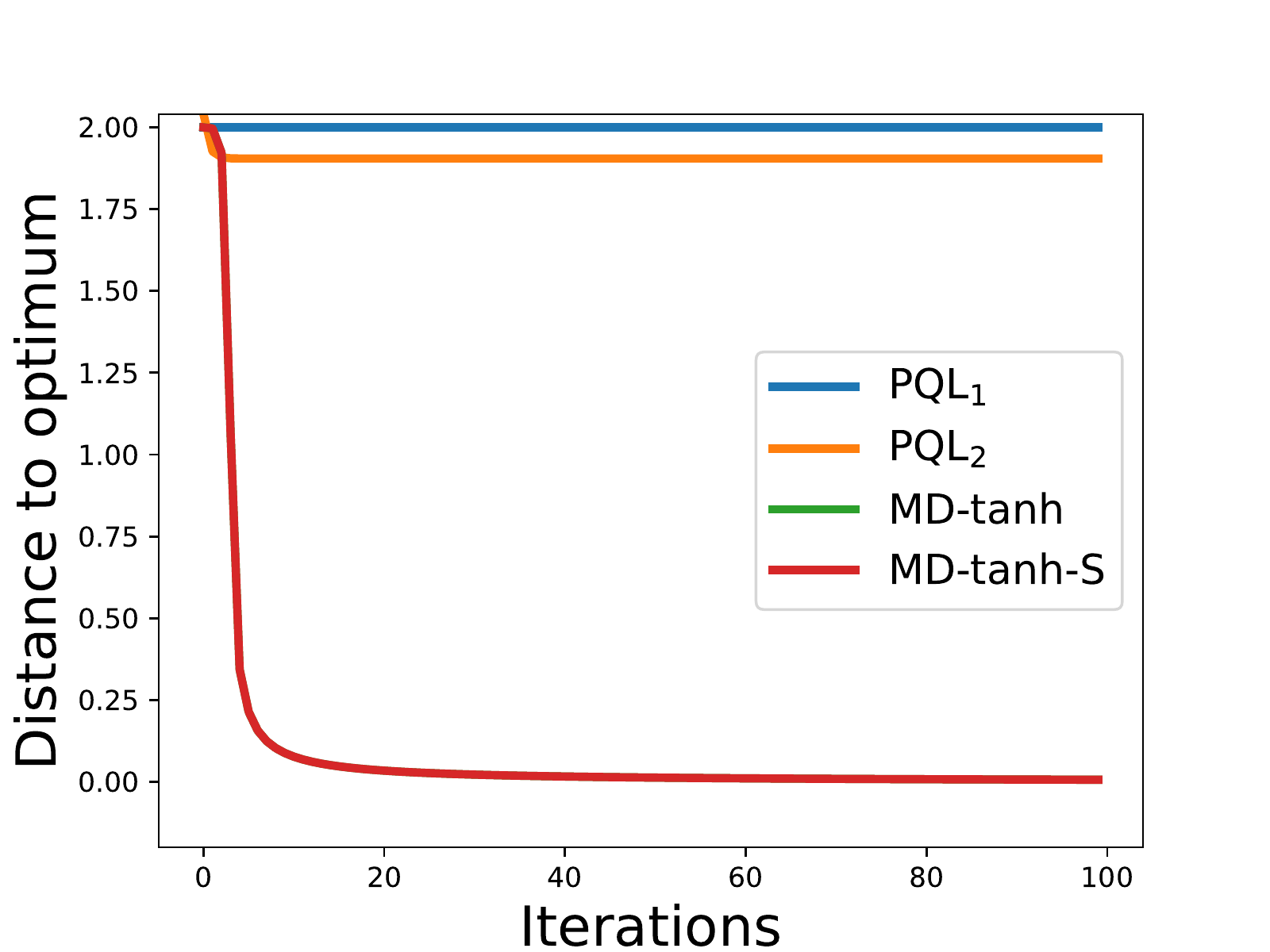}
    \caption{\small Convergence}
    \end{subfigure}%
    \vspace{-1ex}
    \caption{\small\em Convergence of \acrshort{PQ} and \acrshort{MD} when ${f(x) = (x+1)^2}$ where $\calX=[-1,1]$, initialization $\tx^0 = 1.2$ and $\beta=4$ (fixed) for all algorithms.
    While \acrshort{MD} versions behave identically and converge to the optimal solution of $-1$, \acrshort{PQ} versions get stuck at suboptimal (\pql{2} is nondiscrete) solutions.  
    }
    \label{fig:md-pq}
\end{figure}
\subsection{Comparison against~\acrlong{PQ}}
The connection between the dual averaging version of~\gls{MD} and~\gls{STE} was recently hinted in~\gls{PQ}~\citenew{bai2018proxquant}. However, no analysis of whether an analogous mirror map exists to the given projection is provided and their final algorithm is not based on \gls{MD}.
In particular, following our notation, the final update equation of~\acrshort{PQ} can be written as:
\SKIP{
\begin{align}\label{eq:pqfinal}
&\tbfx^{k+1} = \bfx^k - \eta\,\bfg^k\ , \nonumber \\ 
&\mbox{assumes $\bfx^k$ and $\bfg^k$ are in the same space} \nonumber \\
&\bfx^{k+1} = \prox(\tbfx^{k+1})\ , \nonumber \\
&\mbox{$\prox: \R^r\to \R^r$ is the proximal mapping defined in~\citenew{bai2018proxquant},}\nonumber
\end{align}
}
\begin{align}\label{eq:pqfinal}
\tbfx^{k+1} &= \bfx^k - \eta\,\bfg^k\ , \nonumber \\ 
\bfx^{k+1} &= \prox(\tbfx^{k+1})\ , \nonumber
\end{align}
assuming $\bfx^k$ and $\bfg^k$ are in the same space, where $\prox: \R^r\to \R^r$ is the proximal mapping defined in~\citenew{bai2018proxquant}, $\eta >0$, and $\bfg^k\in\partial f(\bfx^k)$. Note that, as opposed to~\gls{MD} (refer to~\eqref{eq:mdhybridgd}),~\acrshort{PQ} assumes the point $\bfx^k$ and gradient $\bfg^k$ are in the same space for the $\tbfx^{k+1}$ update to be valid. 
This would only be true for the Euclidean space. 
However, as discussed in~\secref{sec:mda},~\gls{MD} allows gradient descent to be performed on a more general non-Euclidean space by first mapping the primal point $\bfx^k$ to a point $\tbfx^k$ in the dual space via the mirror map. 
This is the core of \gls{MD}, which allows faster convergence rates in certain cases, and enabled theoretical and practical research on~\gls{MD} for the past three decades.

Despite this fundamental difference, here we show that \gls{PQ} can get stuck at a suboptimal (even nondiscrete) solution even in a simple convex setting. 
To this end, similarly to the convergence proof of \gls{PQ}, we consider the case when the annealing hyperparameter $\beta >0$ (denoted by $\lambda$ in~\cite{bai2018proxquant}) is fixed.
In this case, \gls{PQ} optimizes the following objective:
\vspace{-1.5ex}
\begin{equation}\label{eq:pqobj}
\min_{\bfx\in\R^r} f(\bfx) + \beta R(\bfx)\ ,
\vspace{-0.5ex}
\end{equation}
where the regularizer $R$ is a ``W'' shaped nonconvex function.
To this end, even when the loss function $f$ is convex, the above composite \gls{PQ} objective would be nonconvex and has multiple local minima for a range of values of $\beta$.
Therefore the \gls{PQ} algorithm is prone to converge to any of these local minima (which could be nondiscrete). 
Whereas, our \gls{MD} algorithms are guaranteed to converge to the global optimum in the convex setting regardless of the value of $\beta$. 
This phenomenon is illustrated with a simple example in~\figref{fig:md-pq}.
}

%% file: tables/example_projns.tex
\begin{table*}[t]
    \centering
    \begin{tabular}{l|c|l|l}
        \toprule
        Projection ($P_{\beta_k}$) & Space & Mirror Map ($\Phi_{\beta_k}$) & Update Step \\
        \midrule
        $\tanh(\beta_k\tw)$ & $\bfw$ & $\begin{array} {ll} \Phi_{\beta_k}(w) &= \frac{1}{2\beta_k}\big[(1+w)\log(1+w) \\ & \qquad + (1-w)\log(1-w)\big] \end{array}$  & $w^{k+1} = \frac{\frac{1+w^k}{1-w^k}\exp(-2\beta_k\eta g^k) - 1}{\frac{1+w^k}{1-w^k}\exp(-2\beta_k\eta g^k) + 1}$ \\
        $\softmax(\beta_k \tbfu)$ & $\bfu$ & $\Phi_{\beta_k}(\bfu) = \frac{1}{\beta_k}\left[\sum_{\lambda\in\calQ}u_{\lambda}\log(u_{\lambda}) - u_{\lambda}\right]$ & $u^{k+1}_{\lambda} = \frac{u_{\lambda}^k\,\exp(-\beta_k\eta g^{k}_{\lambda})}{\sum_{\mu  \in \calQ} \; 
u_{\mu}^k\,\exp(-\beta_k\eta g^{k}_{\mu})}\quad \forall\, \lambda\in \calQ$ \\
        \bottomrule
    \end{tabular}
    \caption{\em Example projections, corresponding mirror maps, and update steps obtained using \thmref{thm:mdaproj}.
    Here, $k$ is the iteration index, $\eta > 0$ is the learning rate, $g^k$ is the gradient of $f$ computed in the primal space, $\beta_k\ge 1$ is the annealing hyperparameter, and we assume $m=1$ without loss of generality.
    Notice the obtained mirror maps vary at each iteration due to $\beta_k$ and the $\softmax$ update resembles the popular \gls{EDA}~\citenew{beck2003mirror}. 
    }
    \label{tab:eg_projns}
    \vspace{-2ex}
\end{table*}

%% file: text/relatedwork.tex
\section{Related Work}

In this work, we mainly consider parameter quantization, which is usually formulated as a constrained problem and optimized using a modified projected gradient descent algorithm, where the methods~\citenew{ajanthan2018pmf,bai2018proxquant,carreira2017model,chen2019metaquant,courbariaux2015binaryconnect,yang2019quantization,yin2018binaryrelax} mainly differ in the constraint set, the projection used, and how backpropagation through the projection is performed.
Among them, \gls{STE} based gradient descent is the most popular method as it enables backpropagation through nondifferentiable projections and it has shown to be highly effective in practice~\citenew{courbariaux2015binaryconnect}.
In fact, the success of this approach lead to various extensions by including additional layerwise scalars~\citenew{rastegari2016xnor}, relaxing the solution space~\citenew{yin2018binaryrelax}, and even to quantizing activations~\citenew{hubara2017quantized}, and/or gradients~\citenew{zhou2016dorefa}. Some recent works~\citenew{leng2018extremely,ye2019progressive} employ \acrfull{ADMM} framework to learn low-bit neural networks. In contrast to these, recently~\cite{helwegen2019latent} proposed \acrfull{BOP} to avoid using ``latent'' real-valued weights during training.
%
Moreover, there are methods focusing on loss aware quantization~\citenew{hou2016loss}, quantization for specialized hardware~\citenew{esser2015energyEfficient}, and quantization based on the variational approach~\citenew{achterhold2018variationalQuantization,louizon2017bayesianCompression,louizos2018relaxed}. Some recent works~\citenew{liu2018bireal,martinez2019training,liu2020reactnet} have also explored architectural modifications beneficial to increase the capacity of binarized neural networks.
We have only provided a brief summary of relevant methods and for a comprehensive survey, we refer the reader to~\cite{guo2018survey}.
 

%% file: text/experiments.tex
\section{Experiments}

\input{tables/results_binary.tex}

Due to the popularity of binary neural networks~\citenew{courbariaux2015binaryconnect,rastegari2016xnor}, we mainly consider binary quantization and set the quantization levels as $\calQ=\{-1,1\}$.
We perform two sets of extensive experiments for comparisons against the state-of-the-art \acrshort{NN} binarization methods. 

First, we perform {\em full binarization} experiments similar to~\cite{ajanthan2018pmf} where all learnable parameters are binarized and activations are kept floating point on small scale datasets such as \cifar{-10/100} and \tinyimagenet{}\footnote{\url{https://tiny-imagenet.herokuapp.com/}} with \svgg{-16}, \sresnet{-18} and \smobilenet{V2} architectures.
Note, this is more difficult than the standard setup used in the quantization literature where the first and last layers are usually kept in high precision to enhance the performance. 
To ensure fair comparison, we ran the comparable baselines in this setup and performed extensive cross-validation, \eg, {\em up to $3\%$ improvement for \acrshort{PMF}}~\citenew{ajanthan2018pmf} due to this.
In summary, our results indicate that the binary networks obtained by the \gls{MD} variants outperform comparable baselines yielding state-of-the-art performance. 

Secondly, we evaluated our \mdtanhs{} approach on large scale \imagenet{} dataset with \sresnet{-18} in two different setups: 1) only parameters are binarized; and 2) both activations and parameters are binarized. We follow a similar experimental setup for our approach as has been used in baselines. For both the setups, our \mdtanhs{} variant outperforms all the recent baselines {\em even without requiring layerwise scalars} and sets new state-of-the-art for binarized networks on \imagenet{}.

For all the experiments, standard multi-class cross-entropy loss is used unless otherwise mentioned. We crossvalidate the hyperparameters such as learning rate, learning rate scale, rate of increase of annealing hyperparameter $\beta$, and their respective schedules.
We provide the hyperparameter tuning search space and the final hyperparameters in Appendix~\myref{B}.
Our algorithm is implemented in PyTorch~\citenew{paszke2017automatic} and the experiments are performed on NVIDIA Tesla-P100 GPUs. 
Our PyTorch code is available online\footnote{\url{https://github.com/kartikgupta-at-anu/md-bnn}}.

\begin{figure*}
    \centering
    \begin{subfigure}{0.25\linewidth}
    \includegraphics[width=0.99\linewidth]{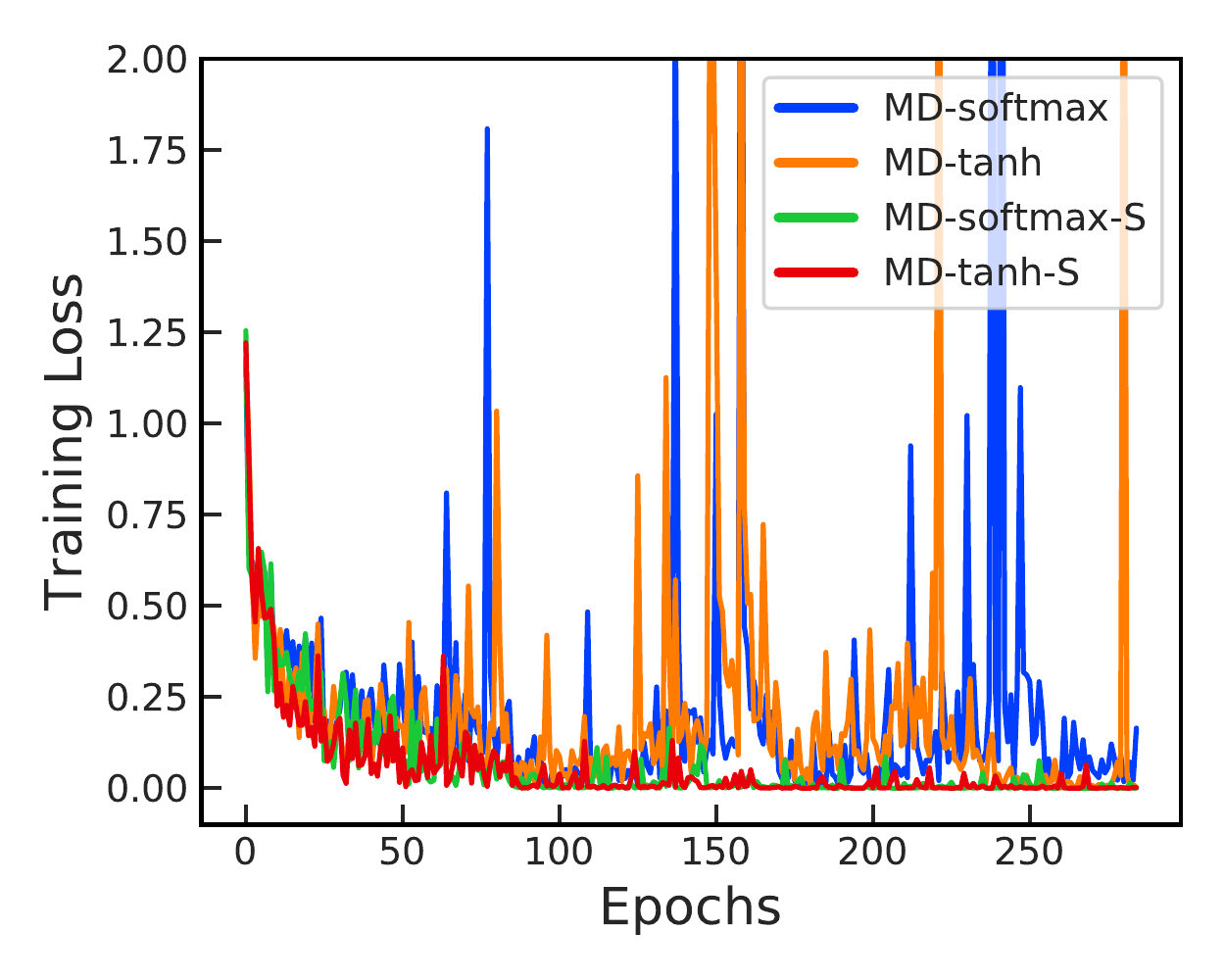}
    \end{subfigure}%
    \begin{subfigure}{0.25\linewidth}
    \includegraphics[width=0.99\linewidth]{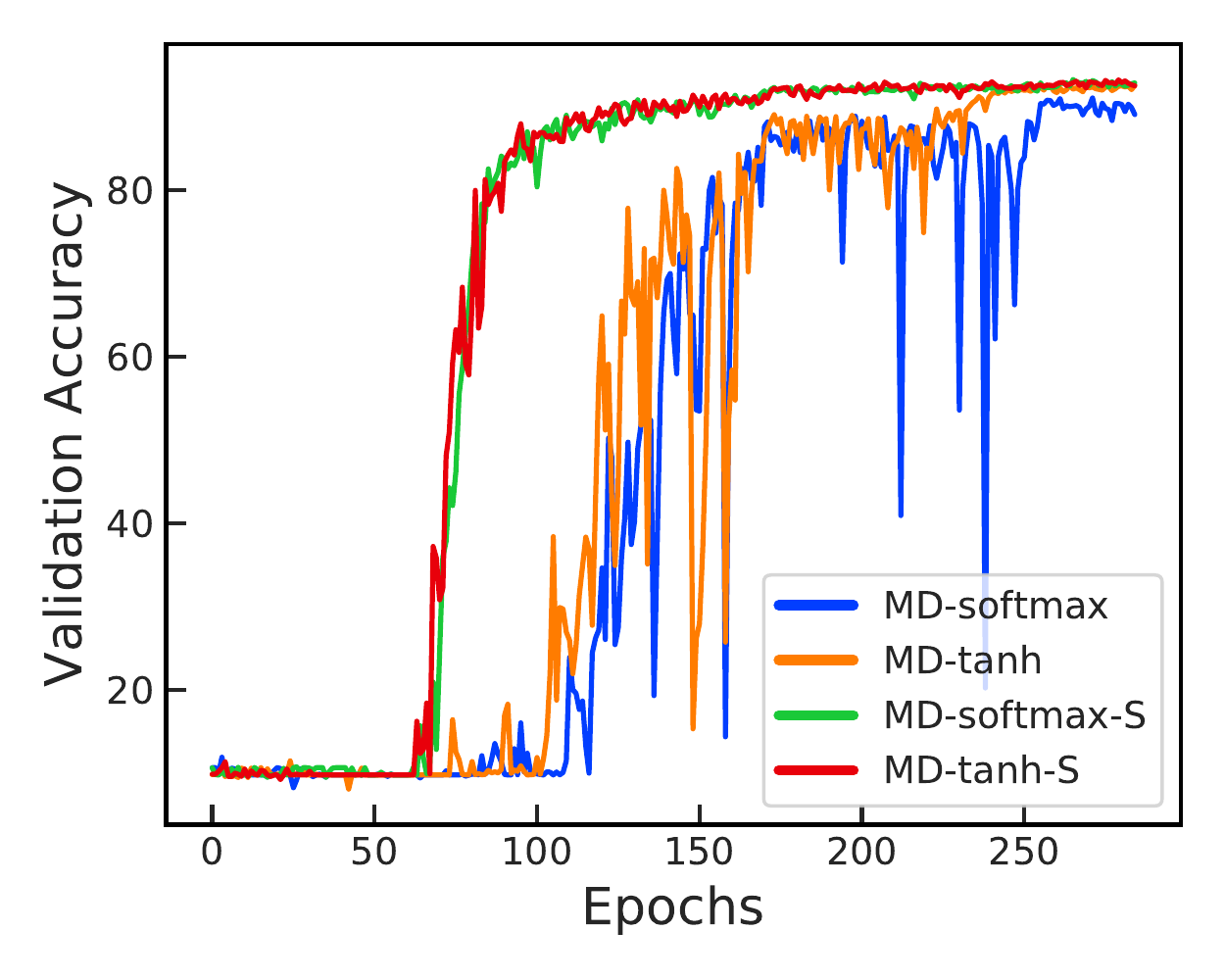}
    \end{subfigure}%
    \begin{subfigure}{0.25\linewidth}
    \includegraphics[width=0.99\linewidth]{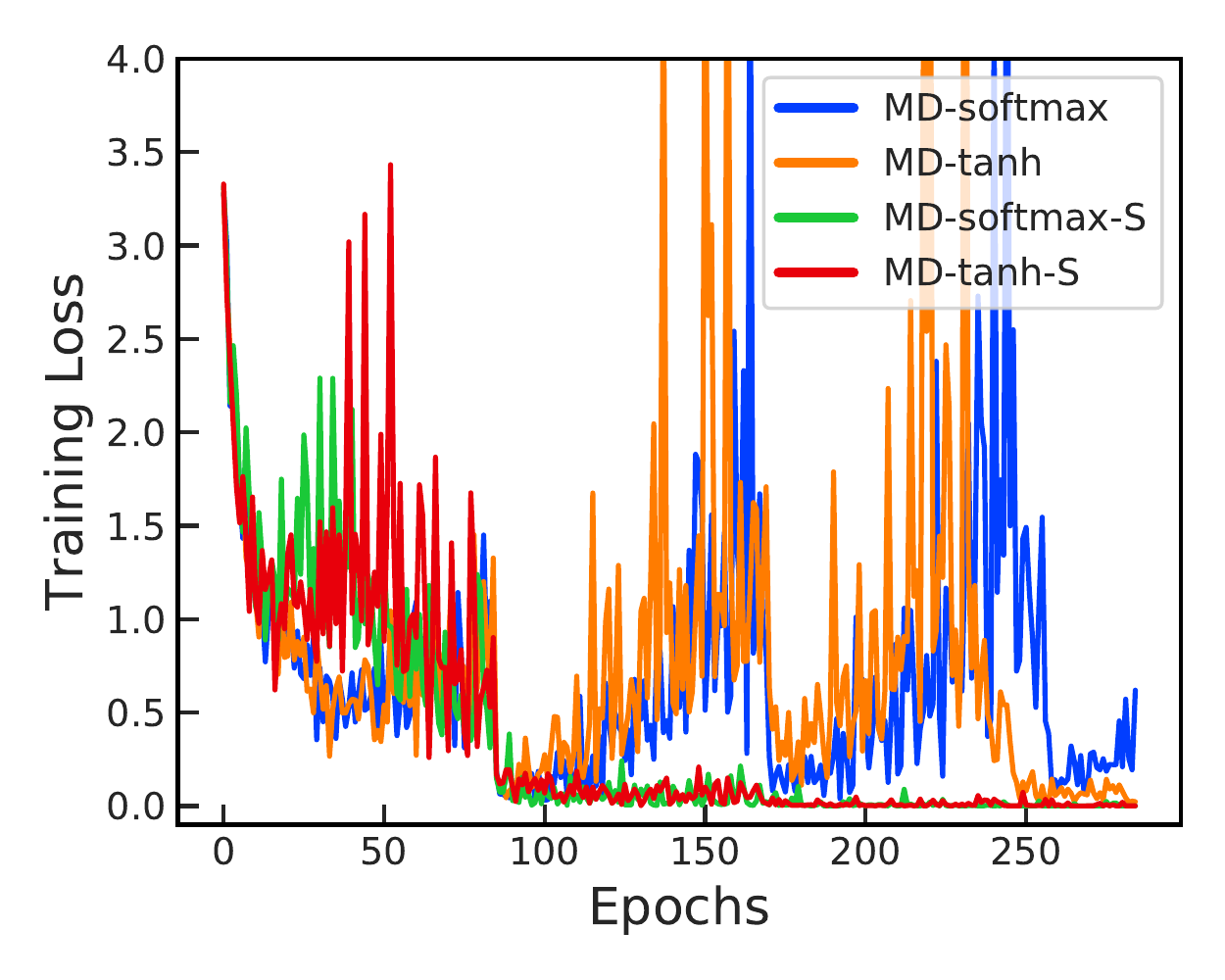}
    \end{subfigure}%
    \begin{subfigure}{0.25\linewidth}
    \includegraphics[width=0.99\linewidth]{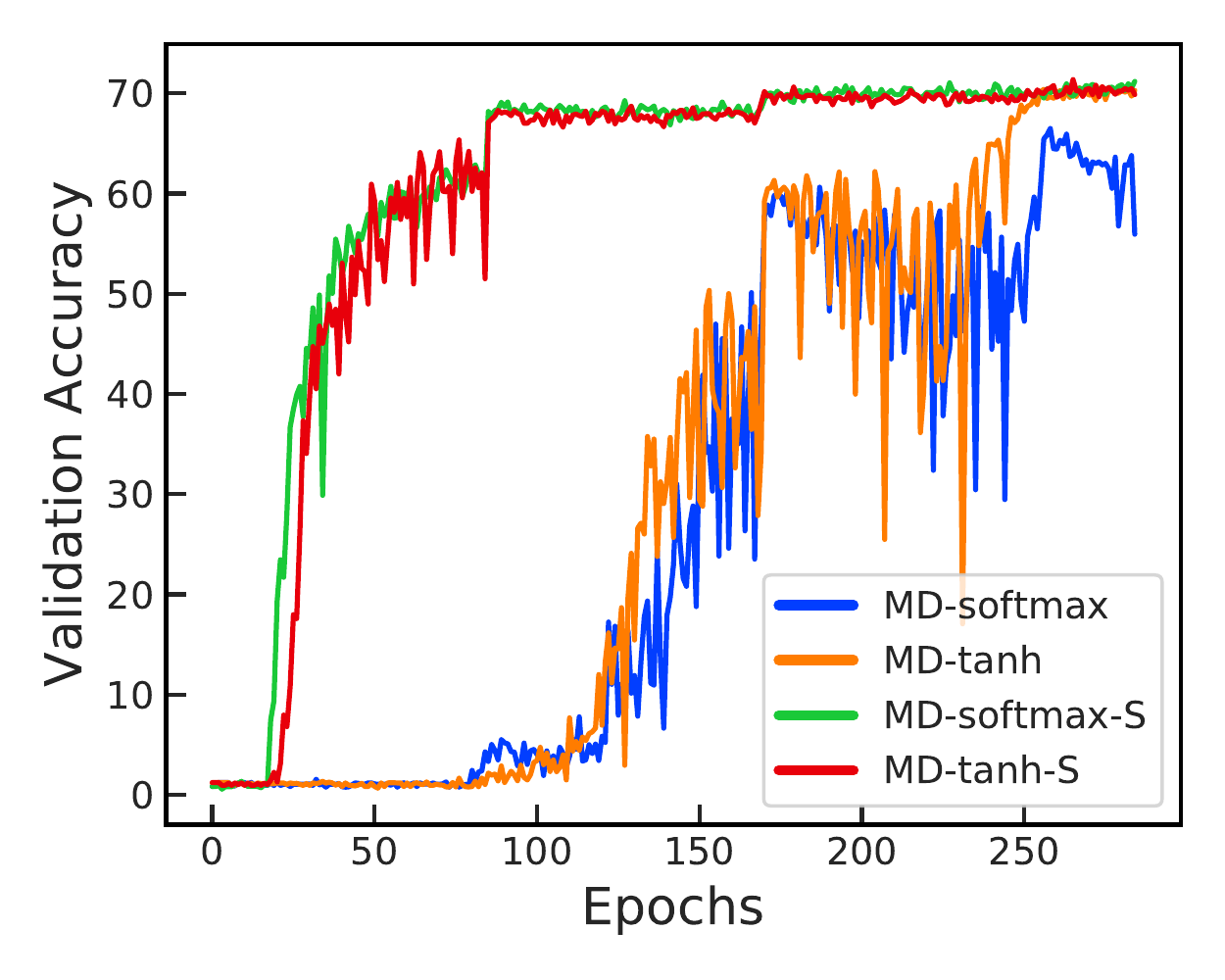}
    \end{subfigure}
    
    \caption{\em Training curves for binarization for \cifar{-10} (first two) and \cifar{-100} (last two) with \sresnet{-18}.
    Compared to original \gls{MD} variants, stable \gls{MD} variants are less noisy and after the initial exploration phase (up to $60$ in \cifar{-10} and $25$ epochs \cifar{-100}), the validation accuracies rise sharply and show gradual improvement afterwards. 
    }
    \label{fig:trcurves}
\end{figure*}
%
\subsection{Full Binarization of Parameters}

\input{tables/bracewell_mobilenet.tex}

We evaluate both of our \gls{MD} variants corresponding to $\tanh$ and $\softmax$ projections and their numerically stable counterparts as noted in~\eqref{eq:mdhybridgd}. The results are compared against parameter quantization methods, namely \acrfull{BC}~\citenew{courbariaux2015binaryconnect}, \acrfull{PQ}~\citenew{bai2018proxquant} and \gls{PMF}~\citenew{ajanthan2018pmf}. 
In addition, for completeness, we also compare against a standard \gls{PGD} variant corresponding to the $\tanh$ projection (denoted as \acrshort{GD}-$\tanh$), \ie, minimizing $f(\tanh(\tbfx))$ using gradient descent.
Note that, numerous techniques have emerged with \gls{BC} as the workhorse algorithm by relaxing constraints such as the layer-wise scalars~\citenew{rastegari2016xnor}, and similar extensions are straightforward even in our case though our variants perform well even without using layer-wise scalars. 

The classification accuracies of binary networks obtained by both variants of our algorithm, namely, \gls{MD}-$\tanh$ and \gls{MD}-$\softmax$, their numerically stable versions (denoted with suffix ``-\acrshort{S}'') and the baselines \gls{BC}, \gls{PQ}, \gls{PMF}, \acrshort{GD}-$\tanh$ and the floating point \gls{REF} are reported in~\tabref{tab:res_binary}. 
Both the numerically stable \gls{MD} variants consistently produce better or on par results compared to other binarization methods while narrowing the performance gap between binary networks and floating point counterparts to a large extent, on multiple datasets.

Our stable \acrshort{MD}-variant perform slightly better than \acrshort{MD}-$\softmax$, whereas for $\tanh$, \acrshort{MD} updates either perform on par or sometimes even better than numerically stable version of \acrshort{MD}-$\tanh$. We believe, the main reason for this empirical variation in results for our \acrshort{MD} variants is due to numerical instability caused by the floating-point arithmetic of logarithm and exponential functions in update steps for \acrshort{MD} (refer to \tabref{tab:eg_projns}).
Furthermore, even though our two \acrshort{MD}-variants, namely \acrshort{MD}-$\softmax$ and \acrshort{MD}-$\tanh$ optimize in different spaces, their performance is similar in most cases. \SKIP{\NOTE{correct this!} This may be explained by the fact that both algorithms belong to the same family where a ``soft'' projection to the constraint set is used and an annealing hyperparameter is used to gradually enforce a quantized solution.}

Note, \acrshort{PQ}~\citenew{bai2018proxquant}  does not quantize the fully-connected layers, biases, and shortcut layers. For fair comparison as previously mentioned, we crossvalidate \acrshort{PQ} with all layers binarized and original \acrshort{PQ} settings, and report the results denoted as \acrshort{PQ} and \acrshort{PQ}* respectively in~\tabref{tab:res_binary}. Our \acrshort{MD}-variants outperform \acrshort{PQ}  consistently on multiple datasets in equivalent experimental settings. 
This clearly shows that entropic or $\tanh$-based regularization with our annealing scheme is superior to a simple ``W'' shaped regularizer and emphasizes that \gls{MD} is a suitable framework for quantization.
Furthermore, the superior performance of \acrshort{MD}-$\tanh$ against \acrshort{GD}-$\tanh$ and on par or better performance of \acrshort{MD}-$\softmax$ against \acrshort{PMF} for binary quantization empirically validates that \gls{MD} is useful even in a nonconvex stochastic setting. This hypothesis along with our numerically stable form of \acrshort{MD} can be particularly useful to explore other projections that are useful for quantization and/or network compression in general.

The training curves for our \gls{MD} variants for \cifar{-10} and \cifar{-100} datasets with \sresnet{-18} are shown in~\figref{fig:trcurves}. The original \gls{MD} variants show unstable behaviour during training. 
Regardless, by storing auxiliary variables, the \gls{MD} updates are demonstrated to be quite stable. 
This clear distinction between \gls{MD} variants emphasizes the significance of practical considerations while implementing \gls{MD} especially in \gls{NN} optimization. 

To further demonstrate the superiority of \acrshort{MD}, we tested on a more resource efficient \smobilenet{V2}~\citenew{sandler2018mobilenetv2} and the results are summarized in~\tabref{tab:res_mobilenet_bracewell}.
In short, our \acrshort{MD} variants are able to fully-binarize \smobilenet{V2} with minimal loss in accuracy on \cifar{} datasets.
For more experiments such as training curves comparison to other methods and ternary quantization results please refer to the Appendix~\myref{B}.
\input{tables/imagenet.tex}

\input{tables/imagenet_binarywtact}

\subsection{Binarization on \imagenet{}}

We evaluated our \mdtanhs{} against state-of-the-art methods on \imagenet{} with \sresnet{-18} for parameter binarization and the results are reported in \tabref{tab:res_imagenet}.
Following the standard practice, we do not quantize the first convolution layer, last fully-connected layer, biases, and batchnorm parameters for all the compared methods and in this case, we set a {\em new state-of-the-art for binarization} with achieving merely $\bf{<4\%}$ reduction compared to the floating-point network.
Note that, the standard practice to quantize on \imagenet{} is to use floating-point scalars in each layer~\citenew{rastegari2016xnor,yang2019quantization}, however, our method outperforms all the methods without requiring layerwise scalars.
In addition, \mdtanhs{} yields the best results even when trained from scratch with ${\bf <1\%}$ reduction compared to finetuning from a pretrained network.  
Furthermore, \gls{MD} enables the training of fully-binarized networks with no additional scalars (except the batchnorm parameters) from scratch, which is considered to be difficult for \imagenet{}~\citenew{rastegari2016xnor}.
More ablation study experiments can be found in Appendix~\myref{B}.

Similar to the above, the results for both parameters and activations binarized networks are reported in \tabref{tab:res_imagenet_binarywtact}. For this experiment, we use Bireal-Net-18~\citenew{liu2018bireal} with PReLU activations as network architecture for our proposed algorithm. 
Even in this experiment, the first convolution layer, last fully-connected layer, biases, and batchnorm parameters are not binarized for all the methods. Our \mdtanhs{} achieves state-of-the-art results even when both parameters and activations are binarized, beating all the comparable baselines by almost ${\bf 1\%}$. 
In addition, similar to~\cite{zhuang2018towards,martinez2019training,liu2020reactnet}, we ran \mdtanhs{} with \acrshort{KL} Divergence loss (replacing the  cross-entropy loss) between the softmax output of our binary network and \acrshort{REF} (trained \sresnet{-34} on \imagenet{}).
Our \mdtanhs{} with \acrshort{KL} divergence loss outperforms previous methods by a significant margin of ${\bf >3\%}$, which clearly reflects the efficacy of \acrshort{MD}.

%% file: tables/results_binary.tex
\begin{table*}[t]
    \centering
    \small    
    \begin{tabular}{ll|c|cc|cc|c}
        \toprule
        &\multirow{2}{*}{Algorithm}   & \multirow{2}{*}{Space} & \multicolumn{2}{c|}{\cifar{-10}} & \multicolumn{2}{c|}{\cifar{-100}} & \tinyimagenet{}\\
         &&  &  \svgg{-16}  & \sresnet{-18} & \svgg{-16}  & \sresnet{-18}  & \sresnet{-18} \\
        \midrule
        &\acrshort{REF} (float) & $\bfw$ & $93.33$ & $94.84$ & $71.50$ & $76.31$ & $58.35$\\
        &\acrshort{BC} & $\bfw$ & $89.04$ & $91.64$ & $59.13$ & $72.14$  & $49.65$\\
        &\acrshort{PQ} & $\bfw$ & $85.41$ & $90.76$ & $39.61$ & $65.13$ & $44.32$ \\
        &\acrshort{PQ}* & $\bfw$ & $90.11$ & $92.32$ & $55.10$ & $68.35$ & $49.97$ \\
        &\acrshort{PMF} & $\bfu$ & $90.51$ & $92.73$ & $61.52$ & $71.85$ & $51.00$ \\
        &\acrshort{PMF}* & $\bfu$ & $91.40$ & $93.24$ & $\textbf{64.71}$ & $71.56$ & $51.52$ \\

        & \acrshort{GD}-$\tanh$ & $\bfw$ & $91.47$ & $93.27$ & $60.67$ & $71.46$ & $51.43$  \\
        
    \midrule
    \parbox[t]{2mm}{\multirow{4}{*}{\rotatebox[origin=c]{90}{Ours}}}
      &\acrshort{MD}-$\softmax$ & $\bfu$ & $90.47$ & $91.28$  & $56.25$ & $68.49$ & $46.52$ \\
      &\acrshort{MD}-$\tanh$ & $\bfw$ & $\textbf{91.64}$ & $92.97$  & $61.31$ & $72.13$ & $\textbf{54.62}$ \\      
      &\acrshort{MD}-$\softmax$-\acrshort{S} & $\bfu$ & $91.30$ & $\textbf{93.28}$  & $63.97$ & $\textbf{72.18}$ & $51.81$ \\
      &\acrshort{MD}-$\tanh$-\acrshort{S} & $\bfw$ & $91.53$ & $93.18$  & $61.69$ & $\textbf{72.18}$ & $52.32$ \\
        \bottomrule
    \end{tabular}
    \caption{\em Classification accuracies on the test set where all the parameters are binarized. \acrshort{PQ}* denotes performance with biases, fully-connected layers, and shortcut layers in float (original \acrshort{PQ} setup) whereas \acrshort{PQ} represents full quantization. 
    \acrshort{PMF}* denotes the performance of \acrshort{PMF} after crossvalidation and the original results from the paper are denoted as \acrshort{PMF}.
    Note our \acrshort{MD} variants obtained accuracies virtually the same as the best performing method and it outperformed the best method by a large margin in much harder \tinyimagenet{} dataset.}
    \label{tab:res_binary}
\end{table*}

%% file: tables/bracewell_mobilenet.tex
\begin{table}
    \centering
    \small
    \begin{tabular}{ll|c|c}
        \toprule
        &Algorithm & \cifar{-10} & \cifar{-100}\\
        \midrule
        &\acrshort{REF} (float) & $93.67$ & $73.97$ \\
        &\acrshort{BC} & $86.84$ & $65.04$ \\
        \midrule
        \parbox[t]{2mm}{\multirow{2}{*}{\rotatebox[origin=c]{90}{Ours}}}
        &\acrshort{MD}-$\softmax$-\acrshort{S} & $89.71$ & $66.40$ \\ 
        &\acrshort{MD}-$\tanh$-\acrshort{S} & $\textbf{89.99}$ & $\textbf{66.63}$ \\
        \bottomrule
    \end{tabular}
    \caption{\em Classification accuracies on the test set for \smobilenet{V2} where all the parameters are binarized.
    Our stable \acrshort{MD} variants significantly outperformed \acrshort{BC} while \acrshort{MD}-$\tanh$-\acrshort{S} is slightly better than \acrshort{MD}-$\softmax$-\acrshort{S}.
    }
    \vspace{-1ex}
    \label{tab:res_mobilenet_bracewell}
\end{table}

%% file: tables/imagenet.tex
\begin{table}
    \centering
    \small
    \begin{tabular}{ll|cc}
        \toprule
        & Algorithm & Top-1 & Top-5 \\
        \midrule
        & \acrshort{REF} (float) & $70.61$ & $89.40$ \\
        \midrule
        &\acrshort{BWN}*~\citenew{rastegari2016xnor} & $60.80$ & $83.00$ \\
        &\acrshort{BR}~\citenew{yin2018binaryrelax} & $63.20$ & $85.10$ \\
        &\acrshort{ELQ}~\citenew{zhou2018explicit} & $64.72$ & $86.04$ \\
        &\acrshort{ADMM}~\citenew{leng2018extremely} & $64.80$ & $86.20$ \\
        &\acrshort{QN}~\citenew{yang2019quantization} & $66.50$ & $\textbf{87.03}$ \\
        \midrule
        \parbox[t]{2mm}{\multirow{3}{*}{\rotatebox[origin=c]{90}{Ours}}} &  \mdtanhs{}*\textsuperscript{+} & $56.67$ & $79.66$ \\
        & \mdtanhs{}* & $65.92$ & $86.29$ \\
        & \mdtanhs{} & $\textbf{66.78}$ & $87.01$ \\
        \bottomrule
    \end{tabular}
    \caption{\em \imagenet{} classification accuracies for binary quantization (only parameters) on \sresnet{-18}. 
   	Here, * indicates training from scratch and \textsuperscript{+} indicates full-binarization except the batchnorm parameters.
   	Note \mdtanhs{} outperforms all other methods setting new state-of-the-art on \imagenet{} binarization.
   	It might seem that the improvement over \acrshort{QN} is marginal, however, \acrshort{QN} requires layerwise scalars and pretraining~\citenew{yang2019quantization}, whereas \mdtanhs{} does not require layerwise scaling and obtains near state-of-the-art results even without pretraining.
    }        
    \label{tab:res_imagenet}
    \vspace{-5ex}
\end{table}

%% file: tables/imagenet_binarywtact.tex
\begin{table}
    \centering
    \small
    \begin{tabular}{ll|c@{\hspace{4pt}}c}
        \toprule
        \multicolumn{2}{l|}{Algorithm} & Top-1 & Top-5 \\
        \midrule
        \multicolumn{2}{l|}{\acrshort{REF} (float)} & $70.6$ & $89.4$ \\ 
        \midrule
        \multicolumn{2}{l|}{BinaryNet~\citenew{hubara2016binarized}} & $42.2$ & $67.1$ \\
        \multicolumn{2}{l|}{Dorefa-Net~\citenew{zhou2016dorefa}} & $52.5$ & $76.7$ \\
        \multicolumn{2}{l|}{XNOR-Net~\citenew{rastegari2016xnor}} & $51.2$ & $73.2$ \\
        \multicolumn{2}{l|}{Bireal-Net~\citenew{liu2018bireal}} & $56.4$ & $79.5$ \\
        \multicolumn{2}{l|}{Bireal-Net~\citenew{liu2018bireal} (PReLU)} & $59.0$ & $81.3$ \\
        \multicolumn{2}{l|}{PCNN (J=1)~\citenew{gu2019projection}} & $57.3$ & $80.0$ \\
        \multicolumn{2}{l|}{QN~\citenew{yang2019quantization}} & $53.6$ & $75.3$ \\
        \multicolumn{2}{l|}{BOP~\citenew{helwegen2019latent}} & $54.2$ & $77.2$ \\
        \multicolumn{2}{l|}{GBCN~\citenew{liu2019gbcns}} & $57.8$ & $80.9$ \\
        \multicolumn{2}{l|}{IR-Net~\citenew{qin2020forward}} & $58.1$ & $80.0$ \\
        \multicolumn{2}{l|}{Noisy Supervision~\citenew{hantraining}} & $59.4$ & $81.7$ \\
        \midrule
        \parbox[t]{2mm}{\multirow{2}{*}{\rotatebox[origin=c]{90}{Ours}}} & \acrshort{MD}-$\tanh$-\acrshort{S} & $60.3$ & $82.3$ \\
        & \acrshort{MD}-$\tanh$-\acrshort{S} (\acrshort{KL} div. loss) & $\textbf{62.8}$ & $\textbf{84.3}$ \\

        \bottomrule
    \end{tabular}
    \caption{\em ImageNet classification accuracies for binary quantization (both parameters and activations) on \sresnet{-18}. Here, our \mdtanhs{} and \cite{hantraining} use Bireal-Net-18 \citenew{liu2018bireal} with PReLU activations as baseline architecture. We also show results of our method with \acrshort{KL} divergence loss between softmax output of our binary network and \acrshort{REF} (trained \sresnet{-34} on \imagenet{}).
    Note, \mdtanhs{} clearly outperforms all the methods.
    }
    \label{tab:res_imagenet_binarywtact}
\end{table}

%% file: text/discussion.tex
\section{Discussion}
In this work, we have introduced an \gls{MD} framework for \gls{NN} quantization by deriving mirror maps corresponding to two projections useful for quantization and provided two algorithms for quantization.
Theoretically, we provided a convergence analysis in the convex setting for time-varying mirror maps and discussed conditions to ensure convergence to a discrete solution when an annealing hyperparameter is employed.
In addition, we have discussed a numerically stable implementation of \gls{MD} by storing an additional set of auxiliary variables and showed that this update is strikingly analogous to the popular \gls{STE} based gradient method.
The superior performance of our \gls{MD} formulation even with simple projections such as $\tanh$ and $\softmax$ 
is encouraging and we believe, \gls{MD} would be a suitable framework for not just \gls{NN} quantization but for network compression in general.
In the future, we intend to focus more on the theoretical aspects of \gls{MD} in conjunction with stochastic momentum based optimizers such as Adam.

%% file: text/ack.tex
\section{Acknowledgements}
This work was supported by the ERC grant ERC-2012-AdG 321162-HELIOS, EPSRC grant Seebibyte EP/M013774/1, EPSRC/MURI grant EP/N019474/1, and the Australian Research Council Centre of Excellence for Robotic Vision (project number CE140100016). We would like to thank Pawan Kumar for useful discussions on mirror descent and we acknowledge the Royal Academy of Engineering, FiveAI, Data61, CSIRO, and National Computing Infrastructure, Australia.

%% file: mda_aistats.bbl
\begin{thebibliography}{}

\bibitem[Achterhold et~al., 2018]{achterhold2018variationalQuantization}
Achterhold, J., Kohler, J.~M., Schmeink, A., and Genewein, T. (2018).
\newblock Variational network quantization.
\newblock {\em ICLR}.

\bibitem[Ajanthan et~al., 2019]{ajanthan2018pmf}
Ajanthan, T., Dokania, P.~K., Hartley, R., and Torr, P.~H. (2019).
\newblock Proximal mean-field for neural network quantization.
\newblock {\em ICCV}.

\bibitem[Bai et~al., 2019]{bai2018proxquant}
Bai, Y., Wang, Y.-X., and Liberty, E. (2019).
\newblock Proxquant: Quantized neural networks via proximal operators.
\newblock {\em ICLR}.

\bibitem[Beck and Teboulle, 2003]{beck2003mirror}
Beck, A. and Teboulle, M. (2003).
\newblock Mirror descent and nonlinear projected subgradient methods for convex
  optimization.
\newblock {\em Operations Research Letters}.

\bibitem[Boyd and Vandenberghe, 2009]{boyd2009convex}
Boyd, S. and Vandenberghe, L. (2009).
\newblock {\em Convex optimization}.
\newblock Cambridge university press.

\bibitem[Bubeck, 2015]{bubeck2015convex}
Bubeck, S. (2015).
\newblock Convex optimization: Algorithms and complexity.
\newblock {\em Foundations and Trends{\textregistered} in Machine Learning}.

\bibitem[Bubeck et~al., 2012]{bubeck2012towards}
Bubeck, S., Cesa-Bianchi, N., and Kakade, S.~M. (2012).
\newblock Towards minimax policies for online linear optimization with bandit
  feedback.
\newblock {\em Conference on Learning Theory}.

\bibitem[Carreira-Perpin{\'a}n and Idelbayev, 2017]{carreira2017model}
Carreira-Perpin{\'a}n, M.~A. and Idelbayev, Y. (2017).
\newblock Model compression as constrained optimization, with application to
  neural nets. part ii: Quantization.
\newblock {\em NeurIPS Workshop on Optimization for Machine Learning}.

\bibitem[Chen et~al., 2019]{chen2019metaquant}
Chen, S., Wang, W., and Pan, S.~J. (2019).
\newblock {MetaQuant:} learning to quantize by learning to penetrate
  non-differentiable quantization.
\newblock {\em NeurIPS}.

\bibitem[Courbariaux et~al., 2015]{courbariaux2015binaryconnect}
Courbariaux, M., Bengio, Y., and David, J.-P. (2015).
\newblock Binaryconnect: Training deep neural networks with binary weights
  during propagations.
\newblock {\em NeurIPS}.

\bibitem[Esser et~al., 2015]{esser2015energyEfficient}
Esser, S.~K., Appuswamy, R., Merolla, P.~A., Arthur, J.~V., and Modha, D.~S.
  (2015).
\newblock Backpropagation for energy-efficient neuromorphic computing.
\newblock {\em NeurIPS}.

\bibitem[Goyal et~al., 2017]{goyal2017accurate}
Goyal, P., Doll{\'a}r, P., Girshick, R., Noordhuis, P., Wesolowski, L., Kyrola,
  A., Tulloch, A., Jia, Y., and He, K. (2017).
\newblock Accurate, large minibatch sgd: Training imagenet in 1 hour.
\newblock {\em arXiv preprint arXiv:1706.02677}.

\bibitem[Gu et~al., 2019]{gu2019projection}
Gu, J., Li, C., Zhang, B., Han, J., Cao, X., Liu, J., and Doermann, D. (2019).
\newblock Projection convolutional neural networks for 1-bit cnns via discrete
  back propagation.
\newblock {\em AAAI}.

\bibitem[Guo, 2018]{guo2018survey}
Guo, Y. (2018).
\newblock A survey on methods and theories of quantized neural networks.
\newblock {\em arXiv preprint arXiv:1808.04752}.

\bibitem[Han et~al., 2020]{hantraining}
Han, K., Wang, Y., Xu, Y., Xu, C., Wu, E., and Xu, C. (2020).
\newblock Training binary neural networks through learning with noisy
  supervision.
\newblock {\em ICML}.

\bibitem[He et~al., 2016]{he2016deep}
He, K., Zhang, X., Ren, S., and Sun, J. (2016).
\newblock Deep residual learning for image recognition.
\newblock {\em CVPR}.

\bibitem[Helwegen et~al., 2019]{helwegen2019latent}
Helwegen, K., Widdicombe, J., Geiger, L., Liu, Z., Cheng, K.-T., and Nusselder,
  R. (2019).
\newblock Latent weights do not exist: Rethinking binarized neural network
  optimization.
\newblock {\em NeurIPS}.

\bibitem[Hou et~al., 2017]{hou2016loss}
Hou, L., Yao, Q., and Kwok, J.~T. (2017).
\newblock Loss-aware binarization of deep networks.
\newblock {\em ICLR}.

\bibitem[Hsieh et~al., 2018]{hsieh2018mirrored}
Hsieh, Y.-P., Kavis, A., Rolland, P., and Cevher, V. (2018).
\newblock Mirrored langevin dynamics.
\newblock {\em NeurIPS}.

\bibitem[Huang et~al., 2017]{huang2017snapshot}
Huang, G., Li, Y., Pleiss, G., Liu, Z., Hopcroft, J.~E., and Weinberger, K.~Q.
  (2017).
\newblock Snapshot ensembles: Train 1, get m for free.
\newblock {\em ICLR}.

\bibitem[Hubara et~al., 2016]{hubara2016binarized}
Hubara, I., Courbariaux, M., Soudry, D., El-Yaniv, R., and Bengio, Y. (2016).
\newblock Binarized neural networks.
\newblock {\em NeurIPS}.

\bibitem[Hubara et~al., 2017]{hubara2017quantized}
Hubara, I., Courbariaux, M., Soudry, D., El-Yaniv, R., and Bengio, Y. (2017).
\newblock Quantized neural networks: Training neural networks with low
  precision weights and activations.
\newblock {\em JMLR}.

\bibitem[Ioffe and Szegedy, 2015]{ioffe2015batch}
Ioffe, S. and Szegedy, C. (2015).
\newblock Batch normalization: Accelerating deep network training by reducing
  internal covariate shift.
\newblock {\em ICML}.

\bibitem[Lee et~al., 2019]{lee2018snip}
Lee, N., Ajanthan, T., and Torr, P. H.~S. (2019).
\newblock {SNIP}: Single-shot network pruning based on connection sensitivity.
\newblock {\em ICLR}.

\bibitem[Leng et~al., 2018]{leng2018extremely}
Leng, C., Dou, Z., Li, H., Zhu, S., and Jin, R. (2018).
\newblock Extremely low bit neural network: Squeeze the last bit out with
  {ADMM}.
\newblock {\em AAAI}.

\bibitem[Li et~al., 2017]{li2017training}
Li, H., De, S., Xu, Z., Studer, C., Samet, H., and Goldstein, T. (2017).
\newblock Training quantized nets: A deeper understanding.
\newblock {\em NeurIPS}.

\bibitem[Liu et~al., 2019]{liu2019gbcns}
Liu, C., Ding, W., Hu, Y., Zhang, B., Liu, J., and Guo, G. (2019).
\newblock Gbcns: Genetic binary convolutional networks for enhancing the
  performance of 1-bit dcnns.
\newblock {\em CoRR}.

\bibitem[Liu et~al., 2020]{liu2020reactnet}
Liu, Z., Shen, Z., Savvides, M., and Cheng, K.-T. (2020).
\newblock Reactnet: Towards precise binary neural network with generalized
  activation functions.
\newblock {\em arXiv preprint arXiv:2003.03488}.

\bibitem[Liu et~al., 2018]{liu2018bireal}
Liu, Z., Wu, B., Luo, W., Yang, X., Liu, W., and Cheng, K.-T. (2018).
\newblock Bi-real net: Enhancing the performance of 1-bit cnns with improved
  representational capability and advanced training algorithm.
\newblock {\em ECCV}.

\bibitem[Louizos et~al., 2019]{louizos2018relaxed}
Louizos, C., Reisser, M., Blankevoort, T., Gavves, E., and Welling, M. (2019).
\newblock Relaxed quantization for discretized neural networks.
\newblock {\em ICLR}.

\bibitem[Louizos et~al., 2017]{louizon2017bayesianCompression}
Louizos, C., Ullrich, K., and Welling, M. (2017).
\newblock Bayesian compression for deep learning.
\newblock {\em NeurIPS}.

\bibitem[Martinez et~al., 2019]{martinez2019training}
Martinez, B., Yang, J., Bulat, A., and Tzimiropoulos, G. (2019).
\newblock Training binary neural networks with real-to-binary convolutions.
\newblock In {\em International Conference on Learning Representations}.

\bibitem[McDonnell, 2018]{mcdonnell2018training}
McDonnell, M.~D. (2018).
\newblock Training wide residual networks for deployment using a single bit for
  each weight.
\newblock {\em ICLR}.

\bibitem[McMahan, 2017]{mcmahan2017survey}
McMahan, H.~B. (2017).
\newblock A survey of algorithms and analysis for adaptive online learning.
\newblock {\em JMLR}.

\bibitem[Nemirovsky and Yudin, 1983]{nemirovsky1983problem}
Nemirovsky, A.~S. and Yudin, D.~B. (1983).
\newblock Problem complexity and method efficiency in optimization.

\bibitem[Paszke et~al., 2017]{paszke2017automatic}
Paszke, A., Gross, S., Chintala, S., Chanan, G., Yang, E., DeVito, Z., Lin, Z.,
  Desmaison, A., Antiga, L., and Lerer, A. (2017).
\newblock Automatic differentiation in {P}y{T}orch.

\bibitem[Qin et~al., 2020]{qin2020forward}
Qin, H., Gong, R., Liu, X., Shen, M., Wei, Z., Yu, F., and Song, J. (2020).
\newblock Forward and backward information retention for accurate binary neural
  networks.
\newblock {\em CVPR}.

\bibitem[Rastegari et~al., 2016]{rastegari2016xnor}
Rastegari, M., Ordonez, V., Redmon, J., and Farhadi, A. (2016).
\newblock Xnor-net: Imagenet classification using binary convolutional neural
  networks.
\newblock {\em ECCV}.

\bibitem[Sandler et~al., 2018]{sandler2018mobilenetv2}
Sandler, M., Howard, A., Zhu, M., Zhmoginov, A., and Chen, L.-C. (2018).
\newblock {Mobilenetv2:} inverted residuals and linear bottlenecks.
\newblock {\em CVPR}.

\bibitem[Simonyan and Zisserman, 2015]{simonyan2014very}
Simonyan, K. and Zisserman, A. (2015).
\newblock Very deep convolutional networks for large-scale image recognition.
\newblock {\em ICLR}.

\bibitem[Yang et~al., 2019]{yang2019quantization}
Yang, J., Shen, X., Xing, J., Tian, X., Li, H., Deng, B., Huang, J., and Hua,
  X.-s. (2019).
\newblock Quantization networks.
\newblock {\em CVPR}.

\bibitem[Ye et~al., 2019]{ye2019progressive}
Ye, S., Feng, X., Zhang, T., Ma, X., Lin, S., Li, Z., Xu, K., Wen, W., Liu, S.,
  Tang, J., et~al. (2019).
\newblock Progressive dnn compression: A key to achieve ultra-high weight
  pruning and quantization rates using admm.
\newblock {\em arXiv preprint arXiv:1903.09769}.

\bibitem[Yin et~al., 2019]{yin2019understanding}
Yin, P., Lyu, J., Zhang, S., Osher, S., Qi, Y., and Xin, J. (2019).
\newblock Understanding straight-through estimator in training activation
  quantized neural nets.
\newblock {\em ICLR}.

\bibitem[Yin et~al., 2018]{yin2018binaryrelax}
Yin, P., Zhang, S., Lyu, J., Osher, S., Qi, Y., and Xin, J. (2018).
\newblock Binaryrelax: A relaxation approach for training deep neural networks
  with quantized weights.
\newblock {\em SIIMS}.

\bibitem[Zhang et~al., 2015]{zhang2015bit}
Zhang, R., Lin, L., Zhang, R., Zuo, W., and Zhang, L. (2015).
\newblock Bit-scalable deep hashing with regularized similarity learning for
  image retrieval and person re-identification.
\newblock {\em TIP}.

\bibitem[Zhang and He, 2018]{zhang2018convergence}
Zhang, S. and He, N. (2018).
\newblock On the convergence rate of stochastic mirror descent for nonsmooth
  nonconvex optimization.
\newblock {\em CoRR}.

\bibitem[Zhou et~al., 2018]{zhou2018explicit}
Zhou, A., Yao, A., Wang, K., and Chen, Y. (2018).
\newblock Explicit loss-error-aware quantization for low-bit deep neural
  networks.
\newblock {\em CVPR}.

\bibitem[Zhou et~al., 2016]{zhou2016dorefa}
Zhou, S., Wu, Y., Ni, Z., Zhou, X., Wen, H., and Zou, Y. (2016).
\newblock Dorefa-net: Training low bitwidth convolutional neural networks with
  low bitwidth gradients.
\newblock {\em CoRR}.

\bibitem[Zhou et~al., 2017a]{zhou2017mirror}
Zhou, Z., Mertikopoulos, P., Bambos, N., Boyd, S., and Glynn, P. (2017a).
\newblock Mirror descent in non-convex stochastic programming.
\newblock {\em CoRR}.

\bibitem[Zhou et~al., 2017b]{zhou2017stochastic}
Zhou, Z., Mertikopoulos, P., Bambos, N., Boyd, S., and Glynn, P.~W. (2017b).
\newblock Stochastic mirror descent in variationally coherent optimization
  problems.
\newblock {\em NeurIPS}.

\bibitem[Zhuang et~al., 2018]{zhuang2018towards}
Zhuang, B., Shen, C., Tan, M., Liu, L., and Reid, I. (2018).
\newblock Towards effective low-bitwidth convolutional neural networks.
\newblock In {\em Proceedings of the IEEE conference on computer vision and
  pattern recognition}, pages 7920--7928.

\end{thebibliography}
